\documentclass[11pt]{article}

\usepackage{microtype}
\usepackage{graphicx}
\usepackage{subfigure}
\usepackage{booktabs} % for professional tables
\usepackage{multirow}
\usepackage{times}
\usepackage{amsthm}
\usepackage{mathtools}
\usepackage{booktabs, dcolumn}
\usepackage{mathrsfs}
\usepackage{caption}
\usepackage[T1]{fontenc}
\usepackage{fancyvrb}
\usepackage{xcolor}
\usepackage{amsmath,amssymb,listings,epstopdf,float,bm}
\usepackage{algorithm,algpseudocode}
\usepackage[margin=1in]{geometry}
\usepackage[round]{natbib}

\usepackage{nomencl}
\usepackage{mdframed}
\usepackage{multicol}

\makenomenclature

\newcommand{\newstuffblue}[1]{{\color{black}{#1}}}

\newtheorem{thm}{Proposition}[section]
\theoremstyle{remark}
\newtheorem*{remark}{Remark}
\theoremstyle{definition}

\newtheorem{lemma}[thm]{Lemma}
\newcommand{\im}{\mathcal{I}}
\newcommand{\E}{\mathbb{E}}

\newcommand{\taun}{{\tau}}
\newcommand{\cA}{\mathcal{A}}

\newcommand{\sse}{\sigma_{\text{se}}}
\newcommand{\eps}{\epsilon}

\newcommand{\xx}{\mathbf{x}}

\newcommand{\bybar}{\bar{\mathbf{y}}_{1:k_n}}
\newcommand{\bxbar}{\bar{\mathbf{x}}_{1:k_n}}
\newcommand{\xnew}{x_{k_n+1}}

\DeclareMathOperator\sgn{sign}
\DeclareMathOperator\diag{diag}

\newcommand{\newstuff}[1]{{\color{black}{#1}}}

\makeatletter

\newenvironment{ALC@g}{
    \begin{list}{}{ \itemsep\z@ \itemindent\z@
    \listparindent\z@ \rightmargin\z@
    \topsep\z@ \partopsep\z@ \parskip\z@\parsep\z@
    \leftmargin 1em
    \addtolength{\ALC@tlm}{\leftmargin}
    }
  }
  {\end{list}}
  \makeatother

\newcommand{\CASE}[1]{\State \textbf{case} #1\textbf{:} \begin{list}}
\newcommand{\ENDCASE}{\end{list}}

\newcommand{\DEFAULT}{\State \textbf{default:} \begin{list}}
\newcommand{\ENDDEFAULT}{\end{list}}
\newcommand{\DEFAULTLINE}[1]{\State \textbf{default:} }

\algdef{SE}[SUBALG]{Indent}{EndIndent}{}{\algorithmicend\ }%
\algtext*{Indent}
\algtext*{EndIndent}

\usepackage{hyperref}

\title{Adaptive Batching for Gaussian Process Surrogates with Application in Noisy Level Set Estimation}

\author{{Xiong Lyu}\thanks{Department of Statistics and Applied Probability, University of California at Santa Barbara, Santa Barbara, CA 93106-3110, USA (\href{mailto:lyu@pstat.ucsb.edu}{lyu@pstat.ucsb.edu}; \href{mailto:ludkovski@pstat.ucsb.edu}{ludkovski@pstat.ucsb.edu})}
	\and {Michael Ludkovski}$^*$}

\date{\vspace{-5ex}}

\begin{document}

	\maketitle
	
	\begin{abstract}
		
		We develop adaptive replicated designs for Gaussian process metamodels of stochastic experiments. Adaptive batching is a natural extension of sequential design heuristics with the benefit of replication growing as response features are learned, inputs concentrate, and the metamodeling overhead rises. Motivated by the problem of learning the level set of the mean simulator \newstuff{response, we} develop \newstuff{five} novel schemes: Multi-Level Batching (MLB), Ratchet Batching (RB), Adaptive Batched Stepwise Uncertainty Reduction (ABSUR), Adaptive Design with Stepwise Allocation (ADSA) and Deterministic Design with Stepwise Allocation (DDSA). Our algorithms simultaneously (MLB, RB and ABSUR) or sequentially (ADSA and DDSA) determine the sequential design inputs and the respective number of replicates. Illustrations using synthetic examples and an application in quantitative finance (Bermudan option pricing via Regression Monte Carlo) show that adaptive batching brings significant computational speed-ups with minimal loss of modeling fidelity.
		
	\end{abstract}


%\jnlcitation{\cname{%
%\author{Lyu X.},
%\author{M. Ludkovski}} (\cyear{2016}),
%\ctitle{A regime analysis of Atlantic winter jet variability applied to evaluate HadGEM3-GC2}, \cjournal{}, \cvol{}.}

	\section{Introduction}\label{sec:introduction}
	
	Metamodels offer a cheap statistical representation of complex and/or expensive stochastic simulators that arise in applications ranging from engineering to environmental science and finance \citep{santner2003design}. Gaussian process (GP) frameworks have emerged as the leading family of metamodels thanks to their flexibility, analytical tractability and superior empirical performance. However, for GP metamodels to be fast, it is imperative to keep the respective design size $|\cA|$ manageable. In particular, unless the simulator is truly expensive or the input domain is vast, the typical recommendation is to restrict to hundreds of inputs, $|\cA| \ll 10^3$. This creates a major tension as frequently the stochastic simulator has low signal-to-noise ratio or a complex noise structure. A prototypical example is where the simulator $Y(x) = F(X_{[0,\Delta t ]})|_{X_0 = x}$ involves functionals of a continuous-time Markov chain or stochastic differential equation solution $(X_t)$, whereby the stochasticity tends to dominate the trend/drift term for short $\Delta t$, and moreover simulation noise is non-Gaussian and state-dependent (heteroskedastic).
	
	A natural solution is to employ \emph{batching}, known in the stochastic simulation community as nested Monte Carlo.
	Re-using the same input to generate multiple outputs allows for a Law of Large Numbers (LLN) averaging which can be {analytically} combined with the GP predictive equations to keep the computational complexity as a function of $k$ (number of unique inputs) rather than of the capital-$N$ (number of simulator calls). The seminal technique of \emph{stochastic kriging} \citep{ankenman2010stochastic} shows that these computational savings are exact assuming the GP hyperparameters, in particular the noise variance {$\taun^2$}, are known. Such batching becomes critical in the use of GP models in our motivating application of solving optimal stopping problems via Regression Monte Carlo, where tens of thousands of simulations are called for.
	
	In the classical setup, the metamodeling objective is to learn the mean response over the entire domain \citep{koehler1998estimating, le2015asymptotic, chen2017sequential}, whereby, modulo heteroskedastic noise, one expects to utilize the same batching level across all inputs, i.e.~splitting the total budget $N=k \times r$ into $k$ batches of $r$ replicates at locations $\bar{x}_1, \ldots, \bar{x}_k$. See~\cite{ankenman2010stochastic} for a discussion of how to pick $k$ for a given budget $N$, as well as some proposals for handling non-constant $\taun^2(x)$.  We are interested in more targeted objectives, where the picture is much less clear. As two canonical examples we recall Bayesian Optimization (finding the maximum mean response) and Level Set Estimation (determining the input sub-domain where the mean response exceeds a given threshold). In both settings GP metamodels have been shown to especially shine, not least because they organically match the sequential adaptive designs typically utilized; the respective Expected Improvement schemes form a major feature of the GP ecosystem. Since these objectives imply preferentially sampling a small portion of the input space---the neighborhood of the maximum, or the neighborhood of the desired contour---the exploration-exploitation paradigm leads to increasingly concentrated designs. Such concentration suggests to adaptively determine the amount of batching. Intuitively, replication should be low for more exploratory sites and should rise in the neighborhood of interest, where we replicate to achieve computational savings. Indeed, the intrinsic cost of replication is linked to the variability of the response at the respective inputs, which will be minimal if the inputs are very close together.  From a different perspective, replication trades off costly, precise outputs (large $r$) vis-a-vis cheap outputs with low signal-to-noise ratio (low $r$). %We then wish to optimize the number of unique locations with the fidelity of the averaged simulator outputs. The latter gains from more compact designs were investigated in Binois et al.~\citep{binois2019replication}.

	The above motivates \emph{adaptively batched} designs, where $r$ is input-dependent.  While this idea was investigated for Bayesian Optimization \citep{klein2017fast, poloczek2017multi} and for Integrated Mean Squared Error (IMSE) minimization \citep{ankenman2010stochastic,binois2019replication}, neither of these fully reveal the underlying tension between exploration (replicate less, larger metamodel overhead) and exploitation (replicate more, generate computational savings).
	In this article we propose several schemes that explicitly focus on this issue. To evaluate them we concentrate on the problem of level set estimation where the contour is adaptively learned through the sequential design but retains a spatial structure (unlike Bayesian Optimization where convergence to the single input yielding the global maximum is desired). Consequently, we expect a complex interaction between the selection of inputs and the respective replication amounts. In this context, our main contribution is to extend the paradigm of Expected Improvement to include sequential selection of both the input locations $x_{n}$ and the replication counts $r_n$. We benchmark the proposed algorithms and show that they provide significant savings compared to the naive fixed-batching approach. In particular, we are able to obtain schemes that reduce $N \simeq 10^5$ simulations to efficient replicated designs of just a few hundred unique inputs.
	
	Beyond benchmarking the developed algorithms on several synthetic examples, we also implement \newstuff{and extend them to heteroskedastic modeling} for the motivating application of valuation of Bermudan options. In the latter context, the Regression Monte Carlo (RMC) paradigm is used to provide a simulation-based algorithm that hinges on recursive estimation of certain level sets that correspond to the so-called stopping boundaries. Building upon the successful use of GP surrogates for RMC \citep{ludkovski2018kriging, lyu2018evaluating}, we demonstrate that adaptive batching significantly speeds up this approach, making  it more scalable and efficient. In particular while in \citep{ludkovski2018kriging} sequential design was typically too slow to be useful, adaptively batched models beat basic implementation on both speed and memory requirements. We note that there are other important applications of level set estimation,  from quantifying the reliability of a system or its failure probability \citep{bect2012sequential}, to ranking pay-offs from several available actions in dynamic programming \citep{hu2015sequential}. %In Section~\ref{sec:bermudan} we consider one concrete case study coming from simulation-based algorithms for valuation of Bermudan options . Our adaptively batched schemes provide the most efficient implementation to date of the GP RMC algorithm, in particular scaling well to 3D and 5D settings.
	
	The rest of the paper is organized as follows. Section~\ref{sec:model} formalizes the GP model and the contour-learning objective. Section~\ref{sec:batchdesign} develops heuristics for sequential designs that jointly optimize over the new input and replication level. Section~\ref{sec:csao} takes a different tack and explores dynamic replication through allocating new simulations to existing inputs. Section~\ref{sec:synthetic} benchmarks the proposed schemes on three synthetic case studies and Section~\ref{sec:bermudan} on two more examples from Bermudan option pricing. Section~\ref{sec:conclusions} concludes.

	\section{Statistical Model}\label{sec:model}
	Consider a latent $f: D \rightarrow \mathbb{R}$ which is a continuous function over a $d$-dimensional input space $D \subseteq  \mathbb{R}^d$. We wish to identify the contour $\partial S$, where, without loss of generality, $S$ is the zero level set 	
	\begin{align}
	S=\{x \in D: f(x) \geq 0 \}. %\quad N=\{x \in D: f(x)< 0\}. \nonumber
	\end{align}
	Thus, our metamodeling objective is  equivalent to learning the sign of $f(x)$ for any $x \in D$.
	%Or, equivalently, given any $x\in D = S \cup N$, we want to estimate the sign of the function $f(x)$.
	%
	%Without loss of generality the threshold is taken to be zero, so that the level set estimation is equivalent to learning the sign of the response function $f$. For later use we also define the corresponding zero-contour of $f$, namely the partition boundary $\partial S = \partial N = \{x \in D :f(x)=0\}$.
	For any $x_i\in D$, we have access to a simulator $Y(x_i)$ that generates noisy outputs of $f(x_i)$:
	\begin{align}
	Y(x_i) &= f(x_i)+\epsilon_i, \label{fundamental}
	\end{align}
	where $\epsilon_i$'s are realizations of independent, mean zero random variables with variance {${\taun^2}$}. To describe replicated inputs, let $\bar{x}_i$, $i = 1,...,k$ denote the unique inputs, and $y_i^{(j)}$ be the $j^{th}$ output of $r_i \geq 1$ replicates observed at $\bar{x}_i$. Let $\bar{\mathbf{y}}_{1:k} =\{\bar{y}_i,1 \leq i \leq k\}$ store averages over replicates, $\bar{y}_i := \frac{1}{r_i} \sum_{j=1}^{r_i}y_{i}^{(j)}$. \newstuff{This notation follows the ``unique-n'' formulation proposed by \cite{binois2019replication}, which was shown to reduce the computation cost from $\mathcal{O}(N^3)$ to $\mathcal{O}(n^3)$ compared to the ``full-N'' formulation.} %Overall, we decompose the total budget $N = \sum_{i=1}^k r_i$ across $k$ sites.
	
	The inference of $\partial S$ proceeds by building a metamodel $\hat{f}$, which induces \newstuff{$\hat{S} = \{x \in D: \hat{f}(x) > 0\},$} and evaluating its \emph{error rate} $\mathcal{ER}$, i.e.~the integral over the symmetric difference between $\hat{S}$ and true $S$ weighted by a given measure $\mu(\cdot)$:
	\begin{align}
	{\cal ER}(S,\hat{S}) &=  \int_{x \in D}\!\! \mathbb{I} (\sgn \hat{f}(x) \neq \sgn f(x)) \mu(dx) = \mu(S \Delta \hat{S}), \label{loss}
	\end{align}
	where $S \Delta \hat{S} := (S  \cap \hat{S}^c) \bigcup (S^c \cap \hat{S})$. \newstuff{$S$ can also be defined using {Vorob`ev} expectation \cite{chevalier2014fast} or conservative probability estimate \cite{bolin2015excursion, azzimonti2018estimating}.} %Frequently, the inference is carried out by first producing an estimate $\hat{f}$ of the response function; in that case we take $\hat{S}=\{x \in D:\hat{f}(x) \geq 0\}$) and rewrite the loss  as
	%\begin{align}\label{eq:loss-f}
	%L(f, \hat{f})   &= \int_{x \in D} \mathbb{I} (\sgn \hat{f}(x) \neq \sgn f(x)) \mu(dx),
	%\end{align}
	%where $\mathbb{I}(\cdot)$ is the indicator function. We name it as \emph{Error Rate}  in level set estimation and use uniform distribution for $\mu$.
	
	Reconstructing $S$ via a metamodel can be divided into two aspects: the construction of the response model $x \mapsto Y(x)$, and the development of the design of experiments (DoE) for efficiently selecting the inputs $\bar{x}_1, \bar{x}_2, \ldots$. To account for the second aspect, we use $n$ to denote the rounds of sequential DoE, $k_n$ to denote the number of unique inputs $\bar{x}$'s sampled by step $n$ and $N_n = \sum_{i=1}^{k_n} r_i^{(n)}$ the respective number of simulator calls made. The superscript on $r_i$ allows the replicate counts to evolve over $n$ as well, see Section \ref{sec:csao}. The metamodel training set by step $n$ consists of $\cA_n = \left\{(\bar{x}_i, r_i^{(n)}, \bar{y}_i), 1 \leq i \leq k_n\right\}$.

	The Gaussian process paradigm treats $f$ as a random function whose posterior distribution is determined from its prior  and the training set(s) $\cA_n$. We view $f(\cdot) \sim GP( m(\cdot), K(\cdot,\cdot))$ as a realization of a Gaussian process specified by its mean function $m(x) := \mathbb{E}[f(x)]$ and covariance function $K(x,x') := \mathbb{E}[(f(x)-m(x))(f(x')-m(x'))]$. The noise distribution is {$\eps \sim \mathcal{N}(0, \taun^2)$}; \newstuff{and thus the observation $\bar{y}$ also follows a normal distribution.} For simplicity we take $m(x)=0$. The conditional distribution $f | \cA_n$ is another Gaussian process, with posterior mean $\hat{f}^{(n)}(x_*)$ and covariance $v^{(n)}(x_*,x_*')$ at arbitrary inputs $x_*, x_*'$ given by
	\begin{align}
	\hat{f}^{(n)}(x_*) &=  \mathbf{k}(x_*)[\mathbf{K}+\taun^2\mathbf{R}^{(n)}]^{-1}\bar{\mathbf{y}}_{1:k_n},  \label{mean}\\
	v^{(n)}(x_*,x_*') &=  K(x_*,x_*')-\mathbf{k}(x_*) [\mathbf{K}+\taun^2\mathbf{R}^{(n)}]^{-1}\mathbf{k}(x_*')^T, \label{cov}
	\end{align}
	with the $1 \times k_n$ vector $\mathbf{k}(x_*) = K(x_*, \bar{\mathbf{x}}_{1:k_n})$, the $k_n \times k_n$ matrix $\mathbf{K}$ given by $\mathbf{K}_{ij} = K(\bar{x}_i, \bar{x}_j)$, and the $k_n \times k_n$ diagonal matrix $\mathbf{R}^{(n)}$ given by $R_{ii}^{(n)} :=\frac{1}{r_i^{(n)}}$. The posterior mean $\hat{f}^{(n)}(x_*)$ is treated as a point estimate of $f(x_*)$, and the posterior standard deviation $s^{(n)}(x_*) := \sqrt{v^{(n)}(x_*,x_*)}$ as the uncertainty of this surrogate.
	
	\begin{remark}
				It is also common in practice that the simulators exhibit input-dependent noise, calling for a heteroskedastic metamodel. Given the noise distribution $\eps_i \sim \mathcal{N}(0, \taun(x_i)^2)$ with a known $\taun(\cdot)$, the conditional distribution $f | \cA_n$ is given by
			\begin{align*}
			\hat{f}^{(n)}(x_*) &=  \mathbf{k}(x_*)[\mathbf{K}+\tilde{\mathbf{R}}^{(n)}]^{-1}\bar{\mathbf{y}}_{1:k_n}, \\ % \label{meanhet}\\
			v^{(n)}_(x_*,x_*') &=  K(x_*,x_*')-\mathbf{k}(x_*) [\mathbf{K}+\tilde{\mathbf{R}}^{(n)}]^{-1}\mathbf{k}(x_*')^T,  %\label{covhet}
			\end{align*}
			with the diagonal matrix $\tilde{\mathbf{R}}^{(n)}$ given by $\tilde{R}_{ii}^{(n)} :=\frac{\taun(x_i)^2}{r_i^{(n)}}$. All the batching algorithms proposed in Section \ref{sec:batchdesign} and \ref{sec:csao} naturally extend to the heteroskedastic context if we replace $\taun^2\mathbf{R}^{(n)}$ with $\tilde{\mathbf{R}}^{(n)}$. The main challenge is then to handle estimation of the unknown conditional variance $\taun(\cdot)$, see  e.g.~\cite{ankenman2010stochastic, binois2019hetgp}. The algorithms proposed below have been ported to work with the \texttt{R} \texttt{hetGP} library \citep{binois2019hetgp} that provides an efficient way to jointly learn the mean and variance response surfaces under replicated designs.
	\end{remark}

	\section{Adaptive Designs} \label{sec:batchdesign}
	
	\subsection{Level Set Estimation}
	An adaptive DoE approach is needed to select $\bar{x}_1, \bar{x}_2,\ldots$ sequentially since the level-set $S$ is defined in terms of the unknown $f$. The standard framework of DoE is to add new inputs one-by-one at each round, using  an acquisition function $\im_n(x)$ to pick $\bar{x}_{n+1}$. The acquisition function quantifies the value of information from running a new simulation at $x$ conditional on an existing training set $\cA_n$, and picks $\bar{x}_{n+1}$ as the  myopic maximizer of $\im_n$:
	\begin{align}
	\bar{x}_{n+1} = \arg \sup_{x \in D } \im_{n}(x). \label{seq}
	\end{align}
	Building upon the seminal Expected Improvement criterion~\citep{jones1998efficient},
	various level-set sampling criteria were proposed by ~\citet{bichon2008efficient}, ~\citet{picheny2010adaptive}, ~\citet{bect2012sequential} and ~\citet{ranjan2012sequential}. Further instances of $\im(x)$ can be found in ~\citet{chevalier2013estimating, chevalier2014fast}, ~\citet{azzimonti2015quantifying, azzimonti2020adaptive}, and~\citet{bolin2015excursion}. The basic idea in sequential level-set estimation is to assess the \emph{information gain} from new simulations, targeting the learning of the contour. Most of the above criteria were originally proposed for deterministic experiments with no simulation noise, or cases with known $\taun^2$. We refer to~\citet{lyu2018evaluating} for a summary of level set estimation in stochastic experiments with heteroskedastic $\tau^2(x)$, which can be seen as the counterpart of the earlier study in ~\citet{jalali2017comparison} for Bayesian Optimization with stochastic simulators.
	
	In this section we construct a sequential batched DoE to jointly select $(\bar{x}_{n+1}, r_{n+1})$. At each DoE round we pick a \emph{new} input $\bar{x}_{n+1}$ and the associated replication amount $r_{n+1}$; thus by round $n$ there are $n$ unique inputs.
	%In this section we extend such criteria to joint optimization over $\bar{x}_{n+1}$ and $r_{n+1}$.
	In our first proposal, we formulate this task within a multi-fidelity framework, which is widely used in Bayesian Optimization \citep{kandasamy2016gaussian, kandasamy2016multi, kandasamy2017multi, poloczek2017multi}. Thanks to the LLN, we interpret $r_n$ as \emph{fidelity}:  a small number of replicates  is cheap but inaccurate; inputs with a large number of replicates are viewed as high-fidelity queries: expensive but accurate. Our interest is then to choose the fidelity level to query next, balancing the trade-off between accuracy and cost. As a second proposal, we relate replication to simulation and model fitting overhead costs, leading to maximization of the information gain $\im(x,r)$ per unit cost \citep{klein2017fast, mcleod2017practical}.
	
	\begin{remark}
		Another meaning of batched DoE refers to selecting multiple new inputs $\bar{x}_k$ in parallel, see \citet{chevalier2014fast}. In this article, batching always refers to using replicates; we add (at most) one new input at each DoE round.
	\end{remark}

	To begin, we repurpose two existing acquisition functions well suited to our needs. \newstuff{In our first proposal, we formulate the choice of input $x_{n+1}$ and its replicate count $r_{n+1}$ as two separate steps, which implies that $\mathcal{I}_n$ is only based on the existing information. }
	The first acquisition function is Contour Upper Confidence Bound (cUCB) \citep{lyu2018evaluating} which stems from the Upper Confidence Bound (UCB) strategies proposed by \citet{srinivas2012information} for Bayesian Optimization. cUCB blends the minimization of $|\hat{f}^{(n)}(x)|$ (exploitation) with maximization of the posterior uncertainty $s^{(n)}(x)$ (exploration):
	\begin{align}
	\im_n^{\text{cUCB}}(x) &:= \left\{ -|\hat{f}^{(n)}(x)| + \rho^{(n)} s^{(n)}(x) \right\} \mu(x), \label{ucb}
	\end{align}
	where \newstuff{$\rho^{(n)}$} is a sequence of UCB weights, \newstuff{and $\mu$ is a probability measure on the Borel $\sigma$-algebra $\bm{\mathcal{B}}(D)$ (e.g.,~$\mu=\text{Leb}_D$ the Lebesgue measure on $D$)}. Thus, cUCB targets inputs with high response uncertainty (large $s^{(n)}(x)$), and close to the contour $\partial \hat{S}$ (small $|\hat{f}^{(n)}(x)|$). See~\citet{lyu2018evaluating} on the choice of the UCB weight sequence $\rho^{(n)}$. \newstuffblue{Maximizing $\im_n^{\text{cUCB}}(\cdot)$ yields $x_{n+1}$; see Sections \ref{subsec:MLB} and \ref{subsec:rb} on various ways to select the corresponding $r_{n+1}$.}
	
	%\todo{We follow the recipe in Lyu et al.~\citep{lyu2018evaluating} to use $\gamma^{(n)} = IQR(\hat{f}^{(n)}) /3 Ave(s^{(n)})$ which keeps both terms in \eqref{ucb} approximately comparable as $n$ changes.}
	
	 \newstuff{In the second proposal, we jointly pick $x_{n+1}$ and $r_{n+1}$ in a single step, utilizing a look-ahead criterion.} The gradient Stepwise Uncertainty Reduction (gSUR) criterion focuses on the local empirical error $E_n$ defined by
	\begin{align}
	{E}_n(x)  := \Phi\bigg(-\frac{|\hat{f}^{(n)}(x)|}{s^{(n)}(x)}\bigg). \label{criterionmee}
	\end{align}
	We interpret ${E}_n(x)$ as the local probability of misclassification of $\{ x \in S\}$, see~\citet{bichon2008efficient,echard2010kriging,lyu2018evaluating,ranjan2012sequential}.
	gSUR aims to select the input which produces the greatest \emph{reduction} between  the current $E_n(x)$ given $\cA_n$ and the expected ${E}_{n+1}(x)$ conditional on the one-step-ahead design, $\cA_{n+1} = \cA_n \cup (\bar{x}_{n+1}, r_{n+1}, \bar{y}_{n+1})$.  To do so, gSUR ties the selection of $\bar{x}_{n+1}$ to the look-ahead standard deviation $s^{(n+1)}(x,r)$  at $x$ conditional on $\cA_n$ and sampling $r$ times at $x$. The latter is proportional to the current standard deviation $s^{(n)}(x)$ with the proportionality factor linked to $r$ \citep{chevalier2014corrected}:
	\begin{align}
	\frac{s^{(n+1)}(x, r)^2}{s^{(n)}(x)^2} &= \frac{\frac{\taun^2}{r}}{\frac{\taun^2}{r} + s^{(n)}(x)^2}, \label{varprop2}
	\end{align}
	since the replicated outputs $y^{(j)}_{n+1}$ are i.i.d.. Based on \eqref{varprop2} and using the fact that $\mathbb{E}_{\bar{Y}(x)}[ \hat{f}^{(n+1)}(x) ] = \hat{f}^{(n)}(x)$,  the gSUR metric approximates the effect of $\bar{Y}(x)$ on the look-ahead local empirical error $E_{n+1}(x)$:
	\begin{align}  \label{criterionmeesur}
	\im_n^{\text{gSUR}}(x, r)
	& := \left\{ \Phi\bigg(-\frac{|\hat{f}^{(n)}(x)|}{s^{(n)}(x)}\bigg)- \Phi\bigg(-\frac{|\hat{f}^{(n)}(x)|}{s^{(n+1)}(x,r)}\bigg) \right\} \mu(x) \\
	& \simeq  \left\{ E_n(x) - \mathbb{E}_{\bar{Y}(x)} \left[ E_{n+1}(x) \right] \right\} \mu(x).\nonumber
	\end{align}
	We note that $\im_n^{\text{gSUR}}(x, r) = 0$ for $x \in \partial \hat{S}^{(n)}$ (i.e.~when $\hat{f}^{(n)}(x) = 0$) so that the gSUR metric naturally enforces some exploration by sampling close to, but not exactly at, the estimated contour.

	\subsection{Multi-Level Batching} \label{subsec:MLB}

    \newstuffblue{Having determined $\bar{x}_{n+1}$ via the cUCB criterion $\im^{cUCB}_n$ \newstuff{\eqref{ucb}}, we turn to the task of picking $r_{n+1}$.} The most basic batching strategy is Fixed Batching (FB): $$r_{n+1} \equiv r_0$$ for some pre-specified batching level $r_0$. %{\color{red} Wu \& Frazier \citep{wu2016parallel} discussed parallel knowledge gradient with a FB design  for Bayesian Optimization and show that selecting $r_0$ is a sensitive issue.}
	To improve upon FB, we select $r_{n+1}$  from a discrete set  $\mathbf{r}_L := \{r^1,\ldots,r^L\}$, interpreted as representing $L$ different \emph{sampling fidelities}. Query at $x$ on the $\ell$-th level implies using $r^{\ell}$ replicates to generate observations $y^{(j)}, j = 1,\ldots,r^{\ell}$ yielding the average $\bar{y}$. The cost of the $\ell$-th fidelity is proportional to $r^{\ell}$. The multi-fidelity analogy \citep{kandasamy2016gaussian} is based on the idea of using low/cheap fidelities to explore and then high/expensive fidelities to exploit the desired contour.

In our context, we rely on the look-ahead standard deviation $s^{(n+1)}(\bar{x}_{n+1}, \cdot)$ in \eqref{varprop2}. Our Multi-Level Batching (MLB) Algorithm~\ref{alg:mlb} aims to match $s^{(n+1)}(\bar{x}_{n+1},r_{n+1})$ with a given threshold $\gamma_n$ which acts as the target level for the next-step standard deviation. Intuitively, $\gamma_n$ controls the credibility of the model; it is progressively lowered as the input space is explored. Recall that $r \mapsto s^{(n+1)}(\bar{x}_{n+1},r)$ is monotone decreasing in \eqref{varprop2}; MLB chooses the \emph{highest fidelity} $r_{n+1} \in \bm{r}_L$ for which $s^{(n+1)}(\bar{x}_{n+1},r_{n+1}) >\gamma_n$. If $s^{(n+1)}(\bar{x}_{n+1},r) > \gamma_n$ for all  $r \in \bm{r}_L$ then we use the highest fidelity level $r_{n+1} = r^L$; if $s^{(n+1)}(\bar{x}_{n+1},r) < \gamma_n$ for all $r \in \bm{r}_L$ then we lower the threshold by multiplying $\gamma_n$ by a reduction factor $\eta < 1$, and try to identify $r_{n+1}$ again, cf.~\cite{kandasamy2016gaussian}.

	\begin{algorithm}[tb]
		\caption{Multi-Level Batching (MLB)} \label{alg:mlb}
		\begin{algorithmic}
			\State {\bfseries Input:} $\mathbf{r}_L$, $\eta$, $k_0$, $r_0$
			\State $\mathcal{A}_{k_0} \gets \{(\bar{x}_i, r_0, \bar{y}_i),1 \leq i \leq k_0\}$, $(\hat{f}^{(k_0)}, s^{(k_0)}) \gets f|\mathcal{A}_{k_0}$, $\gamma \gets Ave(s^{(0)}(\bar{x}_{1:k_0}) )$.
			\State $N_{k_0} \gets r_0 \times k_0$.
			\For{$n=k_0,k_0+1,\ldots$}
			\State $\bar{x}_{n+1} \gets \arg\max_{x\in D} \im_{n}^{cUCB}(x)$.
			\While{ $s^{(n+1)}(\bar{x}_{n+1}, r^1) < \gamma$  \Comment{ Check if need to lower threshold} }
			\State $\gamma \gets \eta \times \gamma$.
			\EndWhile
			\State $r_{n+1} \gets \max \{ r\in \mathbf{r}_L  :\, s^{(n+1)}(\bar{x}_{n+1}, r) \geq \gamma\}$.
			\State $\bar{y}_{n+1} \gets \frac{1}{r_{n+1}}\sum_{j =1}^{r_{n+1}} y^{(j)}$.
			\State Update $\mathcal{A}_{n+1} \gets \mathcal{A}_{n} \cup \{(\bar{x}_{n+1}, r_{n+1}, \bar{y}_{n+1})\}$.
			\State Obtain $(\hat{f}^{(n+1)}, s^{(n+1)}) \gets f|\mathcal{A}_{n+1}$.
			\State $N_{n+1} \gets N_n + r_{n+1}$.
			%\State $n \gets n+1$.
			\EndFor
		\end{algorithmic}
	\end{algorithm}

	\subsection{Ratchet Batching} \label{subsec:rb}
	
	By construction, the MLB Algorithm~\ref{alg:mlb} will step back and forth between different replication levels $r^{\ell}$. Since intuitively the design should concentrate as $n$ grows, we expect $r_{n}$ to grow over time which is achieved through the decreasing $\gamma_n$. \newstuffblue{By enforcing that $n \mapsto r_{n}$ is monotonically non-decreasing (in line with the intuition that replication becomes increasingly beneficial as $n$ grows) we can simplify the choice of $r_{n+1}$ and reduce algorithmic overhead. The resulting Ratchet Batching (RB) scheme picks $r_{n+1}$ among just two fidelity levels (compared to $L$ levels in MLB) and is summarized in Algorithm~\ref{alg:rb}.} Let $r_n^\uparrow = \min \{ r \in \bm{r}_L  : r > r_n\}$ be the next level. Then
	RB either keeps $r_{n+1} = r_n$ if $s^{(n+1)}(\bar{x}_{n+1}, r_n) \ge \gamma_n >  s^{(n+1)}(\bar{x}_{n+1}, r_n^\uparrow)$ or increments to $r_{n+1} = r_n^\uparrow$ if $s^{(n+1)}(\bar{x}_{n+1}, r_n)  >  s^{(n+1)}(\bar{x}_{n+1}, r_n^\uparrow) \ge \gamma_n$. In the third case where $s^{(n+1)}(\bar{x}_{n+1}, r_n) < \gamma_n$ we lower the threshold $\gamma_n$ as in MLB. For RB, the reduction factor $\eta$ for $\gamma$ should be close to 1, to avoid excessive ratcheting up. If $\eta$ is not large enough, there is a risk to skip levels in $\bm{r}_L$ and to  end up with
	excessive replication relative to number of simulation calls, leading to insufficient exploration.

\begin{algorithm}[tb]
		\caption{Ratchet Batching (RB)}\label{alg:rb}
		\begin{algorithmic}
			\State {\bfseries Input:} $\mathbf{r}_L$, $\eta$, $k_0$, $r_0$
			\State $\mathcal{A}_{k_0} \gets \{(\bar{x}_i, r_0, \bar{y}_i),1 \leq i \leq k_0\}$, $(\hat{f}^{(k_0)}, s^{(k_0)}) \gets f|\mathcal{A}_{k_0}$, {$\gamma \gets s^{(k_0)}$}.
			\State $N_{k_0} \gets r_0 \times k_0$.
			\For{$n=k_0,k_0+1,\ldots$}
			\State $\bar{x}_{n+1} \gets \arg\max_{x\in D} \im_{n}^{cUCB}(x)$.
			%		\State $r_{n+1} \gets \max \{z, r \geq r_n|s^{(n+1)}(x_{n+1}, r) \geq \gamma\}$.
			\While{$s^{(n+1)}(\bar{x}_{n+1}, r_{n}) <  \gamma \ $} \Comment{Check if need to lower threshold}
			\State $\gamma \gets \eta \times \gamma$.
			%		\State $r_{n+1} \gets \max \{z, r \geq r_n|s^{(n+1)}(x_{n+1}, r) \geq \gamma\}$.
			\EndWhile
			\State $r_n^{\uparrow} \gets \min \{ r \in \mathbf{r}_L : r > r_n \}$
			\State $r_{n+1} \gets r_n \cdot 1_{\{ s^{(n+1)}(\bar{x}_{n+1}, r_n^\uparrow) < \gamma\}} + r_n^{\uparrow} \cdot 1_{\{ s^{(n+1)}(\bar{x}_{n+1}, r_n^\uparrow) \ge \gamma\}}$
			%\State $ r_{n+1} \gets \min \{ r \geq r_n : \, r \in \mathbf{r}_L \,\text{and}\, s^{(n+1)}(\bar{x}_{n+1}, r) \geq \gamma\}$.}
			\State $\bar{y}_{n+1} \gets \frac{1}{r_{n+1}}\sum_{j =1}^{r_{n+1}} y^{(j)}$.
			\State Update $\mathcal{A}_{n+1} \gets \mathcal{A}_{n} \cup \{(\bar{x}_{n+1}, r_{n+1}, \bar{y}_{n+1})\}$.
			\State Obtain $(\hat{f}^{(n+1)}, s^{(n+1)}) \gets f|\mathcal{A}_{n+1}$.
			\State $N_{n+1} \gets N_n + r_{n+1}$.
			\EndFor
		\end{algorithmic}
	\end{algorithm}
		
	\subsection{Adaptively Batched Stepwise Uncertainty Reduction} \label{sec:asur}
	
	The FB, MLB and RB schemes all pick $\bar{x}_{n+1}$ first and then $r_{n+1}$. We next propose a procedure to pick both through a joint criterion optimization. The main idea is to tie the choice of $r_{n+1}$ to \emph{cost}, namely to maximize the ratio of the information gain and the cost of generating $r$ outputs, plus the optimization overhead. The inclusion of the overhead in $\im_n$ comes from \citep{swersky2013multi, klein2017fast,mcleod2017practical} \newstuff{in Bayesian Optimization problems}, where the authors treated the total cost as the sum of query cost $T_{sim}$ and the GP metamodeling overhead $c_{ovh}$. \citet{stroh2017sequential} discussed estimating a probability of exceeding a threshold in a multi-fidelity stochastic simulator, where the input $\bar{x}_{n+1}$ and the fidelity are estimated in a sequential way. We develop an analogue for level-set estimation via	a gSUR-based acquisition function
	\begin{align}
	\im^{ABSUR}_n(x, r) & := \frac{\im^{gSUR}_n(x, r)}{c(r) + c_{ovh}(n)},  \label{asuri}
	%&= \frac{\Phi\bigg(-\frac{|\hat{f}^{(n)}(x)|}{s^{(n)}(x)}\bigg) - \Phi\bigg(-\frac{|\hat{f}^{(n)}(x)|}{s^{(n+1)}(x, r)}\bigg)}{c(r) + c_{ovh}(n, d, \bm{\theta})}, \nonumber
	\end{align}
	where $c_{ovh}(n)$ is the overhead and $c(r) = r \cdot T_{sim}$ is the cost of $r$ evaluations, linear in $r$.  Combining \eqref{asuri} and \newstuff{\eqref{criterionmeesur}}, we obtain
	\begin{align}
	\im^{ABSUR}_n(x, r) :=
	\frac{\Phi\bigg(-\frac{|\hat{f}^{(n)}(x)|}{s^{(n)}(x)}\bigg) - \Phi\bigg(-\frac{|\hat{f}^{(n)}(x)|}{s^{(n)}(x)}\frac{\sqrt{rs^{(n)}(x)^2+\taun^2}}{\taun}\bigg)}{r \cdot T_{sim} + {c}_{ovh}(n)}. \label{absur}
	\end{align}
	The resulting ABSUR Algorithm~\ref{alg:absur} myopically maximizes $\im^{ABSUR}$ over $x \in D$ and $r \in \mathcal{R}= [\underline{r}, \bar{r}]$.
	Intuitively, \newstuff{similar to the gSUR, ABSUR also targets the neighborhood of the zero contour $\partial S$} and the value of $r_{n+1}$ is controlled by $s^{(n)}(x)^2$ and ${c}_{ovh}(n)$; more replication results when $s^{(n)}(x)^2$ is small \newstuff{(neighborhood of the zero contour $\partial S$)} or ${c}_{ovh}(n)$ is large \newstuff{(at a later stage of active learning)}. \newstuffblue{One could replace the numerator in \eqref{absur} with other similar metrics that target reduction of contour uncertainty \citep{lyu2018evaluating}.}

	\begin{algorithm}[htb]
		\caption{Adaptive Batched SUR (ABSUR)}\label{alg:absur}
		\begin{algorithmic}
			\State{\bfseries Input:} {$\mathcal{R}= [\underline{r}, \bar{r}]$, $k_0$, $r_0, T_{sim}$, overhead cost function $n \mapsto c_{ovh}(n)$}
			\State $\mathcal{A}_{k_0} \gets \{(\bar{x}_i, r_0, \bar{y}_i),1 \leq i \leq k_0\}$, $(\hat{f}^{(k_0)}, s^{(k_0)}) \gets f|\mathcal{A}_{k_0}$
			\State  $N_{k_0} \gets r_0 \times k_0$
			\For{$n=k_0,k_0+1,\ldots$}
			\State $(\bar{x}_{n+1}, r_{n+1}) \gets \arg \sup_{x\in D, r \in \newstuff{\mathcal{R}}}  \im^{ABSUR}_{n}(x, r)$.
			\State $\bar{y}_{n+1} \gets \frac{1}{r_{n+1}}\sum_{j =1}^{r_{n+1}} y^{(j)}$.
			\State Update $\mathcal{A}_{n+1} \gets \mathcal{A}_{n} \cup \{(\bar{x}_{n+1}, r_{n+1}, \bar{y}_{n+1})\}$.
			\State Obtain $(\hat{f}^{(n+1)}, s^{(n+1)}) \gets f|\mathcal{A}_{n+1}$.
			\State $N_{n+1} \gets N_n + r_{n+1}$.
			\EndFor
		\end{algorithmic}
	\end{algorithm}

	There are four hyperparameters in ABSUR: the simulation cost $T_{sim}$, the overhead cost function $c_{ovh}(n)$ and the lower/upper bounds of replication $[\underline{r}, \bar{r}]$. For $c_{ovh}(n)$ we follow the recipe in~\citep{mcleod2017practical}, modeling it as a quadratic function of $n$ to reflect the prediction complexity of GPs:
	\begin{align}\label{c-over}
	{c}_{ovh}(n; \bm{\theta}) &= \theta_0 + \theta_1n + \theta_2n^2, % + \epsilon,\quad \epsilon \sim \mathcal{N}(0, \sigma_c^2).
	\end{align}
	where $\bm{\theta}$ are fitted empirically.
	Alternatively~\citet{klein2017fast} kept $c_{ovh}(n)$ as a constant. The constant $T_{sim}$ represents the cost of obtaining each observation, \newstuff{measured in the same units as $c_{ovh}(n)$ (up to rescaling $\bm{\theta}$, we can assume $T_{sim}=1$).}  If simulations are cheap, we would like to replicate more, and indeed lower $T_{sim}$ leads to larger $r_n$'s and therefore smaller designs. %\newstuff{For example, in simulations shown in Section \ref{sec:synthetic} where the simulations are cheap ($T_{sim}$ less than ~1 second), ABSUR leads to smaller designs compared with MLB and RB.}
\newstuffblue{This feature implies that ceteris paribus $T_{sim}$ should be set larger when input spaces are more voluminous, e.g.~in higher-dimensional settings.}

	\section{Adaptive Design with Stepwise Allocation} \label{sec:csao}
	
	The four strategies (FB, MLB, RB and ABSUR) discussed in Section \ref{sec:batchdesign} visit each input site $\bar{x}_{n+1}$ only once. Consequently, the respective replicate count $r_{n+1}$ is determined at step $n+1$ and then remains the same throughout the latter steps. As an alternative, one can sequentially \emph{allocate} new simulations across existing designs, thereby gradually growing $r_i^{(n)}$. Namely, the algorithm identifies existing ``informative'' inputs and augments their replicate counts, without changing the number of unique inputs $k_n$ across the sequential design rounds $n$.
	In our context, we pair this augmentation with the option of expanding the design set itself. This choice is similar to the classical exploitation (do not change $k_n$) versus exploration (increase $k_n$). The resulting ADSA approach resembles \emph{Stepwise Approximate Optimal Design (SAO)}, an IMSE-based sequential design strategy proposed by~\citet{chen2017sequential} for mean response prediction.
	
	At each step $n$ of the ADSA strategy we are given a budget of $\Delta r^{(n)}$ additional simulations, and the main decision is to determine whether we should choose a new input $\bar{x}_{k_n+1}$ that then receives all these $\Delta r^{(n)}$ replicates, or we should allocate the $\Delta r^{(n)}$ new simulator calls across the existing inputs $\bar{\mathbf{x}}_{1:k_n}$. In the latter case, we aim to minimize the global look-ahead integrated contour uncertainty $\mathcal{L}^{(n+1)}$ where the metric $\mathcal{L}^{(n)}$ is defined by
	\begin{align}\label{eq:adsa-l}
	\mathcal{L}^{(n)} & :=  \sum_{j=1}^M \omega^{(n)}_j \hat{f}^{(n)}(x_{j,*}) = (\bm{\omega}^{(n)})^T \mathbf{f}^{(n)}_* \simeq \int_D \Phi( - \hat{f}(x)/ s^{(n)}(x) ) \hat{f}^{(n)}(x) \mu(dx),
	\end{align}
	where ${\mathbf{x}}_*= x_{1,*},\ldots,x_{M,*}$ is a test set of size $M$, $\mathbf{f}^{(n)}_* \equiv \hat{f}( \mathbf{x}_*)$ is the vector of  predicted responses at $\mathbf{x}_*$, and $\omega^{(n)}_j \equiv \omega( x_{j,*})\mu(x_{j,*}) = \Phi( - \hat{f}^{(n)}(x_{j,*})/ s^{(n)}(x_{j,*}) ) \mu(x_{j,*})$ are the weights that target the level-set region of interest (compare to the targeted integrated mean square error (tIMSE) criterion proposed by~\citet{picheny2010adaptive}).
	
	For allocation purposes, we approximate the look-ahead $\mathcal{L}^{(n+1)}$ as a linear combination of the $M$ predictions $\hat{f}^{(n+1)}(x_{j,*})$ with \emph{fixed} \newstuff{weights} $\bm{\omega}^{(n)}$, whereby our goal is to minimize the variance of $ (\bm{\omega}^{(n)})^T \mathbf{f}^{(n+1)}_*$ conditional on the extra allocations $\Delta r_i^{(n)}$ at each input $\bar{x}_i$. Since the covariance matrix of $\mathbf{f}^{(n+1)}_*$ given replication counts $\mathbf{R}^{(n+1)}$ is 	
	\begin{align}
	\mathbf{C}^{(n+1)} &= \mathbf{k}({\mathbf{x}}_*,  {\mathbf{x}}_*) -  \mathbf{k}({\mathbf{x}}_*, \bar{\mathbf{x}}_{1:k_n})(\bm{K} + \taun^2 \mathbf{R}^{(n+1)})^{-1} \mathbf{k}({\mathbf{x}}_*, \bar{\mathbf{x}}_{1:k_n})^{T}  \label{ccov}
	\end{align}
	the objective becomes the quadratic program \newstuff{that minimizes}
	\begin{align}
	\im_{SAO}( (\Delta r_i)_{i=1}^{k_n} ) =(\bm{\omega}^{(n)})^T {\mathbf{C}}^{(n+1)} \bm{\omega}^{(n)} %\mapsto \min!
	\label{tmse}
	\end{align}
	under the constraint $\sum_{i} \Delta r_i^{(n)} = \Delta r^{(n)}$.

	Define the $k_n \times k_n$ matrix $\bm{\Sigma}^{(n)} = \mathbf{K} + \taun^2 \mathbf{R}^{(n)}$ and the $M \times k_n$ matrix $\mathbf{K}_* := \mathbf{K}({\mathbf{x}}_*, \bar{\mathbf{x}}_{1:k_n}) $.
	The next proposition, proven in Section~\ref{sec:deltar-proof}, explains how to pick $\Delta r_i^{(n)}$'s to minimize \eqref{tmse}.
	
	\begin{thm} \label{thm:deltar}
		Let $\Delta \mathbf{R}^{(n)} := \mathbf{R}^{(n)} - \mathbf{R}^{(n+1)}$ be a $k_n \times k_n$ diagonal matrix with elements $\Delta \mathbf{R}_{ii}^{(n)}  = \frac{\Delta r_i^{(n)}}{(r_i^{(n)} + \Delta r_i^{(n)})r_i^{(n)}} = [r_i^{(n)}]^{-1} - (r_i^{(n)} + \Delta r_i^{(n)})^{-1}$, $i = 1,\ldots,k_n$. Assume $\max_{i = 1,\ldots,k_n} \Delta \mathbf{R}_{ii}^{(n)}  \ll 1$. The optimal allocation rule that minimizes \eqref{tmse} is to assign $\Delta r_i^{(n)}$ to each $\bar{x}_i$ such that
		\begin{align}
		r_i^{(n)} + \Delta r_i^{(n)} \propto \mathbf{U}^{(n)}_i, \label{newri}
		\end{align}
		where
		\begin{align}
		\mathbf{U}^{(n)} = (\bm{\Sigma}^{(n)})^{-1}\mathbf{K}_*^{T} \bm{\omega}^{(n)} .\label{ui}
		\end{align}
	\end{thm}

\begin{algorithm}[htb!]
		\caption{Adaptive Design with Stepwise Allocation (ADSA)}\label{alg:csao}
		\begin{algorithmic}
			\State{\bfseries Input:} $\bar{\mathbf{x}}_*$, $\bar{\mathbf{x}}_{1:k_0}$, $k_0$,  $r_0$, $c_{bt}$
			\State $\mathcal{A}_{k_0} \gets \{(\bar{x}_{i}, r_0, \bar{y}_{i}), i = 1,...,k_0\}$. $(\hat{f}^{(k_0)}, s^{(k_0)}) \gets f|\mathcal{A}_{k_0}$, $N_0 \gets r_0 \times k_0$.
			\For{$n = k_0, k_0+1, \ldots$}
			%\WHILE{}
			\State $\Delta r^{(n)} \gets c_{bt} \sqrt{n}$.
			\State Calculate allocations $\Delta r_i^{(n)}, 1 \leq i \leq k_n$ with Algorithm \ref{alg:pegging} \newstuff{(see Appendix \ref{appx:pegging})}.
			\State $\bar{x}_{k_n+1} \gets \arg\max_{x\in D} \im^{cUCB}_{n}(x, \Delta r^{(n)})$.
			\State Calculate $\mathcal{I}_{SAO}^{(n)-all}, \mathcal{I}_{SAO}^{(n)-new}$ in \eqref{timse0} and \eqref{timse1}.
			%		\IF{$\Delta \mathcal{I}_{n+1} < 0$}
			\State \textbf{Case 1:}
			 \Indent
				\State New $\bar{y}_{k_n+1} \gets \frac{1}{\Delta r^{(n)}}\sum_{j =1}^{\Delta r^{(n)}} y^{j}(\bar{x}_{k_n+1})$.
				\State Update $\mathcal{A}_{n+1} \gets \mathcal{A}_{n} \cup \{(\bar{x}_{k_n+1}, \Delta r^{(n)}, \bar{y}_{k_n+1})\}$.
				\State $N_{n+1} \gets N_n + \sum_i \Delta r^{(n)}_i$ (May not be exactly $\Delta r^{(n)}$).
				\State $k_{n+1} \gets k_n + 1$.
			\EndIndent
			%		\ELSE
			\State \textbf{Case 2:}
			\Indent
				\State For $i = 1,...,k_n$, update $\bar{y}_i \gets \frac{\bar{y}_i \times  r_i^{(n)} + \sum_{j =1}^{\Delta r_i^{(n)}} y^{j}(\bar{x}_{i})}{r_i^{(n)} + \Delta r_i^{(n)}}$, $r_i^{(n+1)} \gets r_i^{(n)} + \Delta r_i^{(n)}$
				\State Update $\mathcal{A}_{n+1} \gets \{(\bar{x}_{i}, r_{i}^{(n+1)}, \bar{y}_{i})\}_{i = 1,\ldots,k_n}$.
				\State $N_{n+1} \gets N_n + \sum_{i=1}^{k_n} \Delta r_i^{(n)}$
				\State $k_{n+1} \gets k_{n}$
			\EndIndent
			\State Obtain $(\hat{f}^{(n+1)}, s^{(n+1)}) \gets f|\mathcal{A}_{n+1}$.
			\State ADSA: \textbf{Do} Case \textbf{1 if} $\mathcal{I}_{SAO}^{(n)-all} > \mathcal{I}_{SAO}^{(n)-new}$,  \textbf{otherwise do} Case \textbf{2}
			\State \{FDSA variant:\} \textbf{Do} Case \textbf{2}.
			\State \{DDSA variant:\} \textbf{Do} Case \textbf{1} if $n$ is odd, Case \textbf{2} if $n$ is even.
			\EndFor
		\end{algorithmic}
	\end{algorithm}
	
	After obtaining the allocations $\Delta \mathbf{r}^{(n)}_{1,\ldots,k_n}$, we compute the resulting look-ahead tIMSE metric:
	\begin{align}
	\mathcal{I}_{SAO}^{(n)-all} & := \sum_{j=1}^{M} \tilde{s}^{(n+1)}(x_{j,*})^2 \omega^{(n)}_j, \label{timse1}
	\end{align}
	where the look-ahead variance $\tilde{s}^{(n+1)}(\cdot)^2$ is based on the new replicate counts $r_i^{(n+1)} = r_i^{(n)} + \Delta r_i^{(n)}, i= 1,\ldots,k_n$, see proof in \citep{chevalier2014corrected, hu2015sequential}:
	\begin{align}
	\tilde{s}^{(n+1)}({\mathbf{x}}_*)^2 &= s^{(n)}({\mathbf{x}}_*)^2 - \mathbf{k}_* (\bm{\Sigma}^{(n)})^{-1}\Delta \mathbf{R}^{(n)} (\bm{\Sigma}^{(n)})^{-1}\mathbf{k}_*^{T}.
	\end{align}
	
	The alternative to allocating over existing $\bxbar$ is to pick a new input $\xnew$ and assign it $\Delta r^{(n)}$ simulations. To do so, we use the cUCB criterion to make it consistent with FB, MLB and RB. (Other acquisition functions can also be used and experiments suggest that the algorithm is not sensitive to this choice.) Then we evaluate the resulting  $\im_{SAO}^{(n)-new}$ :
	\begin{align}
	\mathcal{I}_{SAO}^{(n)-new} & := \sum_{j=1}^M s^{(n+1)}(x_{j,*}, \Delta r^{(n)})^2 \omega^{(n)}_j, \label{timse0} \\
	\nonumber s^{(n+1)}({{x}}_{j,*}, \Delta r^{(n)})^2
	&= s^{(n)}({{x}}_{j,*})^2 - \frac{v^{(n)}(x_{j,*}, \bar{x}_{k_n+1})^2}{\frac{\taun^2}{\Delta r^{(n)}}+s^{(n)}(\bar{x}_{k_n+1})^2}.
	\end{align}
	The sums in \eqref{timse1}-\eqref{timse0} are used as approximations of the underlying integrals over $x \in D$.
	%In our implementation,  integrals in equations \eqref{timse0} and \eqref{timse1} are replaced with sums over  a test set $\bar{\mathbf{x}}_*$.
	Finally, we compare $\im_{SAO}^{(n)-new}$ and $\im_{SAO}^{(n)-all}$ to determine whether to sample at the new $\bar{x}_{k_n+1}$ or to allocate to existing  $\mathbf{x}_{1:k_n}$, picking the maximum of the two tIMSE metrics.
	
	For FB, MLB, RB and ABSUR, as we select one new input at each step, we have $k_n = n$. However, for ADSA we either select a new input or re-allocate, so that the resulting design size satisfies $k_n < n$. Thus, relative to the earlier schemes, in ADSA the size of $\cA_n$ and the number of DoE rounds $n$ are no longer deterministically linked and the number of unique inputs is endogenous to the particular algorithm run.
	
	A major goal of all our schemes is for $k_n$ to grow sub-linearly in $n$, i.e.~new inputs are added less frequently as more simulations are run. \newstuffblue{
There are two reasons for this: (1) As $k_n$ grows, the input space is better explored and one should favor exploitation more and more; (2) the GP overhead increases in $k_n$ so that each decision becomes more costly and therefore large batches are preferable. Put another way, $k_n \propto n$ is equivalent to fixed batching $\bar{r} = n/k_n$ and we wish for $r_n$ to grow (at least on average) in $n$.} In ADSA, we organically prefer re-allocation over adding inputs as $n$ grows. The user can further \newstuff{enhance} this situation by making the batches $\Delta r^{(n)}$ also grow in $n$.
	Specifically, we have found a good heuristic in taking $\Delta r^{(n)}$ to be proportional to $\sqrt{n}$ (see proportionality constant $c_{bt}$ in Algorithm \ref{alg:csao}), which is faster compared to constant batch sizes and more accurate than making $\Delta r^{(n)}$ linear in $n$ which is overly aggressive.

	\newstuff{\textbf{Deterministic DSA.} In practice we observe that the ADSA scheme tends to alternate roughly equally between re-allocation and addition of new inputs. To save computational overhead, we consider the simplified \emph{Deterministic Design with Stepwise Allocation} (DDSA) scheme that deterministically alternates between re-allocation and adding inputs, making $k_n = k_0 + \lceil(n-k_0)/2\rceil$ also deterministic. Observe that DDSA no longer needs to evaluate the expensive $\im_{SAO}^{(n)-all}$ and $\im_{SAO}^{(n)-new}$.} %A different shortcut is \emph{Fixed Design Stepwise Allocation} (FDSA) which avoids exploration altogether and keeps $k_n=K$ constant by starting immediately with a large initial design $|\mathcal{A}_0|=K$. FDSA thus always uses re-allocation, aiming to grow the number of replicates for inputs in the neighborhood of the contour. We find that the performance of FDSA is quite sensitive to the choice of the initial inputs, and $k_0$ needs to increase exponentially with dimension $d$.

	\section{Results} \label{sec:synthetic}
	
	\subsection{Synthetic Experiments and Computational Implementation Details} \label{subsec:implementation}
	
	\begin{table}[t]
		\caption{Parameters for the 2-D modified \emph{Branin-Hoo} and the 6-D modified \emph{Hartman} experiments.}
		\label{tbl:experiments}
%		\vskip 0.15in
		\begin{center}
				\begin{tabular}{llll}
					\toprule
					& \textsc{Parameter} & 2-D \emph{Branin-Hoo} & 6-D \emph{Hartman} \\ \hline
					Simulation budget & $N_T$ & 2000 & 6000 \\
					Initial design size & $k_0$ & 20 & 60 \\
					Initial replicates & $r_0$ & 10 & 10 \\
					ADSA test set in \eqref{eq:adsa-l} & $M$ & 500 & 1000 \\
					%\multirow{2}{*}{$\mathbf{r}_L$} & $\{5, 10, 15, 20, 30,  40,50, 60, 80\}$ & \\
					%\multirow{2}{*}{$\mathbf{r}_L$} & $\{5, 10, 15, 20, 30,$ & $\{5, 10, 15, 20, 30,$ \\
					%& $40,50, 60, 80\}$ & $40,50, 60, 80\}$ \\
					\multirow{2}{*}{Replication levels} & \multirow{2}{*}{$\mathbf{r}_L$} & \multicolumn{2}{l}{$[5, 10, 15, 20, 30, 40, 50, 60, 80, $} \\
					& & \multicolumn{2}{r}{$100, 140, 180, 240, 300]$} \\
					{ABSUR  replication range} & $\mathcal{R}$ & $[5, 200]$ & $[5, 300]$ \\
					%$\bm{\theta}$ &  $[0.4, 0.002, 1.9\times 10^{-6}]$ & $[1.2, 0.007, 5.9\times 10^{-6}]$ \\
					{ABSUR  simulation cost} & $T_{sim}$ & 0.01 & 0.05 \\
					%$r^{(n)}$ in FB & 10 & 10\\
					{ABSUR  overhead cost in \eqref{c-over}} & $c_{ovh}(n)$ & \multicolumn{2}{c}{$\bm{\theta} = [0.137, 8.15\times 10^{-4}, 1.99\times 10^{-6}]$} \\
					{ADSA  batch factor} & $c_{bt}$ & 10 & 3.33 \\
					\bottomrule
				\end{tabular}
		\end{center}
%		\vskip -0.1in
	\end{table}
	
	In this section we benchmark the schemes on three synthetic case studies,
	employing rescaled \emph{Branin-Hoo} ($d=2$) and \emph{Hartman} ($d=6$) functions. We make linear transformations to the standard setups in order to rescale the output to $[-1,1]$ and have the zero-contour ``in the middle'' of the input space. For the Branin-Hoo case, we further restrict and rescale the original domain to make $f$ monotone along $x^1$ and to generate a single zero-contour curve. %\newstuff{The modified Branin-Hoo function replaces the little "bumps" at the diagonal with one smooth zero-contour curve and is more suitable for level set estimation problem.}
Full specifications are provided in the Online Supplement, see also~\cite{lyu2018evaluating}. The 2-D case studies with the Branin-Hoo response function employ two noise settings: (i) Gaussian $\epsilon \sim \mathcal{N}(0, 1)$; and (ii) heteroskedastic Student-$t$ where the distribution of $\epsilon$ is input-dependent: {$\epsilon(x) \sim t_{6-4x^1}(0,(0.4(4x^1+1))^2)$}. The latter setting is to test the influence of noise mis-specification. The third case study is in 6-D using the Hartman response and noise $\epsilon \sim \mathcal{N}(0, 1)$.

	The squared-exponential kernel $$
	K_{\text{se}}(x,x') := \sse^2\exp\bigg(-\sum_{i=1}^d\frac{(x^i-x'^i)^2} {2\ell_{i}^2}\bigg)$$ is used throughout as the GP covariance function. The covariance hyperparameters $\bm{\vartheta} = \{\ell_1, \ldots, \ell_d, \sse^2\}$ are estimated via MLE using {the \texttt{fmincon} optimizer in \texttt{MATLAB}}. We re-fit $\bm{\vartheta}$ every five DoE steps and otherwise treat it as fixed across $n$. The noise variance is taken to be known (i.e.~$\taun = 1$) in the first and third case studies. It is fitted (as an unknown constant) along with $\bm{\vartheta}$ for the experiments with Student-$t$ simulation noise.

\newstuff{For the 2-D case study the metrics $\cal{ER}$, $\mathcal{I}_{SAO}^{(n)-all}$, and $\mathcal{I}_{SAO}^{(n)-new}$ are computed as an equally weighted average over test points constructed using Latin Hypercube Sampling over the entire input space. In the 6-D case study we pick 80\% of the test points from the region $\{ x \in D : |f(x)| < 0.7\}$ that is close to the zero-contour and the remaining 20\% from the rest of the input space; the respective weights to compute the metrics are based on the volume of the former region. The same setup was used in \cite{lyu2018evaluating}; see also \cite{chevalier2014fast} for a detailed comparison between different sampling methods.}

	We use FB with batch size $r \equiv 10$ as a baseline, and compare the performance of MLB, RB, ABSUR, ADSA and DDSA. Performance is based on the error rate $\cal{ER}$ in \eqref{loss}, i.e.~evaluating (numerically, using a test set of size $M$) the symmetric difference between the true and estimated level set. This is done at a fixed simulation budget $N_T$, i.e.~each scheme is run for $k_T$ rounds  until $N_{k_T}=N_T$ the budget is exhausted. Note that the resulting number of DoE rounds $k_T$ will vary scheme-by-scheme and potentially run-by-run. We index $N_n,k_n$ by the DoE sequential iterations, while $N_T, k_T$ are indexed by total budget consumed. Table~\ref{tbl:experiments} provides further details about the parameters specific to each scheme. To optimize the various $\im$ acquisition functions we use a global, gradient-free,  genetic optimization approach as implemented in the \texttt{ga} function in \texttt{MATLAB}, with tolerance of $10^{-3}$ and $200$ generations.
	
	We fit all  the Gaussian Process surrogates  using the \texttt{GPstuff} suite in \texttt{MATLAB} \citep{vanhatalo2013gpstuff}. For easier reproducibility, our supplementary material contains \texttt{R} code, including the adaptive batching heuristics, to reproduce Figure~\ref{fig:amput2d} below. %The other examples require more computing time in \texttt{MATLAB}, making reproduction challenging. However,
	We are happy to provide the \texttt{MATLAB} codes upon request.

\newstuff{The proposed adaptive batching strategies are not limited to the vanilla GP setup. Other metamodels can be straightforwardly substituted as long as they allow to efficiently evaluate the $\im_n$ criteria and the batch look-ahead variance $s^{(n+1)}(x, r)$. As an illustration, motivated by the non-Gaussian simulation noise  in the second case study and the option pricing application in Section~\ref{sec:bermudan}, we implement a GP metamodel with Student-$t$ observation noise (henceforth $t$-GP). In the $t$-GP formulation $\epsilon_i$ in \eqref{fundamental}  is taken to be $t$-distributed with variance ${\taun^2}$ and $\nu > 2$ degrees of freedom. \citet{lyu2018evaluating} showed that $t$-GP is a good choice in the face of noise misspecification. Appendix \ref{subsec:t} provides details of using a $t$-GP metamodel via a Laplace approximation approach.} \newstuffblue{Our schemes are moreover ported to work with the \texttt{hetGP} \citep{binois2019hetgp} in \texttt{R}, see Table~\ref{tbl:syntheticresultshet} below.}

	\newstuff{
	\subsubsection{Algorithm Tuning Parameters}
In this section we briefly describe the various tuning parameters associated with the proposed algorithms.
	For the UCB weight sequence $\rho^{(n)}$ in cUCB, we follow the recipe in~\cite{lyu2018evaluating} and set $\rho^{(n)} = IQR(\hat{f}^{(n)}) /3 Ave(s^{(n)})$ which keeps both terms in \eqref{ucb} approximately stable as $n$ changes. For MLB, we initialize $\gamma$ as the average standard deviation $Ave(s^{(k_0)}(\bar{x}_{1:k_0}))$ and take the reduction factor $\eta = 0.5$. For RB we use the same initial $\gamma$ but decrement it slower, $\eta = 0.8$. Higher $\eta$ increases the overall design size $k_T$ and therefore computation time. For MLB, $\eta \in [0.5, 0.7]$ leads to the lowest error rate $\cal{ER}$; for RB, we recommend $\eta \in [0.7, 0.9]$. For the replication levels $\mathbf{r}_L$ used in MLB and RB, we manually construct a ``ladder'' of $r^\ell$'s with spacing that increases roughly proportionally. In our experience, the choice of spacings (i.e.~number of levels $L$) does not play a major role, with the most important parameter of $\mathbf{r}_L$ being its upper bound $r^L$. If $r^L$ is too low, the gains from replication are limited; if $r^L$ is too high we observe over-exploitation with a design that does not have enough unique inputs.

For ABSUR, we recommend minimal replication level $\underline{r}$ of 5 or 10, and maximum replication of $\bar{r}= 0.05N_T$, i.e.~$5\%$ of the total budget $N_T$. Table~\ref{tbl:absur-rbar} in Appendix~\ref{sec:tuning-adsa} shows the impact of varying $\bar{r}$ from $1\%$ to $100\%$ of $N_T$. Unsurprisingly, increasing $\bar{r}$ decreases the design size $k_T$ and computation cost $t$. Note that because the scheme tries to optimize actual $r_n$ in the interval $[\underline{r}, \bar{r}]$, for very large $\bar{r}$ that constraint is not binding and so the impact is minimal, see last few rows in Table~\ref{tbl:absur-rbar}. In the middle of its range, the role of $\bar{r}$ is similar to that of $r^{L}$ for MLB and RB.

The coefficients $\bm{\theta}$ in the quadratic overhead function $c_{ovh}(n)$ in \eqref{c-over}, as well as the simulation cost $T_{sim}$ are pre-tuned via a linear least squares regression with the given simulator and hardware setup. Thus, they are not really tuning parameters, but reflect the relative computational effort between regression and simulation. Nevertheless, to provide some intuition, the right panel of Table~\ref{tbl:absur-rbar} shows the impact of changing $T_{sim}$ for one of our experimental setups. Higher $T_{sim}$ encourages exploration. Thus, to avoid too much exploitation and very high $r_n$'s we recommend not to make $T_{sim}$ too small; in our experiments this translates to $T_{sim} \in [0.01, 1]$.

	For the batch factor in ADSA and DDSA we take $c_{bt} = 20/d$, which favors more exploration in higher-dimensional problems with larger input domains. Table~\ref{tbl:adsa-cbt} in Appendix~\ref{sec:tuning-adsa} shows the effect of changing $c_{bt} \in [10/d,  80/d]$.  For both algorithms the design size $k_T$ decreases as $c_{bt}$ increases. However, the change in $k_T$, as well as in the error rate $\cal{ER}$ for DDSA is more significant than for ADSA, especially when simulation noise is low. %with a smaller $c_{bt}$, ADSA favors reallocation over selecting a new input, and vice versa.
	DDSA achieves lower $\cal{ER}$ with a smaller $c_{bt}$, while ADSA has a lower error rate with $c_{bt}$ lower than $20/d$.
}

\newstuffblue{A benefit of working with simulation batches is that the related computation is trivially parallelizable. Like all sequential methods, our schemes cannot be run fully in parallel, since the choice of $x_{n+1}$ must be done one-by-one. Nevertheless, assuming that most time is spent on simulation, distributing those across several computing cores will generate substantial savings that are not possible without batching. To maximally leverage this, one should set $r_n$ to be a multiple of the available number of cores. In the examples below we do not employ any parallelization.}
	
	\subsection{Algorithm Performance} \label{subsec:synthetic}

	\begin{table}[htb!]
		\caption{Scheme performance across the two  synthetic case studies. Results are means ($\pm$ standard deviations) from 50 runs of each combination of a metamodel and batching scheme.
			%
			%of GP and $t$-GP metamodels with FB, MLB, RB, ABSUR, ADSA and DDSA designs in 2-D \emph{Branin-Hoo} synthetic experiments. Results are based on 50 runs.}
			\label{tbl:syntheticresults}}
		\begin{center}
%			\begin{small}
				\begin{sc}
					\begin{tabular}{lrrr}
						\toprule
						Design &  Error Rate $\mathcal{ER}_T$& Time/s & Ave~$k_T$\\
						\midrule
						\multicolumn{4}{c}{\emph{2-D Branin-Hoo} with $\epsilon \sim \mathcal{N}(0, 1)$} \\
						%\midrule
						%\multicolumn{5}{c}{} \\
						\midrule
						FB &  0.019  $\pm$  0.005 & 118.89  & 200.00 \\
						ABSUR    &  0.021   $\pm$ 0.007 &  10.32 & 35.20 \\
						RB    &  0.021 $\pm$   0.008  & 8.30  & 38.72 \\
						MLB     &  \newstuff{\textbf{0.018  $\pm$  0.008}} &  8.63 & 38.44 \\
						ADSA &  0.020 $\pm$ 0.008 & 14.11 & 34.42 \\
						DDSA &  0.022 $\pm$ 0.007 & \newstuff{\textbf{7.92}}  & 37.00 \\
						\midrule
						\multicolumn{4}{c}{\emph{6-D Hartman} with $\epsilon \sim \mathcal{N}(0, 1)$ and $N_T = 6000$} \\
						%\midrule
						%\multicolumn{5}{c}{$\epsilon \sim \mathcal{N}(0, 1)$} \\
						\midrule
						FB &  0.030  $\pm$  0.004 & 1934.51 & 600.00\\
						ABSUR    & 0.070   $\pm$ 0.015 &  289.52 & 159.80 \\
						RB     & 0.058 $\pm$   0.014  & 104.68 & 143.40 \\
						MLB     & \newstuff{\textbf{0.037   $\pm$ 0.008}} &  294.49  & 240.62 \\
						ADSA  & 0.043 $\pm$ 0.007 & 198.82 & 171.74 \\
						DDSA  & 0.050 $\pm$ 0.009 & \newstuff{\textbf{101.59}}  & 142.00\\
						\midrule	
						\multicolumn{4}{c}{\emph{6-D Hartman} with $\epsilon \sim \mathcal{N}(0, 1)$ and $N_T = 30000$} \\
						\midrule
						FB $r_n = 50$ &  0.015  $\pm$  0.002 & 1654.32 & 600.00\\
						FB $r_n = 100$ &  0.016  $\pm$  0.002 & 461.57  & 330.00\\
						FB $r_n = 200$ &  0.029  $\pm$  0.006 & \newstuff{\textbf{152.21}} & 195.00\\
						ABSUR    &  0.022   $\pm$ 0.003 &  757.18 & 325.25 \\
						RB    &  0.024 $\pm$   0.005  & 227.01 & 237.05 \\
						MLB   &  0.022   $\pm$ 0.006 &  240.61  & 242.95 \\
						ADSA & \newstuff{\textbf{0.016 $\pm$ 0.002}} & 995.57  & 373.80 \\
						DDSA  & 0.017 $\pm$ 0.002 & 522.00  & 350.00 \\
						\bottomrule
					\end{tabular}
				\end{sc}
	%		\end{small}
		\end{center}
	%	\vskip -0.1in
	\end{table}

	\begin{figure*}[ht]
		\begin{tabular}{ccc}
			\begin{minipage}[t]{0.33\linewidth}
				\begin{center}
					\includegraphics[width=1\textwidth]{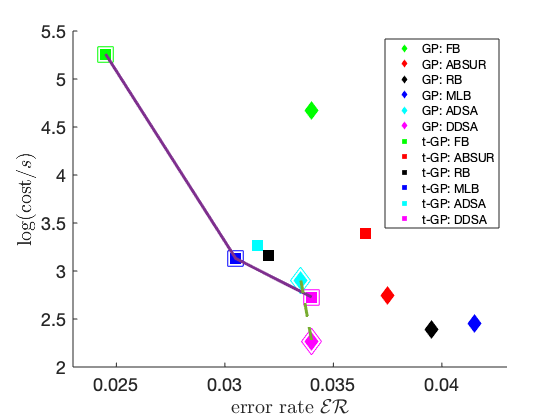}
				\end{center}
			\end{minipage} &
			\begin{minipage}[t]{0.33\linewidth}
				\begin{center}
					\includegraphics[width=1\textwidth]{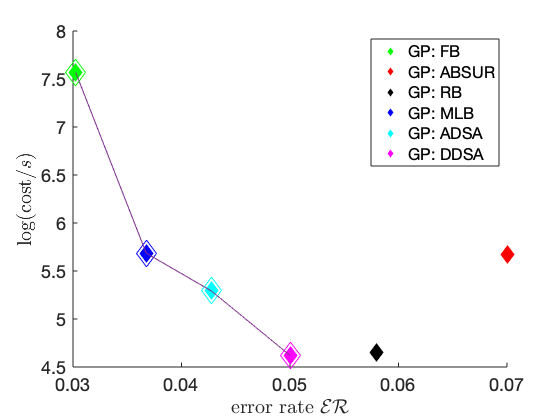}
				\end{center}
			\end{minipage} &
			\begin{minipage}[t]{0.33\linewidth}
				\begin{center}
					\includegraphics[width=1\textwidth]{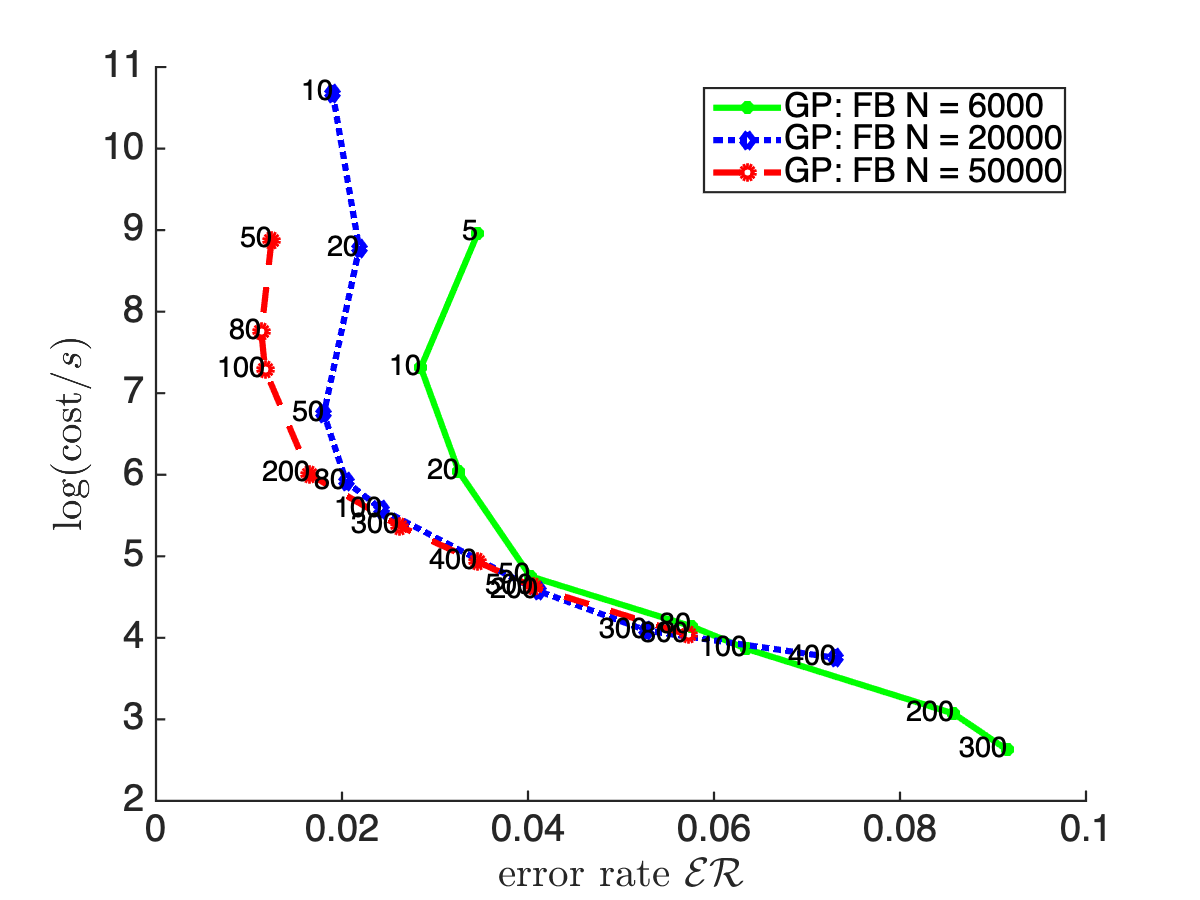}
				\end{center}
			\end{minipage} \\
			2-D Branin-Hoo & 6-D Hartman &  FB Comparison \\
		\end{tabular}
		\caption{Running time and ultimate  error rate $\mathcal{ER}_T$ across different schemes. \emph{Left panel:} 2-D Branin-Hoo with heteroskedastic noise and budget $N_T = 2000$. \emph{Middle panel:} 6-D Hartman function with Gaussian noise and $N_T = 6000$. \emph{Right panel:} 6-D Hartman function with Gaussian noise for FB with different values of $r$. \newstuff{The Pareto frontiers are highlighted for GP (solid line) and $t$-GP (dashed line).}
			\label{fig:ervscost}}
	\end{figure*}
	
	Our main goal with adaptive batching is improved computational performance. Of course, a faster algorithm generally requires to sacrifice predictive accuracy.  As such, direct comparison of schemes is not possible but must be considered through the above trade-off.
	Figure \ref{fig:ervscost} and Table \ref{tbl:syntheticresults} show the link between the error rate $\cal{ER}$  from \eqref{loss} and the running time across the proposed schemes. Since we desire fast and accurate schemes, there is a Pareto frontier going from top-left to bottom-right. In the 2-D case study (shown in the left panel in Figure~\ref{fig:ervscost}), we see that the most accurate scheme is $t$-GP with FB, while the fastest is GP with DDSA. Another Pareto-efficient scheme is $t$-GP with MLB which is arguably the best (the second fastest among $t$-GPs, and the second most accurate). In 6-D ABSUR works poorly, probably due to under-performance of the gSUR criterion; see~\cite{lyu2018evaluating} who showed that cUCB appears to be empirically better for this 6-D Hartman function. Another reason is that gSUR converges in a slower rate, see the middle panel in Figure \ref{fig:costvser}: gSUR takes $N_T \approx 30000$ simulations to achieve a comparably small error rate $\mathcal{ER}$. However, in Figure \ref{fig:ervscost}, $N_T = 6000$ for 6-D experiments. %The best choice are MLB and ADSA; DDSA and RB are somewhat faster but often lead to significant increase in $\cal {ER}$.
	
	Looking at the running times, we see that there are major gains from adaptive batching; the baseline FB scheme takes almost 10 times longer to run than designs with adaptive $r_i$'s. Fixed batching generally performs well in terms of $\cal{ER}$ (as it ends up being more exploratory) but practically those gains are crowded out by the huge cost in computational efficiency. Overall, among the five proposed schemes the recommended choice is MLB and ADSA which tend to produce low $\cal{ER}$ with a significant reduction in computational time.

	\newstuff{As mentioned in the Introduction, the benefit of replication is inextricably tied to simulation noise. To this end, in Appendix~\ref{sec:tuning-adsa} we investigate the role of the signal-to-noise ratio (SNR) on algorithm's performance by varying the noise variance $\taun^2$ in the 2-D case study with Gaussian noise. Figure~\ref{fig:kT_vs_taun2} shows that as $\taun^2$ increases, designs become smaller ($k_T$ decreases, except for ADSA). The performance metrics are reported in
Tables~\ref{tbl:absur-rbar} and \ref{tbl:adsa-cbt} in Appendix~\ref{sec:tuning-adsa}. As expected, lower SNR increases $\mathcal{ER}_T$ and algorithms should be tuned depending on the level of noise. For example, for ADSA and DDSA, one should increase $c_{bt}$ if SNR is low; for ABSUR one should increase $\bar{r}$. Some intuition can also be gleaned from Table~\ref{tbl:syntheticresults} and Table~\ref{tbl:syntheticresultshet}: the second experiment with $t$-distributed noise has much lower SNR compared to the first one with $\epsilon \sim {\cal N}(0,1)$. Lower simulation noise means that less replication is needed, which implies reducing $\bar{r}$ and $r^L$ and tends to advantage MLB compared to ADSA and ABSUR. Consistent with conclusions in \cite{lyu2018evaluating}, $t$-GP performs better than plain GP in such a setup where noise is heavy-tailed.}

%\subsection{Heteroskedastic Case Study}
	
		\begin{table}[htb!]
		\caption{\newstuffblue{Scheme performance in the 2-D heteroskedastic synthetic case study with \emph{2-D Branin-Hoo} response and noise $\epsilon(x^1, x^2) \sim t_{6-4x^1}(0,0.16(4x^1+1)^2)$. Results are means ($\pm$ standard deviations) from 50 runs of each combination of a metamodel and batching scheme. Note that the running times for GP and $t$-GP, which are from \texttt{MATLAB}, and for hetGP, which is from \texttt{R}, are not  comparable.
				%
				%of GP and $t$-GP metamodels with FB, MLB, RB, ABSUR, ADSA and DDSA designs in 2-D \emph{Branin-Hoo} synthetic experiments. Results are based on 50 runs.}
				\label{tbl:syntheticresultshet}}}
		\begin{center}
%			\begin{small}
				\begin{sc}
					\begin{tabular}{lrrr}
						\toprule
						Design &  Error Rate $\mathcal{ER}_T$& Time/s & Ave~$k_T$\\
						\midrule
						\multicolumn{4}{c}{Plain GP  in \texttt{MATLAB}} \\
						\midrule
						FB & 0.034 $\pm$ 0.029 & 106.37 & 200.00 \\
						ABSUR    & 0.037  $\pm$  0.039  & 15.50  & 39.14 \\
						RB     & 0.039  $\pm$  0.035 &  10.93 & 39.92 \\
						MLB    & 0.041  $\pm$  0.041 & 11.61  & 42.26 \\
						ADSA  & \newstuff{\textbf{0.033 $\pm$ 0.042}} & 18.20 & 34.82 \\
						DDSA  & 0.034 $\pm$ 0.043 & \newstuff{\textbf{9.67}} & 37.00 \\
						\midrule
\multicolumn{4}{c}{$t$-GP in \texttt{MATLAB}} \\
						\midrule
						FB  & 0.024  $\pm$  0.010 &  192.44  & 200.00 \\
						ABSUR     & 0.036  $\pm$  0.014 & 29.55  & 35.00 \\
						RB    & 0.032  $\pm$  0.014 &  23.65 &39.66 \\
						MLB     & \newstuff{\textbf{0.030  $\pm$  0.018}} & 22.88  & 39.72 \\
						ADSA  & 0.031 $\pm$ 0.013 & 26.26 & 30.68 \\
						DDSA  & 0.034 $\pm$  0.018 & \newstuff{\textbf{15.30}} & 37.00 \\
						\midrule
\multicolumn{4}{c}{hetGP in \texttt{R}} \\
						\midrule
						FB  & 0.035 $\pm$ 0.010 & 36.93 & 200.00 \\
						ABSUR   & 0.031  $\pm$  0.011 & 5.38  & 46.40 \\
						RB   & 0.035  $\pm$  0.010 &  1.45 &48.10 \\
						MLB     & 0.034  $\pm$  0.017 & \newstuff{\textbf{1.31}}  & 49.10 \\
						ADSA  & 0.035 $\pm$ 0.010 & 2.98 & 41.75 \\
						DDSA  & \newstuff{\textbf{0.030 $\pm$  0.010}} & 1.51 & 36.00 \\
	cIMSPE  & 0.032 $\pm$ 0.016 & 2.47 hrs & 1028.20 \\
						\bottomrule
					\end{tabular}
				\end{sc}
%			\end{small}
		\end{center}
%		\vskip -0.1in
	\end{table}

\newstuffblue{To further investigate the impact of noise on different schemes, as well as to showcase the use of alternative GP metamodels, Table~\ref{tbl:syntheticresultshet} shows results for the 2D Branin-Hoo experiment with heteroskedastic noise $\epsilon \sim t_{6-4x^1}(0,(0.4(4x^1+1))^2)$.
In this experiment we test both the different batching schemes, as well as two other metamodel familiies: $t$-GP and hetGP. $t$-GP extends the GP paradigm to allow for $t$-distributed observations, see Appendix \ref{subsec:t}. \texttt{hetGP}, implemented in the eponymous \texttt{R} library \citep{binois2019hetgp}, non-parametrically learns not just the mean response but also the input-dependent observation noise surface $\tau^2(\cdot)$.}

\newstuffblue{Using the \texttt{hetGP} library we further compare our adaptive batching to the cIMSPE algorithm described in Section 4.2 of \cite{binois2019hetgp}. cIMSPE is similar in spirit to ADSA except that it allocates simulations one-by-one. At each step, cIMSPE uses a criterion $\im_n$ to decide whether to add a new unique input, or increase by one the replicate count at an existing input. The comparison is based on the expected value of $\im_n$ and is replication-biased by comparing not just one-step-ahead but over a horizon of $h$. We use the cUCB criterion $\im^{cUCB}$ and a horizon of $h=3$. While cIMSPE offers a strong motivation for sequential construction of replicated designs, it is extremely slow because it has as no intrinsic batching and therefore requires $N$ sequential steps to allocate $N$ simulations. Consequently, it is only feasible when $N$ is small and takes orders of magnitudes more time in our setting with $N = 2000$. This limitation of the cIMSPE was one of the motivations for explicitly incorporating batching (rather than simply accommodating replication) in our approaches.}

\newstuffblue{Several observations can be gleaned from Table~\ref{tbl:syntheticresultshet}: (1) In terms of metamodels, $t$-GP and hetGP perform better than plain GP in this context with heteroskedastic noise. (2) In terms of adaptive batching schemes, their accuracy ($\cal{ER}$) is generally quite similar. DDSA runs the fastest and has among the lowest running times. (3) The comparator schemes yield similar error rates but are not competitive in terms of running times: cIMSPE is about 100 times slower and generates over a 1000 unique inputs compared to less than 50 for our schemes. FB is also slow (~6 times slower), although in combination with $t$-GP it does achieve the overall lowest error rate $\cal{ER}$.}

	To give some intuition about how the replication level should depend on the total budget $N_T$, the right panel of Figure~\ref{fig:ervscost} shows the performance of FB as we vary $r$ and $N_T$. As expected, lower $r$ generally leads to lower error rate ${\cal ER}$ but longer running time. This indicates the intrinsic necessity to explore the input space adequately which introduces a lower bound regarding the number of unique inputs $k_T = N_T/r$ for FB. However, for very low $r$ (e.g.~$r < 20$ for $N_T=6000$) there is essentially no gain from additional exploration implying that one can safely agglomerate simulations into batches without sacrificing accuracy. The resulting J-shape in the Figure implies that there is an "optimal" $r^*(N)$ that minimizes ${\cal ER}$ without needless performance degradation: $r^*(6000) \simeq 10, r^*(2 \cdot 10^4) \simeq 50, r^*(5 \cdot 10^4) \simeq 100$. This feature showcases both the strength and the weakness of fixed batching: in principle excellent performance is possible if $r \simeq r^*$ is fine-tuned; however such fine-tuning is very difficult and without it FB can be highly inefficient. The proposed adaptive batching schemes aim to automatically fine-tune $r_n$ sequentially removing this limitation.

	\begin{figure}[ht]
		\begin{minipage}[t]{0.33\linewidth}
			\begin{center}
				\includegraphics[width=1\columnwidth]{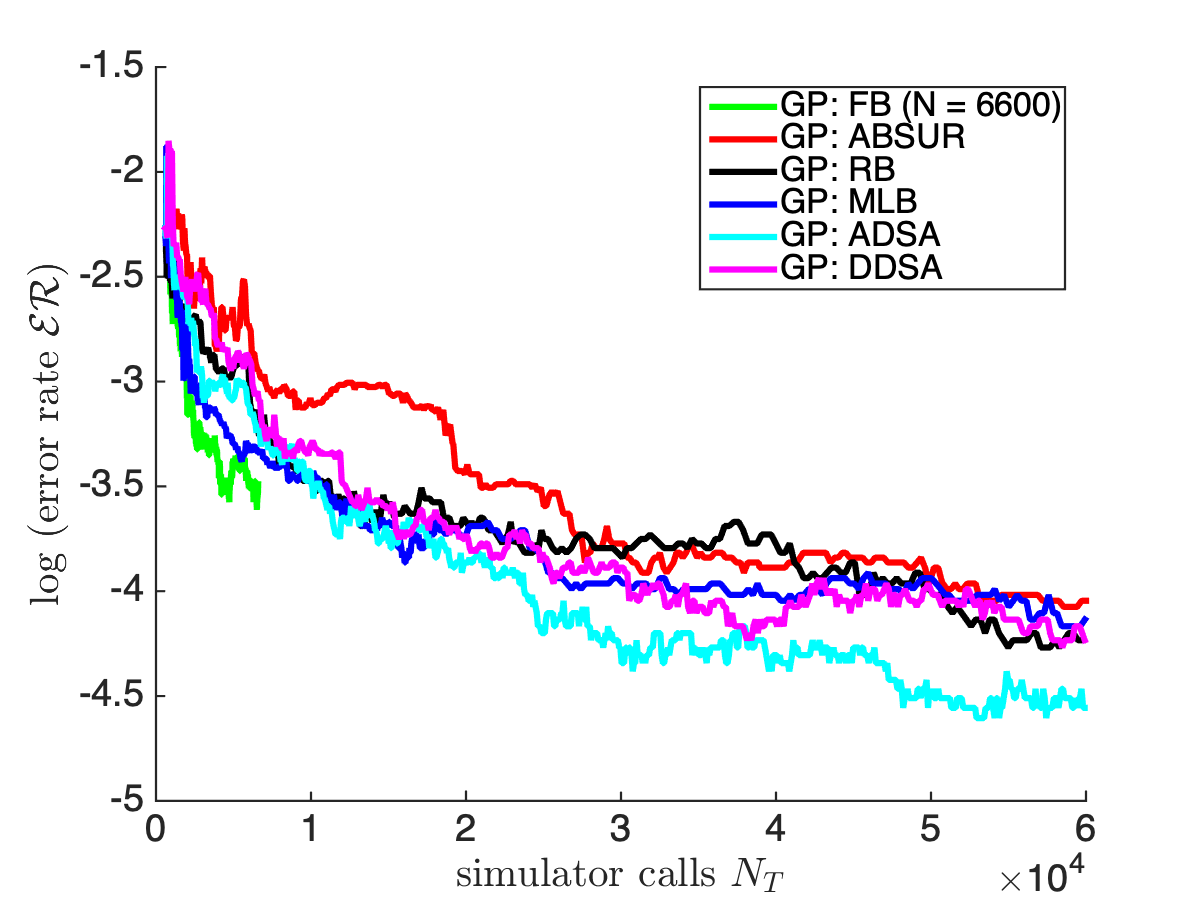}
			\end{center}
		\end{minipage}
		\begin{minipage}[t]{0.33\linewidth}
			\begin{center}
				\includegraphics[width=1\columnwidth]{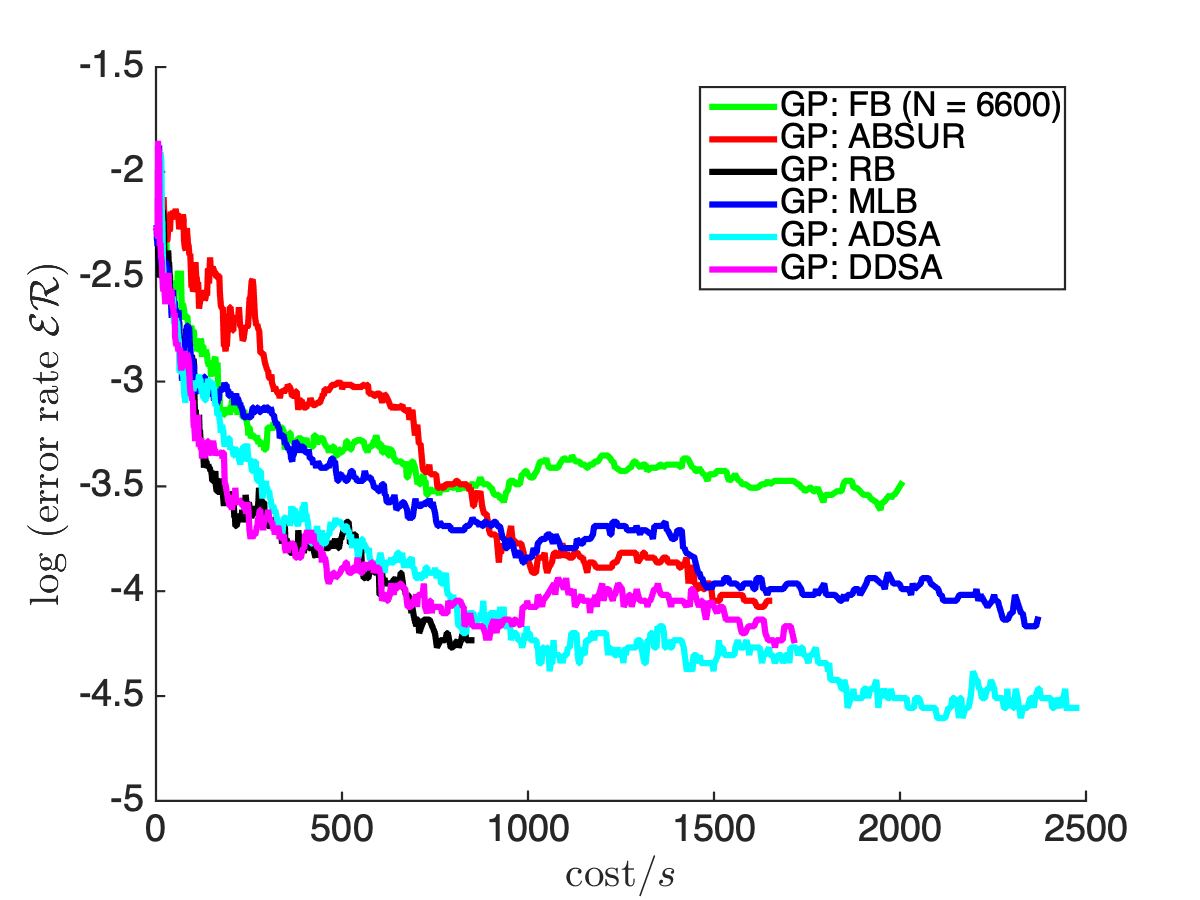}
			\end{center}
		\end{minipage}
		\begin{minipage}[t]{0.33\linewidth}
			\begin{center}
				\includegraphics[width=1\columnwidth]{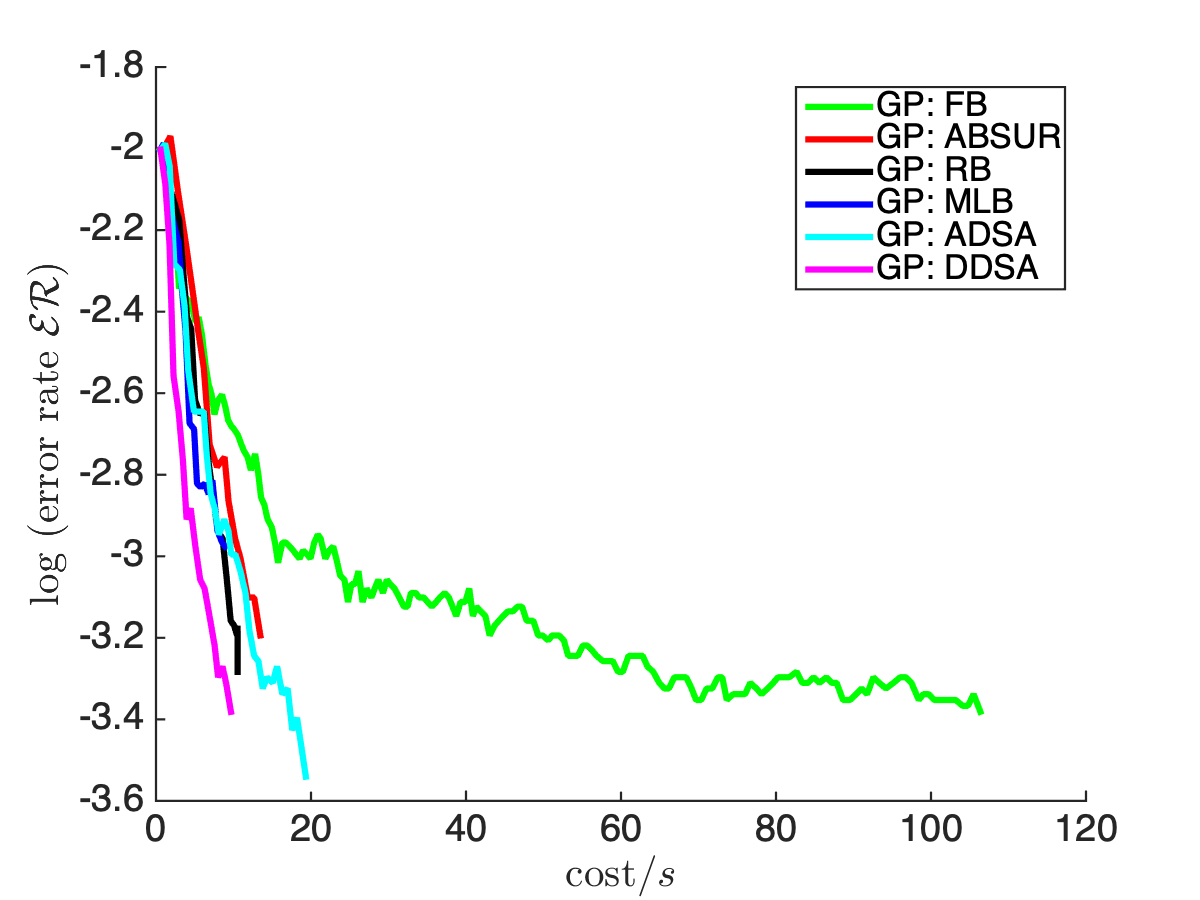}
			\end{center}
		\end{minipage}
		\caption{Log Error rate $\log \mathcal{ER}_t$ as a function of simulator calls $N_t$ for FB ($r = 10$), ABSUR, RB, MLB, ADSA and DDSA and 6-D  experiments (left panel). Log error rate $\log \mathcal{ER}_t$ as a function of running time $t$ for the 6-D case study with Gaussian noise (middle panel) with $N_T = 60000$ and for the 2-D case study with heteroskedastic noise (right panel) with $N_T = 2000$. The FB algorithm is stopped at $N_t = 6000$ since computation is too slow.}
		\label{fig:costvser}	
	\end{figure}

	Another goal of adaptive batching is to enable an organic way to grow designs as $N_T$ changes (while for FB $r$ necessarily must be pre-chosen in terms of $N_T$). A good algorithm is able to efficiently improve its accuracy as $N_T$ grows, avoiding excessive exploration or exploitation.
	The left panel of Figure~\ref{fig:costvser} shows the $\log$ error rate $\mathcal{ER}$ as a function of $N_T$ for FB, ABSUR, RB, MLB, ADSA and DDSA for the 6-D \emph{Hartman} experiments, respectively.  For FB, we stopped at $N_T = 6000$ due to prohibitive running times for designs. We observe that \newstuff{while all schemes perform somewhat similarly, MLB reduces the error rate $\mathcal{ER}$ at the fastest rate when $N_n < 600$, and otherwise, ADSA is the fastest}. ADSA shines in the later stage of sequential development of DoE, since it needs enough ``candidate inputs" to calculate the allocation rule. In terms of computational efficiency, we are concerned not with ${\cal ER}$ in terms of $N_T$ but in terms of running time---i.e.~how much predictive accuracy can be achieved within a given time budget. The respective relationship is shown in the middle and right panels of Figure~\ref{fig:costvser} where the $x$-axis is now in terms of $t$ seconds. We observe that all the adaptive schemes reduce the error rate $\mathcal{ER}$ at a faster rate than a scheme with fixed replication level. In the early stage, RB and DDSA are the fastest, and ABSUR is the slowest. However, as $N_T$ or $t$ continues to rise, ADSA keeps reducing the error rate $\mathcal{ER}$ and eventually achieves a smaller $\mathcal{ER}$ than other algorithms. However, ADSA usually takes slightly longer time.  In conclusion, ADSA is the most accurate algorithm given a large enough cost $t$ or simulator calls $N_T$, and MLB is the most accurate algorithm when $N_T$ is small. Results are consistent with those observed in Figure~\ref{fig:ervscost}.

	\begin{figure}[h]
		%\vskip 0.2in
		\begin{minipage}{0.45\textwidth}
			\begin{center}
				\includegraphics[width=0.9\columnwidth]{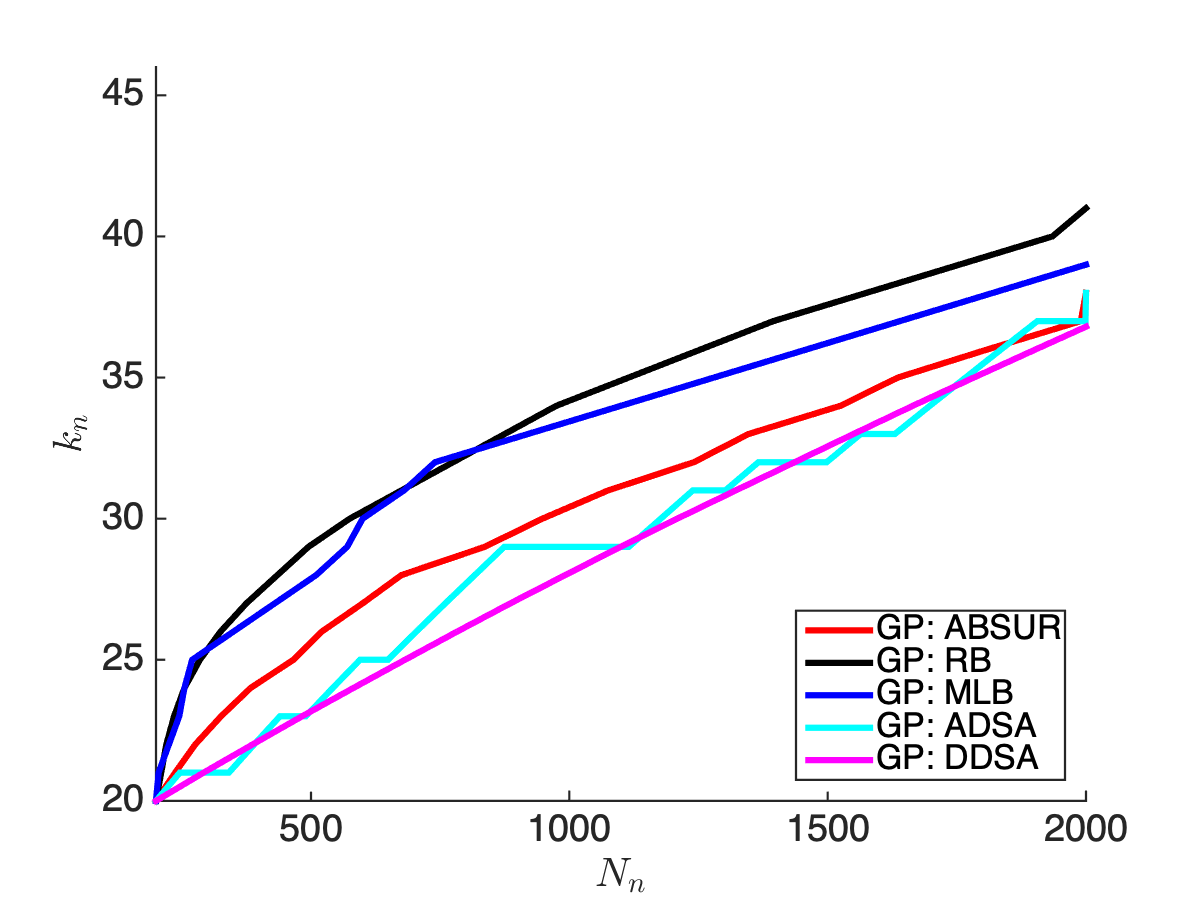} \\
				%2-D case study with Student-$t$ noise
			\end{center}
		\end{minipage}
		\begin{minipage}{0.45\textwidth}
			\begin{center}
				\includegraphics[width=0.9\columnwidth]{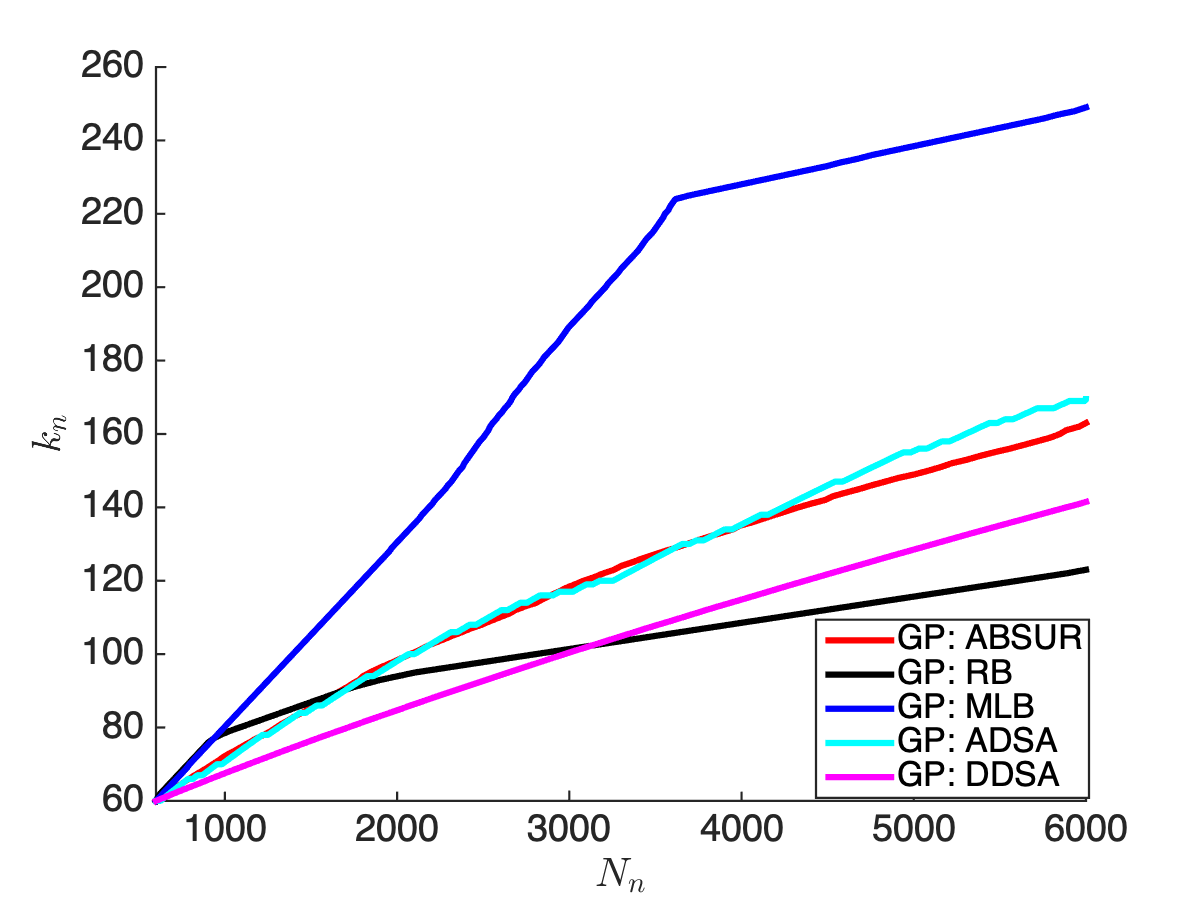} \\
				%6-D case study with Gaussian noise
			\end{center}
		\end{minipage}
		\caption{The design size $k_n$ as a function of simulator calls $N_n$. Left: 2-D case study with heteroskedastic noise; Right: 6-D case study with Gaussian noise. %Right: the running time at each step $t_n$ as a function of design size $k_n$ for 6-D case study with Gaussian noise.}
}
		\label{fig:kvsN}
	\end{figure}

	Recall that GP model fitting complexity is  $\mathcal{O}(k_n^3)$ (driven by the matrix inversion $\mathbf{K}^{-1}$), so that the design size $k_n = | \cA_n|$ is the primary driver of computational efficiency. In the baseline FB scheme, $r^{(n)} \equiv r$ is constant so that $k_n = N_n/r$ grows linearly in simulator budget $N_n$. This is precisely the reason that a constant $r$ becomes impossible to maintain as $N_n$ grows and why we had to abandon FB in the left panel of Figure~\ref{fig:costvser}. A key aim of adaptive batching is to achieve \emph{sub-linear} growth of $k_n$ i.e.~$k_n/N_n \to 0$ as $n$ grows so that $r^{(n)}$ keeps getting larger as we develop the DoE. Figure \ref{fig:kvsN} plots $k_n$ as a function of $N_n$ for 2-D and 6-D experiments.  As desired, we observe a generally concave shape, which is approximately of square-root shape. The stair-case shape of $k_n$ for ADSA is due to the adaptive re-allocation of new simulations which allow to increase $N_n$ without changing $k_n$ at some steps. We note that RB and ADSA achieve the most concave shape and hence would be the fastest for very large $N_n$ which can be seen indirectly in Figure~\ref{fig:costvser} as well. %\newstuffblue{The right panel in Figure \ref{fig:kvsN} shows that the running time at each step increases as the design size $k_n$ grows, indicating the desire to achieve sub-linear growth of $k_n$.}

	\subsection{Comparing Designs}
	To drill down into the designs obtained from different approaches, Figure~\ref{fig:zbyn} visualizes the adaptively batched designs produced for the 2-D Branin-Hoo experiment with heteroskedastic Student-$t$ noise. The left panel displays the resulting design size $k_T$ with simulation budget of $N_T = 2000$. Recall that besides FB and DDSA, design sizes of all other schemes vary across algorithm runs (i.e.~$k_T$ depends on the particular realizations $y_{1:N_T}$), so that $k_T$ is a random variable; in the plot we visualize its boxplot across 50 runs of each scheme. The smallest designs are obtained from ADSA (31-39 unique inputs). DDSA produces exactly $k_T=37$ unique inputs. Recall that DDSA alternates between adding a new site and re-allocating to existing sites, while ADSA does the same adaptively; in this case we find that slightly more than half the time re-allocation is preferred. The design size $k_n$ for ABSUR is slightly larger at 34-42. The value of $k_T$ for RB varies from 37 to 45, while for MLB has the greatest number of unique inputs, ranging from 34 to 50. Given $N_T = 2000$ the above implies that the schemes average about $Ave(r^{(n)}) = $40-60 replicates per site. The middle panel of Figure~\ref{fig:zbyn} shows the replication level $r^{(n)}$ as a function of design size $k_n$ for a typical run of schemes from Section \ref{sec:asur}, illustrating how replication is increased sequentially. Methods that raise $r^{(n)}$ faster end up with smaller design size $k_T$. ABSUR increases $r^{(n)}$ the fastest, with MLB having a similar pattern. With RB $r^{(n)}$ grows slower, implying that RB builds designs with more unique inputs.
	
	\begin{figure}[h]
		\begin{minipage}[t]{0.33\textwidth}
			\begin{center}
				\includegraphics[trim = 20 170 20 190, clip = true, width=1\columnwidth]{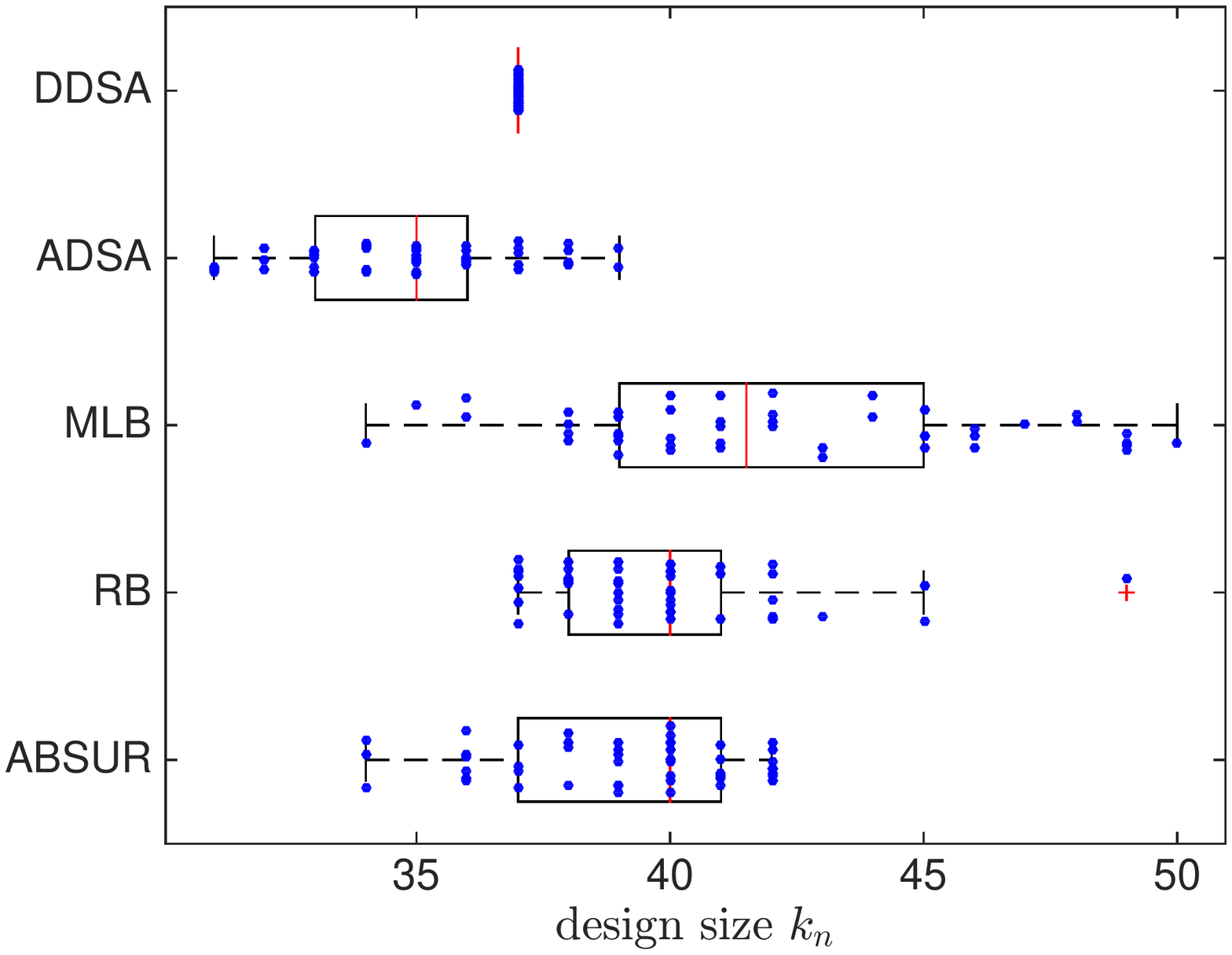}
			\end{center}
		\end{minipage}
		\begin{minipage}[t]{0.33\textwidth}
			\begin{center}
				\includegraphics[trim = 20 170 20 190, clip = true, width=1\columnwidth]{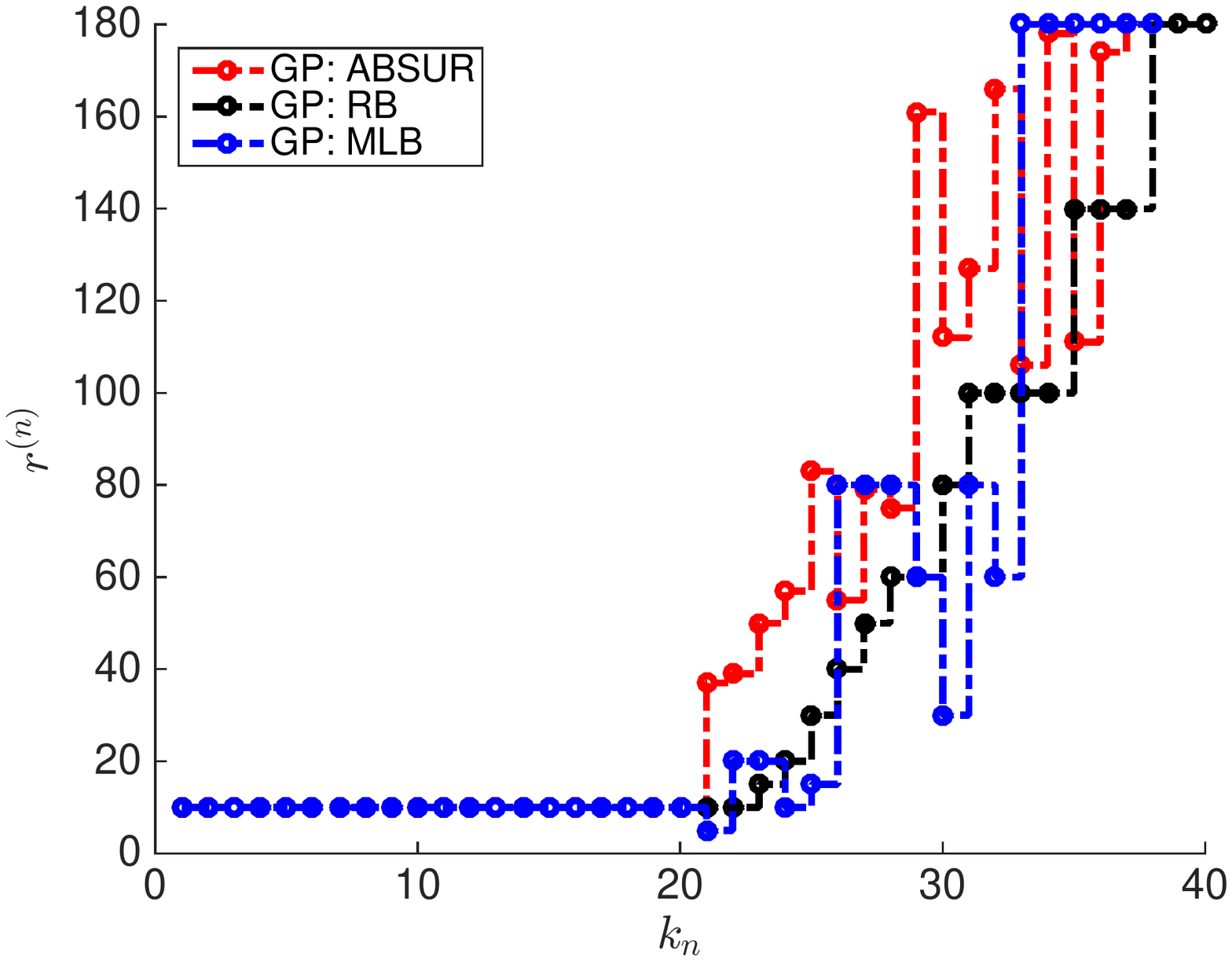}
			\end{center}
		\end{minipage}
		\begin{minipage}[t]{0.33\textwidth}
			\begin{center}
				\includegraphics[trim = 20 170 20 190, clip = true, width=1\columnwidth]{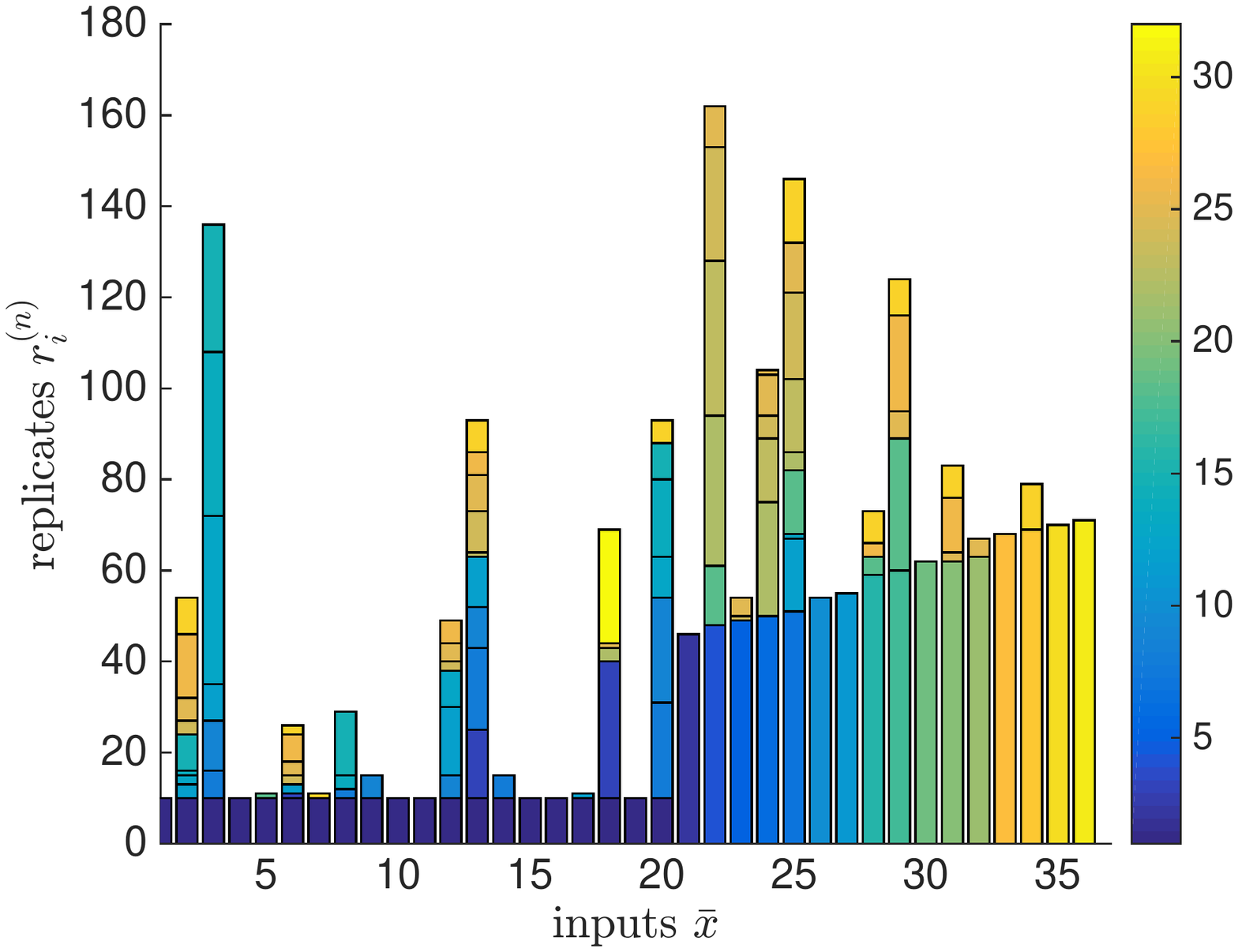}
			\end{center}
		\end{minipage}
		\caption{Visualizing adaptive batching for the 2-D case study with heteroskedastic noise.
			\emph{Left panel:}  distribution of design size $k_T$ corresponding to $N_T=2000$ across 50 algorithm runs. \emph{Middle}: number of replicates $r^{(n)}$ as a function of algorithm step $k_n$ for the schemes of Section~\ref{sec:batchdesign}. \emph{Right:} evolution of $r_i^{(n)}$ for ADSA designs $\bxbar$. The total $r_i^{(N)}$ is decomposed into $\Delta r_i^{(n)}$ for $n=1,\ldots,k_T$ with each $\Delta r$ color-coded by round $n$. }
		\label{fig:zbyn}
	\end{figure}

	The right panel of Figure~\ref{fig:zbyn} visualizes the replication of a representative ADSA run which has the option to add new inputs or re-allocate to existing ones. We show the sequential growth of $r^{(n)}_i$ through a stack histogram: the $x$-axis represents the unique inputs $\xx_i$ as picked by the algorithm and the vertical stacks represent $\Delta r_i^{(n)}$, color-coded by the round $n$ when they were added. We observe that only 10 out of the $n_0 = 20$ original inputs are revisited, and generally about half of the inputs are used in more than one round. At the same time, some inputs, such as $\bar{x}_{13}, \bar{x}_{20}, \bar{x}_{25}$ are visited in numerous rounds.

	Figure \ref{fig:fitted} shows the estimated zero-contour $\partial \hat{S}$  with its 95\% posterior credible band at $N_T=2000$ in the 2-D test case with \newstuff{heteroskedastic} noise. The volume of the credible band $\partial \hat{S}^{(\pm 0.95)}$, defined as
	\begin{align}
	\partial \hat{S}^{(\pm 0.95)}=& \left\{ {x} \in D: \left(\hat{f}^{(N_T)}({x}) + 1.96s^{(N_T)}({x})\right) \left(\hat{f}^{(N_T)}({x}) - 1.96s^{(N_T)}({x})\right)<0 \right\}, \label{cis}
	\end{align}
	captures inputs $x$ whose sign classification remains ambiguous and quantifies the uncertainty about the estimated zero-contour $\partial \hat{S}$. As expected, all schemes start by exploring the input space using a few replicates and then primarily sample in the target region around the level set, with increasing replication. \newstuff{Another feature that can be seen is that all methods favor the upper-left and bottom-right corners, which are regions that are simultaneously close to the edge of the input space (hence larger posterior $s_n(\cdot)$) and close to the zero contour. In particular, highest replication occurs in the upper-left region.}

Comparing the first four plots, we find that the ABSUR is more efficient than RB and MLB, concentrating at the zero-contour faster and simultaneously ramping up $r^{(n)}$ quicker. In the plot, this happens already after just half-a-dozen steps. In contrast, RB takes about a dozen steps to explore with correspondingly low $r^{(n)}$'s. Although MLB also ramps up $r_n$ quickly, it then steps back and forth between low and high replication levels, resulting in a slightly larger $k_T$ than ABSUR. ADSA and DDSA perform similarly. One observation is that they select similar inputs to allocate the extra simulator calls. For example the initial inputs close to the \newstuff{left and top} edge all get more replicates $r_n$ via reallocation in ADSA and DDSA. Across the DoE rounds, ADSA chooses to reallocate budget in approximately 54\% of them, so that $k_T = 0.54N_T/\Delta r$. Therefore, the value of $k_T$ is approximately the same for ADSA and DDSA.

Some of the design differences can be attributed to the different behavior of the underlying heuristics cUCB and gSUR. Indeed, cUCB tends to over-emphasize sampling around the zero-contour, while gSUR is more exploratory and tends to place a few inputs right at the edge of the input domain (upper left corner and lower right corner in the plot with ABSUR). The aggressiveness of cUCB generates more accurate estimates  $\partial\widehat{ S}$ even if the posterior uncertainty is higher (wider CI band) sometimes. %Another example is \newstuff{the upper left corner} of the space gets the most replicates (color yellow) compared with all other inputs for both algorithms. \newstuff{This is happening since this area is close to the edge of the input space where the estimate variance is large due to lack of information.}

	To conclude, the performance of FB is sensitive to value of replicates $r_n$. With higher $r_n$, the running time decreases while the error rate $\cal{ER}$ may increase or decrease. For different budget $N_T$, the "optimal" value of $r_n$ varies. We can tune $r_n$ to obtain FB scheme with best performance for a fixed $N_T$ in synthetic experiments where the ground truth is known. However, $N_T$ is not always provided initially in real experiments. At this time, it is impossible to tune $r_n$ for FB. Adaptive batching designs stand out perfectly. Instead of tuning $r_n$ manually at the start of sequential design, adaptive batching algorithms self-adaptively pick the current "optimal" $r_n$ during sequential design. Among all adaptive batching designs, DDSA and RB are the most efficient algorithms, while ADSA ends up with the most accurate estimate in most cases with approximately twice of running time. For low dimension experiments or larger $N_T$ \newstuff{or higher SNR}, DDSA reaches similar or even better error rate $\cal{ER}$ compared with ADSA, while in high dimension experiments or smaller $N_T$ \newstuff{or lower SNR}, results obtained with ADSA are significantly better than DDSA. \newstuff{ADSA is also more robust to the choice of hyperparameters and has a more stable performance in all cases.}

\begin{figure*}[htb!]
	\begin{minipage}[t]{0.32\linewidth}
		\begin{center}
			\includegraphics[trim = 0 200 120 150, clip = true, width=1\columnwidth]{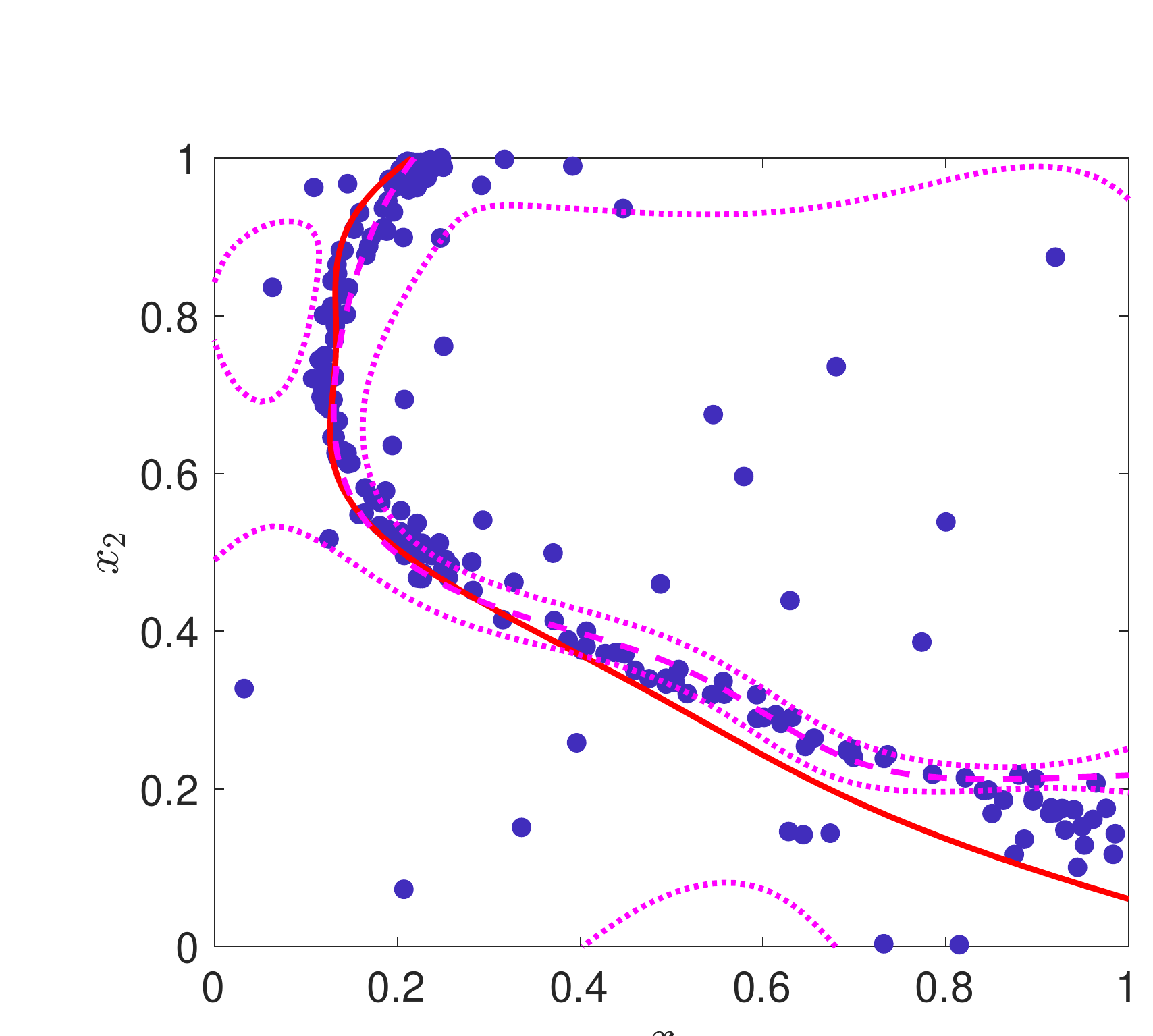} \\
			$\quad$ FB: $k_n = 200$ ($r=10$)
		\end{center}
	\end{minipage}
	\begin{minipage}[t]{0.3\linewidth}
		\begin{center}
			\includegraphics[trim = 30 200 120 150, clip = true,width=1\columnwidth]{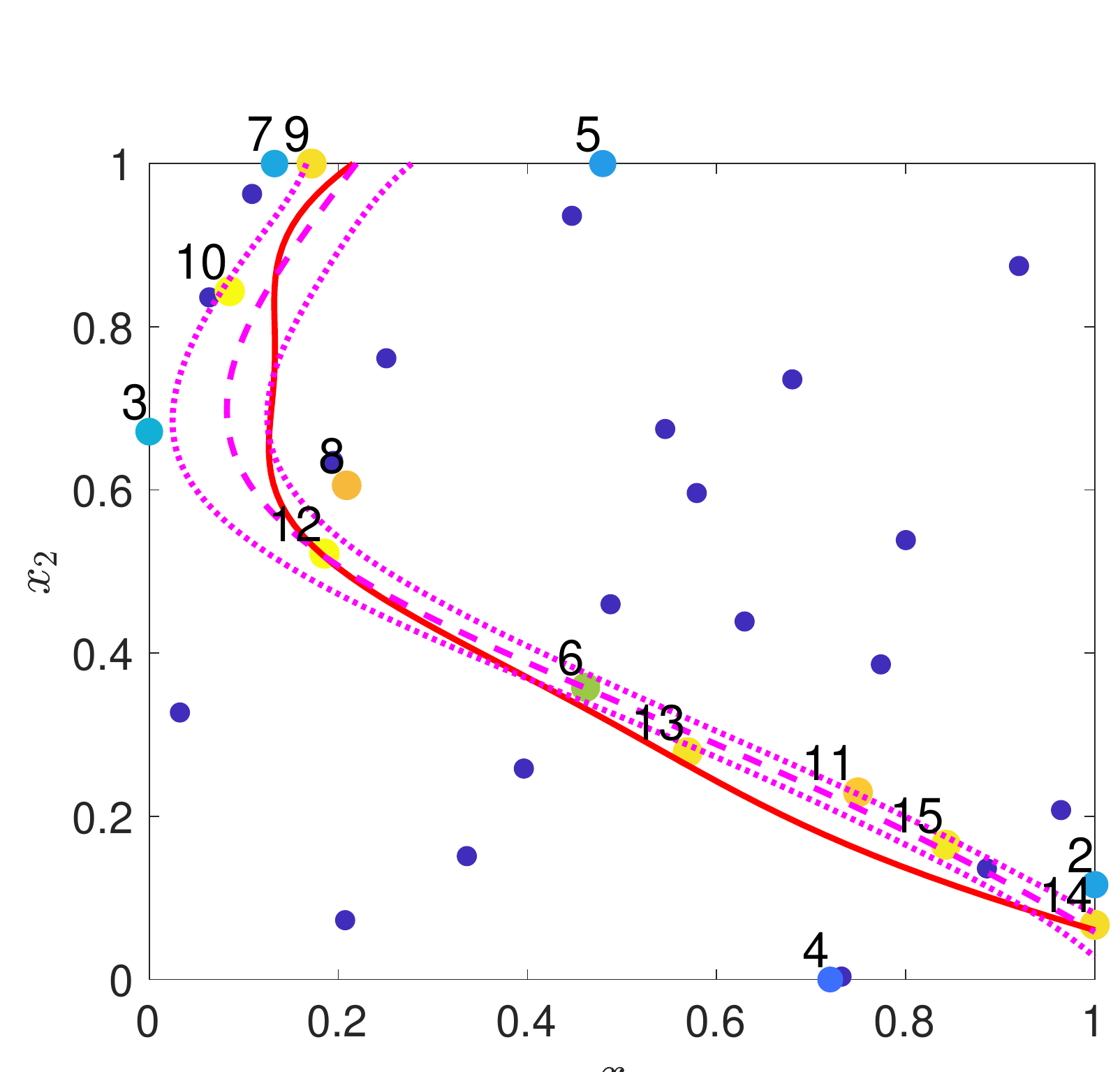} \\
			ABSUR: $k_T = 34$
		\end{center}
	\end{minipage}
	\begin{minipage}[t]{0.37\linewidth}
		\begin{center}
			\includegraphics[trim = 30 200 0 150, clip = true, width=1\columnwidth]{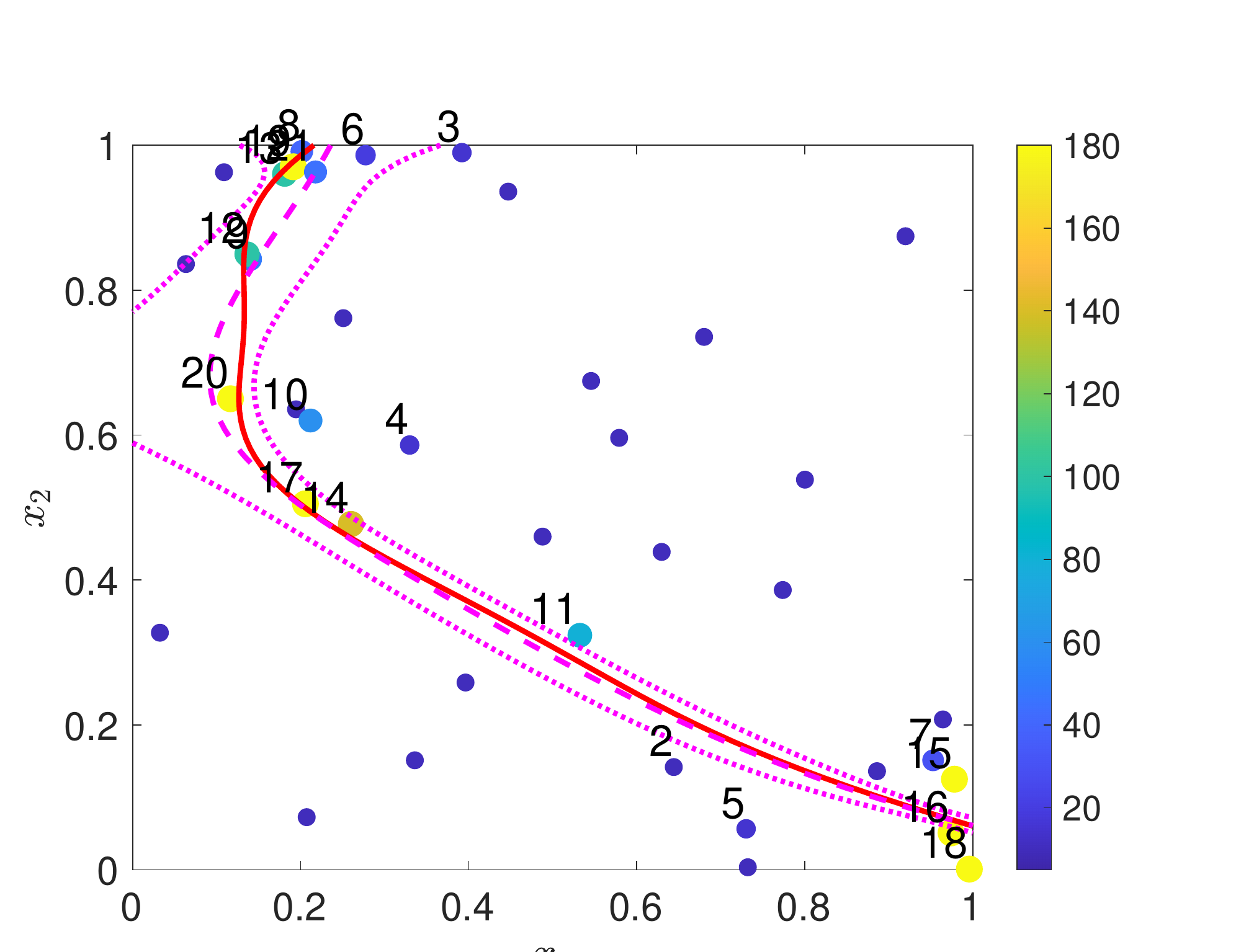} \\
			RB: $k_T = 40$
		\end{center}
	\end{minipage}
	\begin{minipage}[t]{0.32\linewidth}
		\begin{center}
			\includegraphics[trim = 0 200 120 170, clip = true,width=1\columnwidth]{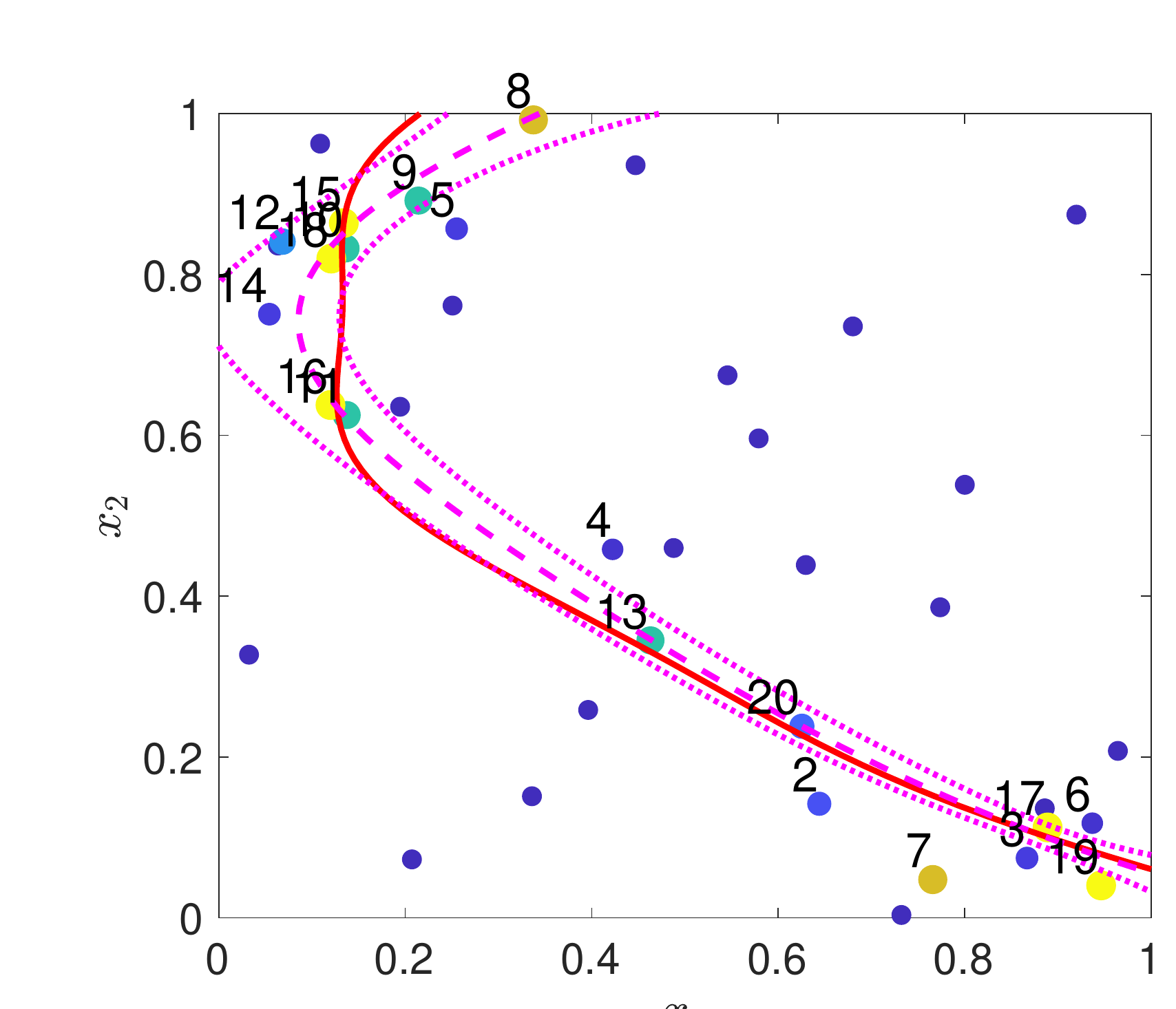} \\
			MLB: $k_T = 39$
		\end{center}
	\end{minipage}
	\begin{minipage}[t]{0.3\linewidth}
		\begin{center}
			\includegraphics[trim = 30 200 120 170, clip = true, width=1\columnwidth]{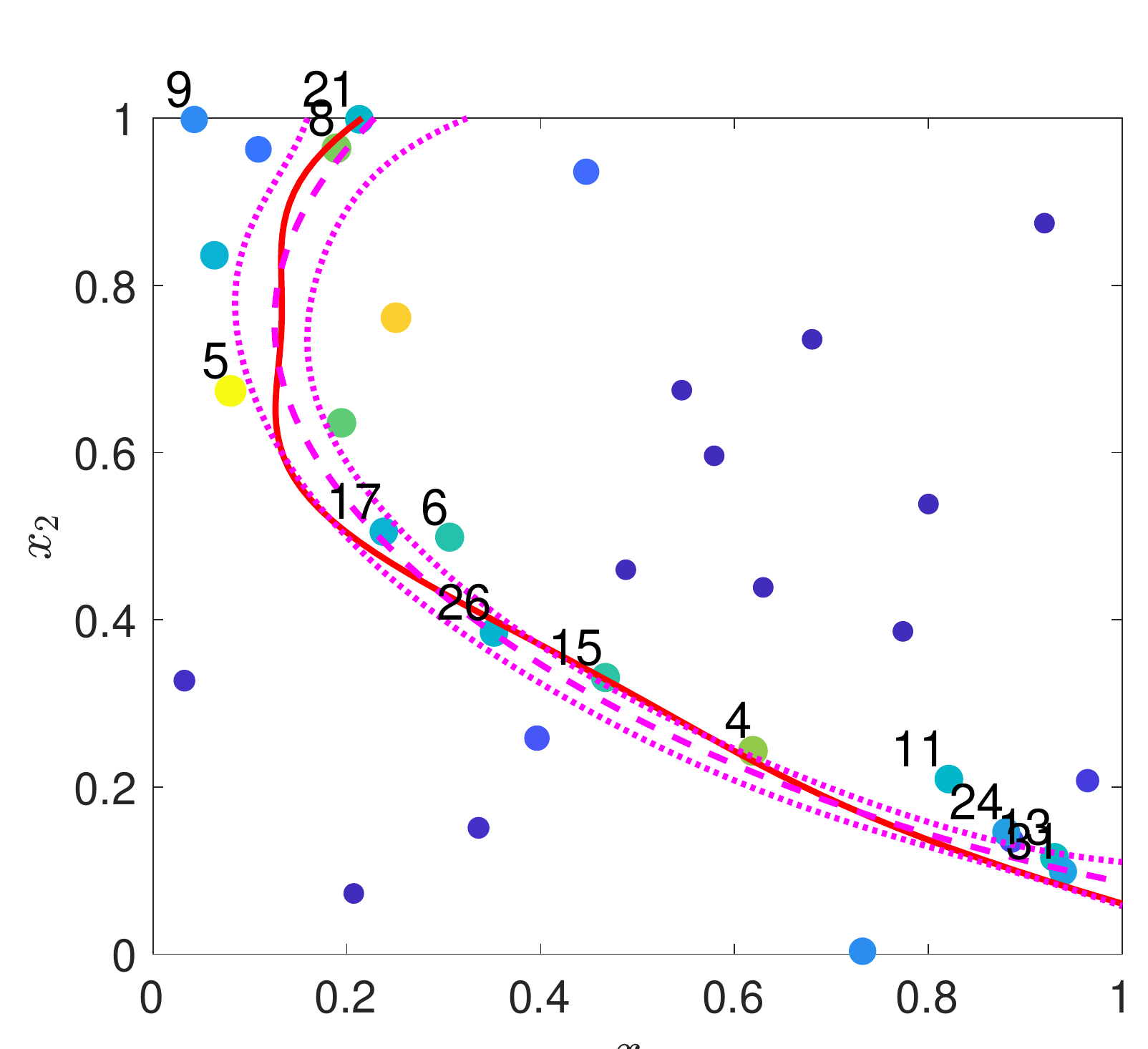} \\
			ADSA: $k_T = 33$
		\end{center}
	\end{minipage}
	\begin{minipage}[t]{0.37\linewidth}
		\begin{center}
			\includegraphics[trim = 30 200 0 180, clip = true,width=1\columnwidth]{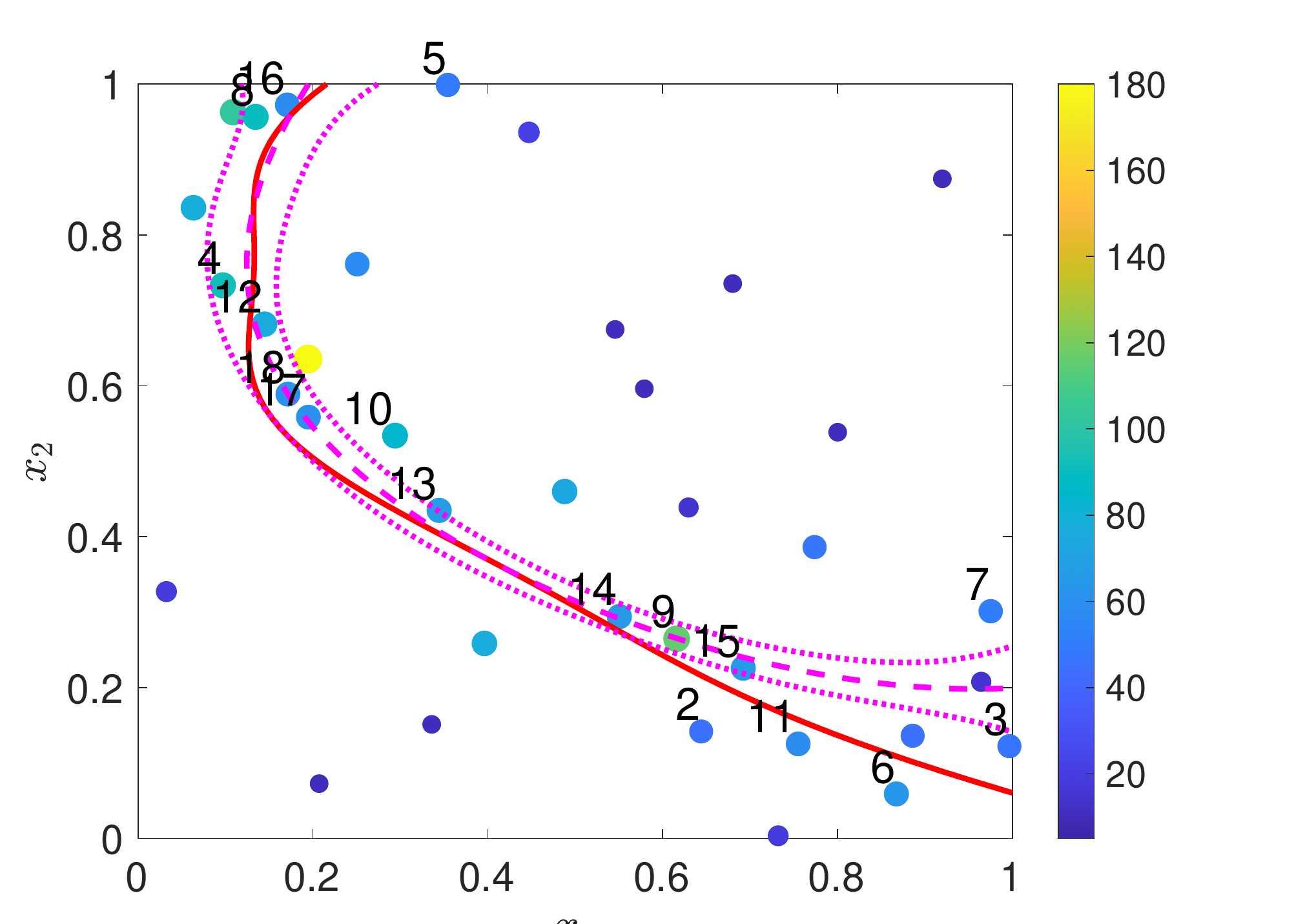} \\
			DDSA: $k_T = 37$
		\end{center}
	\end{minipage}
		\caption{\newstuff{GP fits $f | \cA_{k_T}$ %with FB (upper left), ABSUR (upper middle), RB (upper right), MLB (lower left), ADSA (lower middle) and DDSA (lower right)
			and designs for 2-D case study with heteroskedastic noise. The dashed lines are the estimated posterior zero-contours $\{ x :\hat{f}^{(N)}(x) = 0\}$ to be compared to the true contour (solid line). The dotted lines are the corresponding 95\% credible intervals. The initial design (same across all panels) are the blue unlabelled dots. The labels indicate the order of the inputs $\bar{x}_i, i = 1,\ldots k_n$ and the respective color/size are proportional to the replication level $r^{(n)}$. Design sizes $k_T$ vary across the schemes. %: for FB, $k_n = 200$; for ABSUR, $k_n= 34$; for RB, $k_n=38$; for MLB, $k_n=36$; for ADSA, $k_n=34$; for DDSA, $k_n=37$.
			\label{fig:fitted}}}
	\end{figure*}

	\section{Application to Optimal Stopping} \label{sec:bermudan}
	As a fourth and final case study, we consider an application of contour finding for determining the optimal exercise policy of a Bermudan financial derivative \citep{ludkovski2018kriging}. The underlying simulator is based on a $d$-dimensional geometric Brownian motion $(\bm{Z}_{t})=(z^1_t, \ldots, z^d_t)$ that represents prices of $d$ assets and follows the log-normal dynamics
	\begin{align}
	\bm{Z}_{t+\Delta t}&= \bm{Z}_{t} \exp \bigg((r-\frac{1}{2} \diag{\bm{\Xi}})\Delta t + \sqrt{\Delta t}\bm{\Xi} \Delta \bm{W}_{t}\bigg), \label{eq:brown}
	\end{align}
	where $r$ is the interest rate, $\bm{\Xi}$ is the $d \times d$ covariance matrix and  $\Delta \bm{W}_{t} \sim \mathcal{N}(0,\bm{I}_d)$ are the Gaussian stochastic stocks.
	Let $h(t,z)$ be the option payoff from exercising when $\bm{Z}_{t} = \bm{z}$. %The overall goal is to find the best dynamic exercise strategy $\tau^* \le T$, represented probabilistically as a stopping time with respect to the information generated by $(X_t)$.
	We assume that exercising is allowed every $\Delta t$ time units, up to the option maturity $T$. The overall goal is to determine the stopping regions $\{ S_t : t = \Delta t, 2\Delta t, \ldots, T-\Delta t \}$ to maximize $\E[ h(\tau, \bm{Z}_\tau) ]$, where $\tau = \min\{ t: \bm{Z}_t \in S_t \}$ is the exercise strategy. The dynamic programming principle implies that $S_t$ can be recursively computed as the  zero level set of the timing function $ z\mapsto f(t,\bm{z}) = h(t,\bm{z}) - \E[ h(\tau_t, \bm{Z}_{\tau_t} )]$ where the latter term is the continuation value based on the exercise strategy from the forward-looking $\{ {S}_s, s > t\}$. Numerically, this yields a simulator of $f(t,\bm{z})$ through pathwise reward over one-step-ahead simulations of $\bm{Z}_{t+\Delta t}$.
	
	\begin{table}[htb!]
		\caption{Performance of GP metamodels with FB, MLB, RB, ABSUR, ADSA and DDSA designs in the 2-D Average Put and 3-D Max Call examples. Results are  averages from  20 runs of each scheme.}
		\label{tbl:amput}
%		\vskip 0.15in
		\begin{center}
%			\begin{small}
				\begin{sc}
					\begin{tabular}{lccrr}
						\toprule
						Design & Model & Payoff & Time/s $T$ & Inputs $k_T$ \\
						\midrule
						\multicolumn{5}{c}{\textbf{2-D Average Put}} \\
						\midrule
						%					LHS   & GP & 1.436$\pm$ 0.008& 1.584$\pm$ 0.161 \\
						FB & GP & 1.451$\pm$ 0.002 & 29.82 & 100.00\\
						RB    & GP & 1.443$\pm$ 0.004 & 5.42 & 35.85\\
						MLB    & GP & \newstuff{\textbf{1.440$\pm$ 0.004}} & \newstuff{\textbf{4.92}} & 33.97 \\
						ABSUR     & GP & 1.446$\pm$ 0.004 & 11.40 & 53.80\\
						ADSA & GP & 1.445 $\pm$ 0.003 & 11.76 & 32.87 \\
						DDSA & GP & 1.445 $\pm$ 0.003 & 5.42 & 34.00 \\
						\midrule
						FB & $t$-GP & 1.449 $\pm$ 0.002 & 63.11 & 100.00 \\
						RB    & $t$-GP & 1.445 $\pm$ 0.004 & 11.36 & 36.39 \\
						MLB    & $t$-GP & \newstuff{\textbf{1.443 $\pm$ 0.004}} & \newstuff{\textbf{10.52}} & 35.35 \\
						ABSUR     & $t$-GP & 1.443 $\pm$ 0.004 & 26.13 & 49.79 \\
						ADSA & $t$-GP & 1.447 $\pm$ 0.003 & 19.00 & 44.83 \\
						DDSA & $t$-GP & 1.446 $\pm$ 0.003 & 11.31 & 34.00 \\
						\midrule
						\multicolumn{5}{c}{\textbf{3-D Max Call}} \\
						\midrule
						FB & GP & 11.26  $\pm$  0.01 & 2239.10 & 1000.00 \\
						RB    & GP & \newstuff{\textbf{11.23  $\pm$  0.01}}  & \newstuff{\textbf{37.42}} & 342.39 \\
						MLB    & GP & 11.24 $\pm$   0.01 & 38.17 & 342.07 \\
						ABSUR     & GP & 11.23  $\pm$  0.01 & 109.81 & 407.90 \\
						ADSA & GP& 11.25 $\pm$ 0.01 & 194.05 & 460.33 \\
						DDSA & GP& 11.26 $\pm$ 0.01 & 94.58 & 381.00\\
						\bottomrule
					\end{tabular}
				\end{sc}
	%		\end{small}
		\end{center}
	%	\vskip -0.1in
	\end{table}
	
	In this setting, the underlying distribution of $\bm{Z}_t$ at time $t$ is log-normal since $\log \bm{Z}_t$ is multivariate normal. To reflect this fact which dictates the importance of correctly identifying whether $x \in S_t$ or not (since option exercising decisions are made \emph{along} trajectories of $\bm{Z}$, conditional on the given initial value $\bm{Z}_0 = z_0$), we employ log-normal weights $\mu(dz)= p_{\bm{Z}_t}( \cdot | \bm{z}_0)$ in \eqref{loss}. We further use $\mu$ to weigh the respective $\im_n$ criteria when optimizing for new inputs. In line with the problem context, we assess performance using the ultimate estimated option value. The latter is evaluated via an out-of-sample Monte Carlo simulation that averages realized payoffs along a database of $M' = 10^5$ forward  paths $\bm{z}^{1:M'}_{0:T}$:
	\begin{align}
	\label{eq:hat-V} \hat{V}(0, \bm{z}_0) &= \frac{1}{M'} \sum_{m=1}^{M'} h(\tau^m_0, \bm{z}^{(m)}_{\tau^m_0} ),
	\end{align}
	with $\tau^m_0 := \min \{t : z^{(m)}_{t} \in \widehat{S}_{t} \} \wedge T$.
	Since our goal is to find the \emph{best} exercise value, higher $\hat{V}$'s indicate a better approximation of $\{S_t\}$. To allow a direct comparison, we set parameters matching the test cases in~\cite{ludkovski2018kriging}):
	\begin{align*}
	\text{2-D average Put option: } & & h_{Put}(t,\bm{z}) & = e^{-r t}( {\cal K} - z^1 - z^2)_+; \\
	\text{3-D Max-Call option: } & & h_{Call}(t,\bm{z}) & = e^{-r t}( \max (z^1, z^2, z^3)  - {\cal K})_+.
	\end{align*} These settings have very low signal-to-noise ratio, and non-Gaussian heteroskedastic noise, so  $N_T \gg 10^3$ is imperative. \newstuff{We use plain GP and $t$-GP metamodels (refitted every ten steps) with a constant noise variance $\taun^2$  to model the timing function $f(t,\bm{z})$. All adaptive algorithms combined with homoskedastic and heteroskedastic GP ($t$-GP)  are publicly available as part of the \texttt{mlOSP} library in \texttt{R} \cite{ludkovski2020mlosp}.}

	\begin{table}
		\caption{Parameters for the 2-D Basket Put Option and 3-D Max Call Option.}
		\label{tbl:optionexperiments}
		%\vskip 0.15in
		\begin{tabular}{rcc} \toprule
			& 2-D Basket Put & 3-D Max-Call \\ \midrule
			\begin{tabular}{@{}r@{}} Option \\ Parameters \end{tabular} & \begin{tabular}{@{}c@{}} ${\cal K} = 40, \Delta t = 0.04, T = 1$ \\ $ r = 0.06, \sigma = 0.2, X_0 = [40, 40]$ \end{tabular} & \begin{tabular}{@{}c@{}} ${\cal K} = 100, \Delta t = 1/3, T = 3$ \\ $r = 0.05, \sigma = 0.2, X_0 = [90, 90, 90]$ \end{tabular} \\
			Budget & $N_T=2000$, $k_0 = 20, r_0 = 20$ & $N_T = 30,000$, $k_0 = 300, r_0 = 30$  \\
			FB & $r=20$ & $r= 30$ \\
			MLB/RB & $\mathbf{r}_L = \{20, 30, 40, 50, 60, 80, 120, 160\}$ & $\mathbf{r}_L = \{20, 30, 40, 50, 80, 160, 240, 320, 480, 640\}$ \\
			ABSUR & $\mathcal{R}= [20, 160], T_{sim} = 0.01$ & $\mathcal{R}= [20, 640], T_{sim} = 0.01$ \\
			ADSA & $c_{bt}=10$ & $c_{bt}=6.67$  \\     \bottomrule
		\end{tabular}
	\end{table}
	
	Table~\ref{tbl:amput} shows the performance of different designs/models. In the 2-D setting the best performing scheme is DDSA. We obtain savings of 80\% in computation time compared to the baseline FB scheme. For the 3-D Max Call, DDSA achieves the highest payoff, and  at a fraction ($\sim 1/20$th) of time. RB and MLB lead to slightly smaller payoff than DDSA, but with a saving of 60\% in computation cost. ADSA leads to basically the same payoff as DDSA and takes approximately twice as much time compared with DDSA. ABSUR takes half the time of ADSA, leading to a lower payoff. In both 2-D and 3-D settings, ADSA and DDSA lead to a higher payoff and have a more stable performance than the other adaptive batch designs. In terms of design size $k_T$, ABSUR yields the largest $k_T$, while DDSA yields the most compact designs. %The design size $k_T$ is consistent with results in Figure \ref{fig:zbyn}: MLB yields the largest $k_T$, while DDSA and ADSA yield the most compact designs.
	
	\begin{figure}[htb!]
		\begin{minipage}[t]{0.44\textwidth}
			\begin{center}
				\includegraphics[trim = 0 0 80 0, clip = true, width=1\columnwidth]{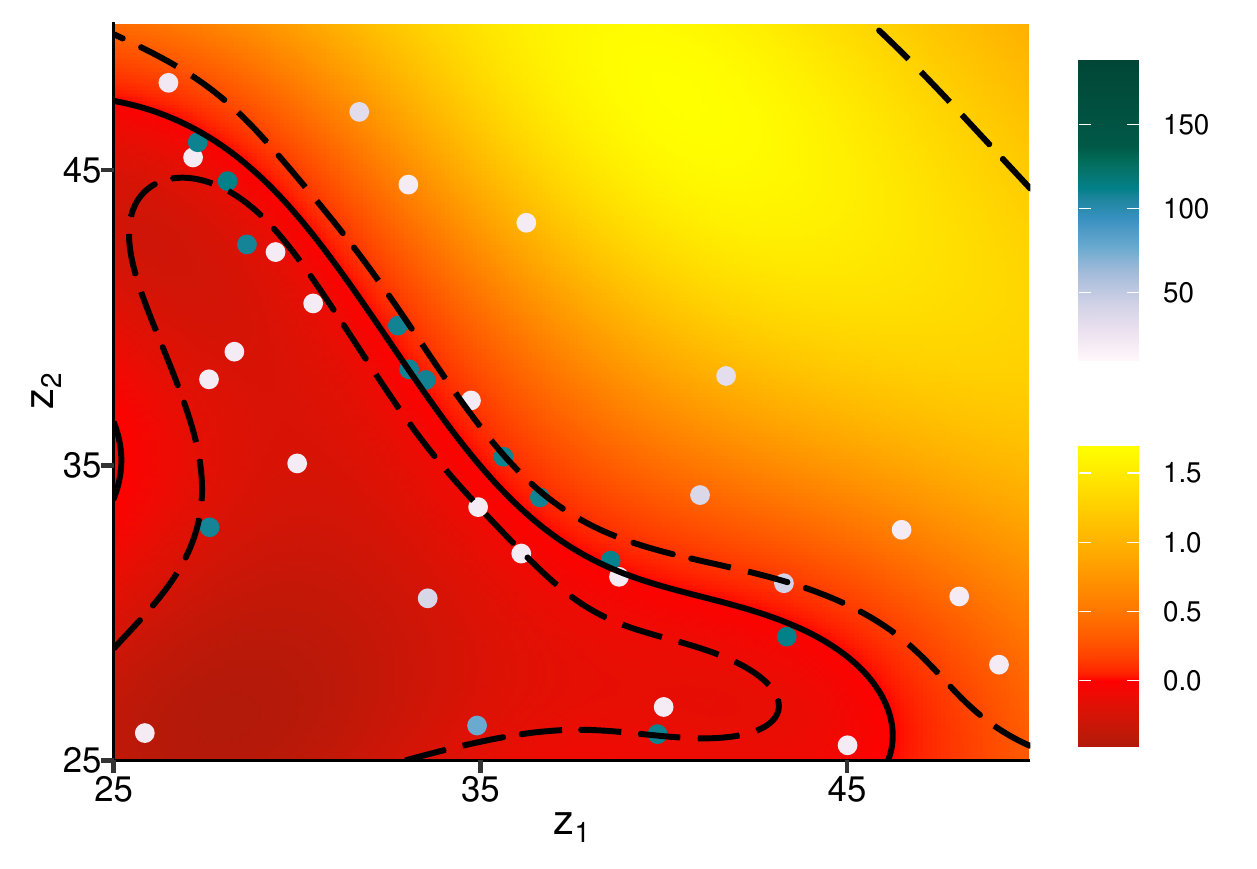} \\
				ABSUR: $k_T = 40$
			\end{center}
		\end{minipage}
		\begin{minipage}[t]{0.56\textwidth}
			\begin{center}
				\includegraphics[trim = 0 0 0 0, clip = true,width=1\columnwidth]{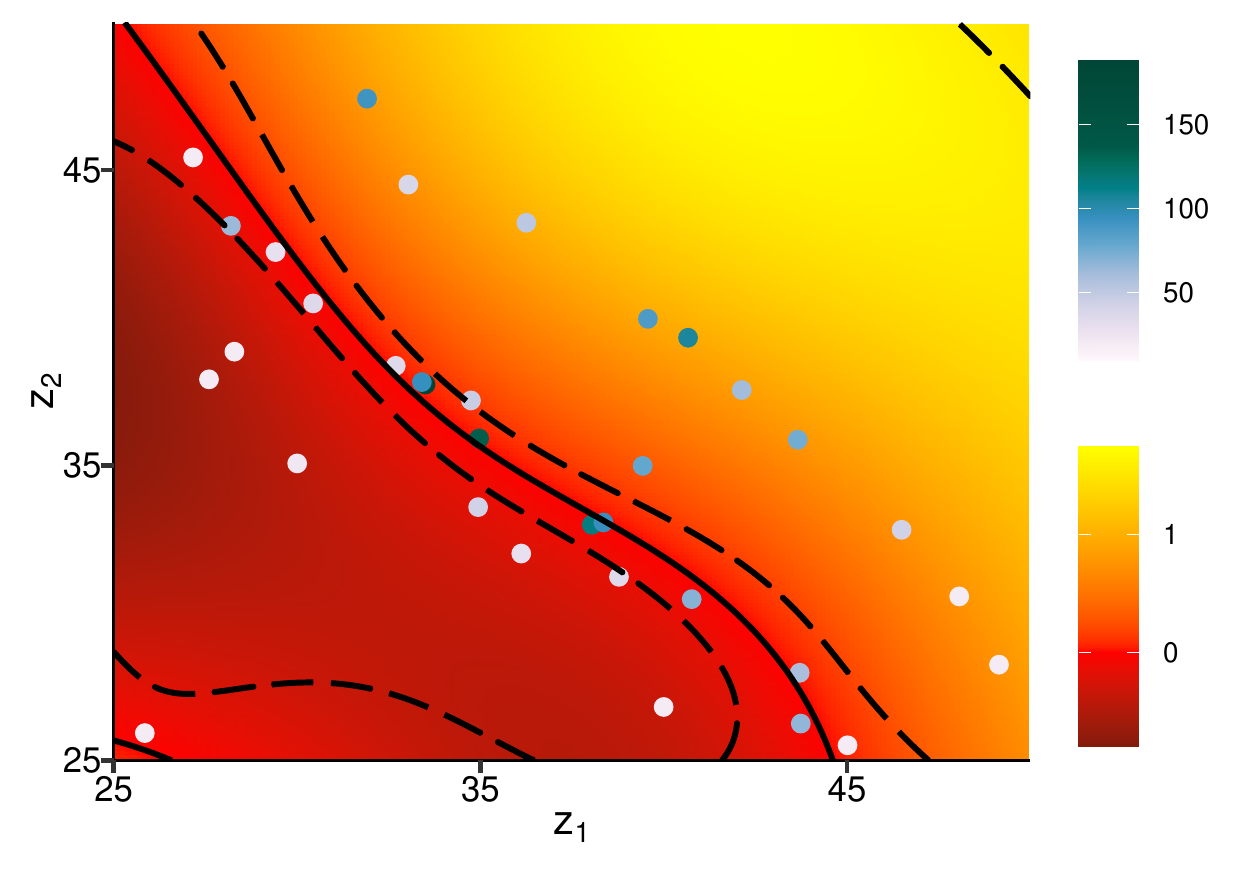} \\ ADSA: $k_T = 37$
			\end{center}
		\end{minipage}
		%\vspace*{-0.1in}
		\caption{GP fits $f^{(k_T)}(t,\cdot)$ and designs $\cA$ for 2-D average put option example at $t = 0.6$ and $N_T=2000$. \emph{Left panel:} ABSUR; \emph{right:}  ADSA.
			The solid lines are the estimated exercise boundary $\hat{f}^{(k_T)}(t,\bm{z}) = 0$ and the dashed lines are the corresponding 95\% credible intervals. The scatter plot is the design $\cA_{k_T}$ color-coded by replicate counts $r_i, i=1,\ldots,k_T$.}
		\label{fig:amput2d}	
	\end{figure}

	Figure \ref{fig:amput2d} shows the GP fits $\hat{f}(t,\bm{z})$ for ABSUR and ADSA for the 2-D Put case study at $t=0.6$. The desired zero-level contour goes from NW to SE and due to the chosen setting should be symmetric about the $z^1=z^2$ line. We see that both strategies select inputs around the contour; {consistent with the results shown in Figure~\ref{fig:fitted}, ABSUR is somewhat more exploratory and yields wider credible intervals for the exercise boundary $\{ \hat{f}^{(k_T)} = 0\}$ in regions close to the edge of the input space, especially at the NW and SE corners. ABSUR uses slightly more design sites $k_T(ABSUR) = 40 > k_T(ADSA) = 37$ and has a flatter distribution of replication counts. In contrast, ADSA uses up to $\max_n r^{(n)} = 188$ replicates. We also observe that several initial designs repeatedly receive more replications (up to 50 counts) in ADSA.}

	\section{Conclusion} \label{sec:conclusions}
	
	We have proposed and investigated five different schemes for adaptive batching in metamodeling of stochastic experiments. \newstuff{All schemes explicitly address the shifting exploration-exploitation trade-off by capturing the intuition of increasingly beneficial replication as sequential design is constructed}. Our presentation focused on the plain Gaussian Process paradigm but as shown are straightforwardly extended to alternatives, such as $t$-GP and \texttt{hetGP}. The key step is to construct an approximation of the batch look-ahead variance $s^{(n+1)}(x,r)$. Our results demonstrate that adaptive batching offers a simple mechanism to extract significant computational gains through building more compact designs and taking advantage of the symbiotic relationship between GPs and replication. Thus, compared with using a constant value for replicates $r$ over all inputs like in FB, we are able to gain more than an order-of-magnitude speed-up with minimal loss of metamodeling fidelity. Among the proposed adaptive batching schemes, we advocate the use of ADSA and DDSA (the latter being essentially a faster heuristic). While they lead to similar results in lower dimensional experiments, ADSA is observed to be more accurate in complex settings, such as higher dimension or low signal-to-noise ratio.
	
	Our focus has been on adaptive batching in the context of level-set estimation. Related problems such as evaluating the probability of failure, or evaluating a tail risk measure, would benefit from the same ideas and will be investigated in follow-up projects. \newstuffblue{Another extension is to tackle $\epsilon$-softened optimization, i.e.~target the region of $\epsilon$-optimal inputs for a given $\epsilon >0$. Such objective might be desirable to practitioners who simultaneously optimize over several (potentially non-qualitative) factors. This entails replacing the zero level set with $f(x) = 0$ with $f(x) = M_n$ where $M_n$ is an estimator for $\max_x f(x) | {\cal A}_n$. For instance, one could obtain $M_n$ similar to the computation of the Expected Improvement criterion in Bayesian Optimization.}  Another important problem that is beyond the scope of the present work is theoretical analysis about the asymptotic complexity of the proposed schemes such as ADSA, for example to establish the long-run growth rate of $k_n$ in order to quantify the asymptotic complexity of the GP metamodel as $N_n \to \infty$.

\newstuffblue{\textbf{Acknowledgements.} We thank the anonymous reviewers for their helpful comments that helped to improve on earlier versionss of the manuscript; we are also grateful to Mickael Binois for useful discussions and help in porting our algorithms from \texttt{MATLAB} to \texttt{R}. Both authors were partially supported by NSF DMS-1521743. ML is additionally supported by NSF DMS-1821240.}

%\clearpage
	
	\appendix

	\section{Allocation Rule}\label{sec:deltar-proof}

\newstuffblue{\begin{proof}[Proof of Proposition~\ref{thm:deltar}]
	%Let $\mathbf{R}^{(n)}$ be a $k_n \times k_n$ diagonal matrix with elements $\frac{1}{r_i^{(n)}}, i = 1,...,k_n$.
	Because the unique inputs are unchanged during the allocation step, comparing
	%According to eq.~\eqref{cov}, the posterior covariance of $f^{(n+1)}({\mathbf{x}}_*)$ based on $\mathbf{R}^{(n+1)}$ is:
	%\begin{align*}
	$\mathbf{C}^{(n+1)}
	= \mathbf{K}({\mathbf{x}}_*,  {\mathbf{x}}_*) -  \mathbf{K}_*(\bm{\Sigma}^{(n+1)})^{-1} \mathbf{K}_*^{T}$
	%\end{align*}
	to $\mathbf{C}^{(n)}=\mathbf{K}({\mathbf{x}}_*,  {\mathbf{x}}_*) -  \mathbf{K}_*(\bm{\Sigma}^{(n)})^{-1} \mathbf{K}_*^{T}$, the only term that changes is $\bm{\Sigma}^{(n+1)}$.
	Minimizing %tIMSE $\im_{SAO}^{(n)} = (\bm{\omega}^{(n)})^T \mathbf{C}^{(n+1)} %[\mathbf{k}(\bar{\mathbf{x}}_*,  \bar{\mathbf{x}}_*) -  \mathbf{k}_*(\bm{\Sigma}^{(n+1)})^{-1} \mathbf{k}_*^{T}]
	eq.~\eqref{tmse} therefore reduces to maximizing
	\begin{align}
	%	\nonumber \im_{SAO}^{(n)} &=  \mapsto \min! \\
	%	\Longleftrightarrow \quad &
	(\bm{\omega}^{(n)})^T \mathbf{K}_*(\mathbf{K} + \taun^2 \mathbf{R}^{(n+1)})^{-1} \mathbf{K}_*^{T} \bm{\omega}^{(n)} %\mapsto \max!
	\label{obj:csao2}
	\end{align}
	Decompose $\Delta \mathbf{R}^{(n)} =: \mathbf{B}^{(n)}\mathbf{B}^{(n)}$. Using the Woodbury Identity,
	\begin{align}
	(\bm{\Sigma}^{(n+1)})^{-1} = (\mathbf{K} + \taun^2 (\mathbf{R}^{(n)} -\Delta \mathbf{R}^{(n)}))^{-1}
	\approx (\bm{\Sigma}^{(n)})^{-1} + \taun^2(\bm{\Sigma}^{(n)})^{-1}\Delta \mathbf{R}^{(n)} (\bm{\Sigma}^{(n)})^{-1}, \label{eqn:reasimplified}
	\end{align}
	\noindent where the last expression is obtained by dropping the term $\mathbf{B}^{(n)}[\mathbf{K} + \taun^2 \mathbf{R}^{(n)}]^{-1}\mathbf{B}^{(n)} \approx \mathbf{0}$ due to $\max_{i} \Delta \mathbf{R}_{ii}^{(n)}  \ll 1$.
	Therefore, maximizing \eqref{obj:csao2} subject to $\sum_{i=1}^{k_n} \Delta r_i^{(n)} = \Delta r^{(n)}$ is equivalent to maximizing {
		\begin{align}
		\tilde{\im}_{SAO}(\Delta \mathbf{R}) &= \taun^2 \cdot (\bm{\omega}^{(n)} )^T \mathbf{K}_* (\bm{\Sigma}^{(n)})^{-1}\Delta \mathbf{R}^{(n)} (\bm{\Sigma}^{(n)})^{-1}\mathbf{K}_*^{T} \bm{\omega}^{(n)}  + \lambda \left(\Delta r^{(n)} - \sum_{i=1}^{k_n} \Delta r_i^{(n)} \right),
		\label{obj:csao22}
		\end{align}}
	where $\lambda$ is a Lagrange multiplier.	The first-order optimality conditions are
	\begin{align}
	\frac{\partial \tilde{\im}_{SAO}}{\partial \Delta r_i^{(n)}} &= - \frac{\taun^2 \cdot (\bm{\omega}^{(n)})^T \mathbf{K}_* (\bm{\Sigma}^{(n)})^{-1} (\bm{\Sigma}^{(n)})^{-1} \mathbf{K}_*^{T} \bm{\omega}^{(n)}}{(r_i^{(n)} + \Delta r_i^{(n)})^2}  - \lambda = 0 \label{obj:csao23}
	\end{align}
	which leads to $r_i^{(n)} + \Delta r_i^{(n)}  \propto [(\bm{\Sigma}^{(n)})^{-1}\mathbf{K}_*^{T} \bm{\omega}^{(n)}]_i$, $1 \leq i \leq k_n$ as in \eqref{ui}.
\end{proof}

Following~\citet{liu2010stochastic}, we use a pegging procedure~\citep{bretthauer1999nonlinear} to obtain integer-valued  $\Delta r_i^{(n)}$, see Algorithm~\ref{alg:pegging} in the Appendix. Note that due to the rounding, the added number of replicates $\sum_{i=1}^{k_n} \Delta r_i^{(n)}$ is not exactly $\Delta r^{(n)}$.  Moreover, there are several approximations in Proposition~\ref{thm:deltar} that render $\Delta r_i^{(n)}$ and~\eqref{newri} suboptimal: (1) we assume that $\max_{i = 1,\ldots,k_n} \Delta \mathbf{R}_{ii}^{(n)}  \ll 1$; (2) we freeze the weights in \eqref{tmse} rather than using  $\bm{\omega}^{(n+1)}$; (3) we round off to integer $\Delta r_i^{(n)}$.

\begin{remark}
	Similar results about minimizing the look-ahead GP variance of a linear combination $\bm{\omega}^T \mathbf{f}$ appear in~\citep{ankenman2010stochastic, chen2017sequential, liu2010stochastic, ludkovski2017sequential}. % Compared to Ankenman et al.~\citep{ankenman2010stochastic} and Chen \& Zhou~\citep{chen2017sequential}, our algorithm is used in level set estimation set-ups. The weights $\bm{\omega}^{(n)}$ behave similar to penalty terms which penalize uncertainty in the neighborhood of the desired contour.
	Relative to~\citet{ankenman2010stochastic} and \citet{chen2017sequential}, we get rid of all integrals, making \eqref{newri} computationally efficient. %\citet{liu2010stochastic} use a 3-round allocation which in our experience is not sufficient to fully explore the space; our multi-round approach performs better.
	The algorithm proposed by~\citet{ludkovski2017sequential} relied on in-sample test set~${\mathbf{x}}_* = \bar{\mathbf{x}}_{1:k_n}$ while our test set is different from the existing inputs.
	
	{Proposition~\ref{thm:deltar} can be extended to the heteroskedastic setting by replacing the constant value $\taun^2$ in equations \eqref{obj:csao2}, \eqref{eqn:reasimplified}, \eqref{obj:csao22} and \eqref{obj:csao23} by a diagonal matrix $\mathbf{S}$ where $\mathbf{S}_{ii} = \taun^2(x_i), 1 \leq i \leq k_n$. Solving eq.~\ref{obj:csao23} leads to $r_i^{(n)} + \Delta r_i^{(n)}  \propto \taun^2(x_i) \mathbf{U}^{(n)}_i$, $1 \leq i \leq k_n$.}
\end{remark}	
}
	\section{Pegging Algorithm for ADSA} \label{appx:pegging}
	
	\begin{algorithm}[htb!]
		\caption{Pegging Algorithm}\label{alg:pegging}
		\begin{algorithmic}
			\State{\bfseries Input:} {$I_0 = \{1,\ldots,k_n\}$, $r = \sum_{i = 1}^{k_n} r_i^{(n)}$, $\mathbf{U}^{(n)}$ from eq.~\eqref{ui}}
			\State $j  \gets 0$.
			\For{all $i \in I_j$}
			\State $\Delta r_i^{(n)} \gets \frac{\mathbf{U}_i^{(n)}}{\sum_{j = 1}^{k_n} \mathbf{U}_j^{(n)}} \times r - r_i^{(n)} $
			\If{$\Delta r_i^{(n)} \geq 0$ for all $i \in I_j$}
			\State \textbf{break}
			\Else
			\State $I_{j+1} \gets \{i \in I_{j}: \Delta r_i^{(n)} > 0\}$
			\State $\Delta r_i^{(n)} = 0$ for $i \notin I_{j+1}$
			\State $r \gets r - \sum_{i \in I_{j}, i \notin I_{j+1}} r_i^{(n)}$
			\State $j \gets j+1$
			\EndIf
			\EndFor
			\State Round all $\Delta r_i^{(n)}, i = 1,..,k_n$ to the nearest integer.
			\State (If $\sum_{i =1}^{k_n} \Delta r_i^{(n)} = 0$, round $\max_{i =1}^{k_n} \Delta r_i^{(n)}$ up to the next integer)
		\end{algorithmic}
	\end{algorithm}
	
	\newstuff{\section{GP with Student $t$-Noise} \label{subsec:t}
	The  marginal likelihood of  $\bybar$ with a $t$-GP is (with $\mathbf{f} := f_{1:k_n} = (f(\xx_1), \ldots, f(\xx_{k_n}))$)
	\begin{align}
	p_{t\mathrm{GP}} \big({\bybar} \big| \; {\bxbar}, \mathbf{r}_{1:k_n}^{(n)}, \mathbf{f} \big) &= \prod_{i=1}^{k_n} \frac{\Gamma((\nu+1)/2)\sqrt{r_i^{(n)}}}{\Gamma(\nu/2)\sqrt{\nu\pi}\taun} \left(1+\frac{r_i^{(n)}(y_i-f_i)^2}{\nu\taun^2}\right)^{-(\nu+1)/2}, \label{liket}
	\end{align}
	where $\Gamma(\cdot)$ is the incomplete Gamma function. To integrate \eqref{liket} against the Gaussian prior $p(f|\bm{\vartheta})$ we use Laplace approximation \citep{williams1998bayesian}. Specifically, we use a second-order Taylor expansion of the log-likelihood around its mode, $\tilde{\mathbf{f}}_{t\mathrm{GP}}^{(n)}:=\arg \max_\mathbf{f}p_{t\mathrm{GP}} (\mathbf{f}|\bxbar,\bybar)$, to obtain a Gaussian approximation to the posterior $f(x_*)| \cA_n \sim \mathcal{N}( \hat{f}_{t\mathrm{GP}}^{(n)}(x_*), s_{t\mathrm{GP}}^{(n)}(x_*)^2)$ with
	\begin{align}
	\hat{f}_{t\mathrm{GP}}^{(n)}(x_*) &= \mathbf{k}(x_*)\mathbf{K}^{-1}\tilde{\mathbf{f}}_{t\mathrm{GP}}^{(n)}, \label{meant} \\
	v_{t\mathrm{GP}}^{(n)}(x_*,x_*') &= K(x_*,x_*')-\mathbf{k}(x_*) \bigg(\mathbf{K}+(\mathbf{W}_{t\mathrm{GP}}^{(n)})^{-1}\bigg) ^{-1}\mathbf{k}(x_*'), \label{covt} \\
	\nonumber &= K(x_*,x_*')-\mathbf{k}(x_*) (\bm{\Sigma}_{t\mathrm{GP}}^{(n)}) ^{-1}\mathbf{k}(x_*')
	\end{align}
	where $\mathbf{W}_{t\mathrm{GP}}^{(n)}$ is diagonal with
	\begin{align}
	W_{t\mathrm{GP},ii}^{(n)} =-\nabla^2 \log p_{t\mathrm{GP}}(\bar{y}_i|\tilde{f}_i^{(n)},\bar{x}_i)
	= (\nu + 1) \frac{\nu\frac{\taun^2}{r_i^{(n)}} - (\bar{y}_i - \tilde{f}_i^{(n)})^2}{(\nu\frac{\taun^2}{r_i^{(n)}} + (\bar{y}_i - \tilde{f}_i^{(n)})^2)^2}, \label{wt}
	\end{align}
	since the likelihood factorizes over observations. Note that $\nu$ is treated as part of the GP hyperparameters and fitted via MLE.
	
	~\citet{lyu2018evaluating} then calculated the approximate step-ahead variance of $t$-GP:
	\begin{align}
	s_{t\mathrm{GP}}^{(n+1)}(x_{k_n+1}, r_{k_n+1}^{(n)})^2 & \simeq s_{t\mathrm{GP}}^{(n)}(x_{k_n+1})^2 \cdot \frac{\frac{\taun^2}{r_{k_n+1}^{(n)}}\frac{\nu+1}{\nu-1}}{\frac{\taun^2}{r_{k_n+1}^{(n)}}\frac{\nu+1}{\nu-1} + s^{(n)}_{t\mathrm{GP}}(x_{k_n+1})^2}. \label{varpropt2}
	\end{align}
	We replace Eq.~\eqref{varprop2} with \eqref{varpropt2} to obtain the acquisition functions for $t$-GP.} %See Appendix~\ref{app: tcsao} for allocation rule of $t$-GP.}
	
%	\subsection
\textbf{Allocation Rule for $t$-GP}: %\label{app: tcsao}
	To implement ADSA and DDSA for $t$-GP we need (i) the analogue of Proposition~\ref{thm:deltar} for the allocation rule  $\Delta \mathbf{r}_{1:k_n}^{(n)}$ over the existing inputs $\bxbar$; (ii) the look-ahead variance $s^{(n+1),new}(x_*)$ conditional on adding a new input; (iii) look-ahead variance $s^{(n+1),all}(x_*)$ conditional on allocating $\Delta \mathbf{r}_{1:k_n}^{(n)}$. For all these tasks, the non-Gaussian likelihood \eqref{liket} underlying $t$-GP calls for further approximations provided in the following three Lemmas.
	
	\begin{lemma}[Allocation Rule]\label{lem:U-tgp}
		The allocation $\Delta \mathbf{r}_{1:k_n}^{(n)}$ is like in Proposition~\ref{thm:deltar} but relies on
		\begin{align}\widetilde{\mathbf{U}}^{(n)}_{t\mathrm{GP}} = (\widetilde{\bm{\Sigma}}^{(n)}_{t\mathrm{GP}})^{-1}\mathbf{K}_*^T\bm{\omega}^{(n)}, \qquad \text{with}\quad \widetilde{\bm{\Sigma}}^{(n)}_{t\mathrm{GP}} := \bigg(\mathbf{K}+\frac{\nu + 1}{\nu - 1}\taun^2\mathbf{R}^{(n)}\bigg). \label{eq:hat-U-tgp}
		\end{align}
	\end{lemma}
	
	{\begin{proof}[Proof of Lemma~\ref{lem:U-tgp}]
			For $t$-GP, the noise matrix $\taun^2 \mathbf{R}^{(n)}$ in eq.~\eqref{cov} is replaced with $(\bm{W}_{t\mathrm{GP}}^{(n)})^{-1}$. To calculate the ADSA/DDSA allocation rule with a $t$-GP metamodel we substitute  $(\bar{y}_{i}-\tilde{f}^{(n)}_{t\mathrm{GP}}(\bar{x}_{i}))^2 \approxeq \frac{\taun^2}{r_i^{(n)}}$ and $\tilde{f}^{(n)}_{t\mathrm{GP}}(\bar{x}_{i}) \approxeq \tilde{f}^{(n+1)}_{t\mathrm{GP}}(\bar{x}_{i})$ in eq. \eqref{wt} to obtain (cf.~\citealt{lyu2018evaluating})
			\begin{align*}
			W_{ii}^{(n)} &= (\nu + 1) \frac{\nu\frac{\taun^2}{r_i^{(n)}} - (\bar{y}_i - \tilde{f}_i^{(n)})^2}{(\nu\frac{\taun^2}{r_i^{(n)}} + (\bar{y}_i - \tilde{f}_i^{(n)})^2)^2} \\
			& \approxeq
			(\nu+1)\frac{\nu\frac{\taun^2}{r_i^{(n)}}-\frac{\taun^2}{r_i^{(n)}}}{\left(\frac{\taun^2}{r_i^{(n)}}+\nu\frac{\taun^2}{r_i^{(n)}} \right)^2}
			= \frac{(\nu - 1)r_i^{(n)}}{(\nu + 1)\taun^2} := \widetilde{W}_{ii}^{(n)}.
			\end{align*}
			Hence,
			$(\bm{W}_{t\mathrm{GP}}^{(n)} )^{-1} \approxeq \widetilde{(\bm{W}}_{t\mathrm{GP}}^{(n)})^{-1} = \frac{\nu + 1}{\nu - 1}\taun^2\mathbf{R}^{(n)}$ and  the covariance matrix ${\mathbf{C}}^{(n)}_{t\mathrm{GP}}$ of $f({\mathbf{x}}_*)$ is approximated as
			\begin{align}
			\nonumber {\mathbf{C}}^{(n)}_{t\mathrm{GP}} &= \mathbf{K}({\mathbf{x}}_*,{\mathbf{x}}_*)-\mathbf{K}_* \bigg(\mathbf{K}+({\mathbf{W}}_{t\mathrm{GP}}^{(n)})^{-1}\bigg) ^{-1}\mathbf{K}_*^T\\
			\nonumber & \simeq \mathbf{k}(\bar{\mathbf{x}}_*,\bar{\mathbf{x}}_*)-\mathbf{k}_*\bigg(\mathbf{K}+\frac{\nu + 1}{\nu - 1}\taun^2\mathbf{R}^{(n)}\bigg) ^{-1}\mathbf{k}_*^T \\
			& \simeq \mathbf{K}(\bar{\mathbf{x}}_*,\bar{\mathbf{x}}_*)-\mathbf{K}_*(\widetilde{\bm{\Sigma}}_{t\mathrm{GP}}^{(n)}) ^{-1}\mathbf{K}_*^T, \label{approxcovt}
			\end{align}
			where $\widetilde{\bm{\Sigma}}_{t\mathrm{GP}}^{(n)}$ matches eq.~\eqref{eq:hat-U-tgp}. The rest of the proof proceeds exactly like for the regular GP model in Proposition~\ref{thm:deltar}, after boosting $\taun^2$ up by a constant ratio to $(\nu+1)/(\nu-1) \taun^2$. Then we obtain $\widetilde{\bm{U}}^{(n)}_{t\mathrm{GP}}$ as defined in~\eqref{eq:hat-U-tgp}. %Combined with Algorithm \ref{alg:pegging} this yields the allocation rule $\Delta \bm{r}_{1:k_n}^{(n)}$ for $t$-GP.
		\end{proof}
		
		Next, we need to approximate the next-step $\bm{W}_{t\mathrm{GP}}^{(n+1)}$.
		Unlike in the Gaussian case where $\bm{\Sigma}^{(n+1)}$ depends only on $\mathbf{R}^{(n+1)}$, for $t$-GP $\bm{W}_{t\mathrm{GP}}^{(n+1)}$ depends on $\bybar$ (because it depends on $\tilde{\mathbf{f}}_{t\mathrm{GP}}$). We therefore need an approximation $\widehat{\bm{W}}_{t\mathrm{GP}}^{(n+1)}$ (the notation is to emphasize that it is different from the previous approximation $\widetilde{\bm{W}}_{t\mathrm{GP}}^{(n)}$ to $\bm{W}_{t\mathrm{GP}}^{(n)}$).
		
		\begin{lemma}[Look-Ahead $t$-GP Variance]\label{lem:s-new} The look-ahead variance at $x_*$ conditional on allocating $\Delta r^{(n)}$ simulations to a new input $\bar{x}_{k_n+1}$ is approximately given by
			\begin{align}
			\tilde{s}_{t\mathrm{GP}}^{(n+1),new}(x_*)^2
			&\approxeq s_{t\mathrm{GP}}^{(n)}(x_*)^2 - \frac{v_{t\mathrm{GP}}^{(n)}(x_*, \bar{x}_{k_n+1})^2}{\frac{(\nu + 1)\taun^2}{(\nu - 1) \Delta r^{(n)}}+s_{t\mathrm{GP}}^{(n)}(\bar{x}_{k_n+1})^2}. \label{s0-csao3-t}
			\end{align}
		\end{lemma}
		%See~\citep{lyu2018evaluating} for the proof of Lemma~\ref{lem:s-new}. We use \eqref{s0-csao3-t} to obtain $\im_{SAO}^{(n),new}$.
		
		Finally, to obtain $\im_{SAO}^{(n),all}$ we define
		\begin{align}
		\widehat{W}^{(n+1)}_{ii}  & := (\nu+1)\frac{\nu\frac{\taun^2}{r_i^{(n+1)}}-(\bar{y}^{(n)}_{i}-\tilde{f}^{(n)}_{t\mathrm{GP}}(\bar{x}_{i}))^2}{\left((\bar{y}^{(n)}_{i}-\tilde{f}^{(n)}_{t\mathrm{GP}}(\bar{x}_{i}))^2+\nu\frac{\taun^2}{r_i^{(n+1)}} \right)^2}, \label{wthattgp}
		\end{align}
		based on the approximation $(\bar{y}_{i}^{(n+1)} - \tilde{f}^{(n+1)}_{t\mathrm{GP}}(x_{i}))^2 \approxeq (\bar{y}_{i}^{(n)}-\tilde{f}^{(n)}_{t\mathrm{GP}}(x_{i}))^2$. This  yields
		
		\begin{lemma}[Look-ahead $t$-GP variance after batch allocation]\label{lem:s-al}
			\begin{align}
			\tilde{s}_{t\mathrm{GP}}^{(n+1),all}({x}_*) \approxeq K({x}_*,{x}_*)-\mathbf{K}_* \bigg(\mathbf{K}+(\widehat{\mathbf{W}}_{t\mathrm{GP}}^{(n+1)})^{-1}\bigg) ^{-1}\mathbf{K}_*^T. \label{covtapprox}
			\end{align}
	\end{lemma}}

	\section{Tuning Parameters for ABSUR and ADSA}\label{sec:tuning-adsa}

	\begin{table*}[htb]
		%	\begin{adjustwidth}{-.5in}{-.5in}
		\caption{Varying $\bar{r}$ (left panel) and $T_{sim}$ (right panel) for ABSUR. We report the mean error rate ${\cal ER}_T$, running time $t$ (in seconds) and the design size $k_T$  for the 2-D synthetic case studies with Gaussian noise $\epsilon \sim {\cal N}(0,\tau^2)$ and budget $N_T=2000$. All other hyperparameters are set as in Table~\ref{tbl:experiments}. Results are based on 20 macroreplications of each scheme.}
		\centering
		{ 	\begin{tabular}{cc}	\begin{tabular}{rrrr}
			\hline\noalign{\smallskip}
			$\bar{r}$  & ${\cal ER}_T$   & $t$  & $k_T$  \\
			\noalign{\smallskip}\hline\noalign{\smallskip}
			\multicolumn{4}{c}{$\taun^2 = 0.01$} \\
			\noalign{\smallskip}\hline\noalign{\smallskip}
			$0.01N_T$ & 0.21\% & 54.1 & 111.5 \\
			$0.025N_T$ & 0.24\% & 28.2 & 59.2 \\
			$0.05N_T$ & 0.23\% & 20.9 & 43.5 \\
			$0.1N_T$  & 0.30\% & 15.3 & 38.6 \\
			$0.25N_T$  & 0.31\% & 13.7 & 36.0 \\
			$N_T$ & 0.58\% & 9.0  & 30.1 \\
			\noalign{\smallskip}\hline\noalign{\smallskip}
			\multicolumn{4}{c}{$\taun^2 = 0.25$} \\
			\noalign{\smallskip}\hline\noalign{\smallskip}
			$0.01N_T$ & 1.26\% & 48.0 & 110.9 \\
			$0.025N_T$ & 1.31\% & 22.0 & 57.6 \\
			$0.05N_T$ & 1.18\% & 13.5 & 40.9 \\
			$0.1N_T$  & 1.29\% & 9.7 & 34.7 \\
			$0.25N_T$  & 1.41\% & 9.9 & 33.1 \\
			$N_T$ & 1.64\% & 8.2 & 29.8 \\
			\noalign{\smallskip}\hline\noalign{\smallskip}
			\multicolumn{4}{c}{$\taun^2 = 1$} \\
			\noalign{\smallskip}\hline\noalign{\smallskip}
			$0.01N_T$ & 2.05\% & 46.1 & 110.8 \\
			$0.025N_T$ & 2.01\% & 21.1 & 57.5 \\
			$0.05N_T$ & 1.78\% & 12.4 & 40.8 \\
			$0.1N_T$  & 1.93\% & 9.7 & 34.3 \\
			$0.25N_T$  & 2.03\% & 9.2 & 32.9 \\
			$N_T$ & 2.24\% & 9.2 & 30.8 \\
			\noalign{\smallskip}\hline\noalign{\smallskip}
		\end{tabular}  & $\qquad$
		\begin{tabular}{rrrr}
			\hline\noalign{\smallskip}
			$T_{sim}$  & ${\cal ER}_T$   & $t$  & $k_T$  \\
			\noalign{\smallskip}\hline\noalign{\smallskip}
			\multicolumn{4}{c}{$\taun^2 = 0.01$} \\
			\noalign{\smallskip}\hline\noalign{\smallskip}
			$0.0001$ & 2.16\% & 11.4 & 31.0 \\
			%$0.0005$ & 3.07\% & 9.2 & 30 \\
			$0.001$  & 0.27\% & 12.5 & 31.9 \\
			%$0.005$  & 2.91\% & 9.3 & 33 \\
			$0.01$  & 0.30\% & 15.2 & 38.6 \\
			%$0.05$ & 2.85\% & 14.3 & 46 \\
			$0.1$ & 0.21\% & 23.6 & 60.4 \\
			%$0.5$ & 2.65\% & 24.9 & 82 \\
			$1$ & 0.19\% & 34.5 & 100.1 \\
			%$5$ & 3.01\% & 33.1 & 123 \\
			$10$ & 0.23\% & 31.6 & 115.1 \\
			\noalign{\smallskip}\hline\noalign{\smallskip}
			\multicolumn{4}{c}{$\taun^2 = 0.25$} \\
			\noalign{\smallskip}\hline\noalign{\smallskip}
			$0.0001$ & 1.45\% & 9.6 & 30.1 \\
			%$0.0005$ & 3.07\% & 9.2 & 30 \\
			$0.001$  & 1.44\% & 9.0 & 30.4 \\
			%$0.005$  & 2.91\% & 9.3 & 33 \\
			$0.01$  & 1.29\% & 10.1 & 34.7 \\
			%$0.05$ & 2.85\% & 14.3 & 46 \\
			$0.1$ & 1.38\% & 16.8 & 53.8 \\
			%$0.5$ & 2.65\% & 24.9 & 82 \\
			$1$ & 1.29\% & 31.6 & 97.7 \\
			%$5$ & 3.01\% & 33.1 & 123 \\
			$10$ & 1.30\% & 37.2 & 128.6 \\
			\noalign{\smallskip}\hline\noalign{\smallskip}
			\multicolumn{4}{c}{$\taun^2 = 1$} \\
			\noalign{\smallskip}\hline\noalign{\smallskip}
			$0.0001$ & 2.27\% & 8.4 & 30.0 \\
			%$0.0005$ & 4.03\% & 7.9 & 30 \\
			$0.001$  & 2.46\% & 8.8 & 30.4 \\
			%$0.005$  & 4.82\% & 8.0 & 32 \\
			$0.01$  & 1.93\% & 9.5 & 34.3 \\
			%$0.05$ & 3.87\% & 12.6 & 45 \\
			$0.1$ & 1.89\% & 16.5 & 53.9 \\
			%$0.5$ & 4.51\% & 24.3 & 84 \\
			$1$ & 1.98\% & 31.5 & 100.6 \\
			%$5$ & 4.46\% & 37.7 & 139 \\
			$10$ & 2.10\% & 44.3 & 141.9 \\
			\noalign{\smallskip}\hline\noalign{\smallskip}
			\end{tabular} \end{tabular}}
		%\end{adjustwidth}
		\label{tbl:absur-rbar}
	\end{table*}

\begin{table*}[htb]
			%	\begin{adjustwidth}{-.5in}{-.5in}
			\caption{Mean error rate ${\cal ER}_T$, computation cost $t$ (in seconds) and the design size $k_T$  for ADSA and DDSA with variable $c_{bt}$ for the 2-D synthetic case studies with Gaussian noise and budget $N_T=2000$. All other hyperparameters are the same as in Table~\ref{tbl:experiments}. Results are based on 20 macroreplications of each scheme.}
			\centering
		{		\begin{tabular}{rrrrrrr}
		\hline\noalign{\smallskip}
		& \multicolumn{3}{c}{ADSA}  & \multicolumn{3}{c}{DDSA}  \\
		\noalign{\smallskip}\hline\noalign{\smallskip}
		\multicolumn{7}{c}{$\taun^2 = 0.01$} \\
		\noalign{\smallskip}\hline\noalign{\smallskip}
		$c_{bt}$  & ${\cal ER}_T$   & $t$  & $k_T$  & ${\cal ER}_T$   & $t$  & $k_T$ \\
		\noalign{\smallskip}\hline\noalign{\smallskip}
		$0.5$ & 0.54\% & 204.2 & 25.4 & 0.21\% & 139.0 & 226 \\
		$1$ & 0.67\% & 125.3 & 23.4 & 0.23\% & 58.0 & 133 \\
		$2.5$ & 0.57\% & 62.8 & 23.9 & 0.20\% & 24.1 & 73 \\
		$5$ & 0.72\% & 37.2 & 23.1 & 0.20\% & 13.8 & 51 \\
		$10$  & 0.83\% & 22.2 & 22.4 & 0.25\% & 7.6 & 37 \\
		$20$  & 1.03\% & 11.9 & 21.3 & 0.39\% & 4.1 & 30 \\
		$40$  & 1.07\% & 6.5 & 20.8 & 2.04\% & 2.3 & 25 \\
		$80$ & 1.42\% & 3.8 & 20.5 & 1.21\% & 1.3 & 23 \\
		\noalign{\smallskip}\hline\noalign{\smallskip}
		\multicolumn{7}{c}{$\taun^2 = 0.25$} \\
		\noalign{\smallskip}\hline\noalign{\smallskip}
		$c_{bt}$  & ${\cal ER}_T$   & $t$  & $k_T$  & ${\cal ER}_T$   & $t$  & $k_T$ \\
		\noalign{\smallskip}\hline\noalign{\smallskip}
		$0.5$ & 1.45\% & 211.9 & 29.6 & 1.20\% & 147.9 & 226 \\
		$1$ & 1.37\% & 125.7 & 26.3 & 1.21\% & 64.5 & 133 \\
		$2.5$ & 1.50\% & 66.3 & 23.9 & 1.26\% & 25.5 & 73 \\
		$5$ & 1.38\% & 38.7 & 23.3 & 1.19\% & 13.5 & 51 \\
		$10$  & 1.41\% & 22.5 & 22.8 & 1.32\% & 7.5 & 37 \\
		$20$  & 1.48\% & 12.8 & 22.2 & 1.43\% & 4.4 & 30 \\
		$40$  & 1.71\% & 6.8 & 21.7 & 1.55\% & 2.4 & 25 \\
		$80$ & 1.76\% & 3.7 & 21.0 & 1.76\% & 1.4 & 23 \\
		\noalign{\smallskip}\hline\noalign{\smallskip}
		\multicolumn{7}{c}{$\taun^2 = 1$} \\
		\noalign{\smallskip}\hline\noalign{\smallskip}
		$c_{bt}$  & ${\cal ER}_T$   & $t$  & $k_T$  & ${\cal ER}_T$   & $t$  & $k_T$ \\
		\noalign{\smallskip}\hline\noalign{\smallskip}
		$0.5$ & 1.94\% & 358.8 & 256.0 & 1.70\% & 146.9 & 226 \\
		$1$ & 1.94\% & 172.0 & 134.0 & 1.80\% & 63.7 & 133 \\
		$2.5$ & 1.91\% & 76.0 & 69.0 & 1.89\% & 27.0 & 73 \\
		$5$ & 1.95\% & 42.8 & 45.9 & 1.90\% & 15.6 & 51 \\
		$10$  & 1.97\% & 24.2 & 33.2 & 1.99\% & 8.0 & 37 \\
		$20$  & 2.04\% & 13.3 & 27.3 & 2.26\% & 4.5 & 29 \\
		$40$  & 2.03\% & 7.0 & 24.2 & 2.71\% & 2.3 & 25 \\
		$80$ & 2.63\% & 4.0 & 22.3 & 3.13\% & 1.3 & 23 \\
		\noalign{\smallskip}\hline\noalign{\smallskip}
		\end{tabular}}
			%\end{adjustwidth}
			\label{tbl:adsa-cbt}
		\end{table*}

\begin{figure*}[htb]
		\begin{center}
			\includegraphics[width=0.5\textwidth,trim=0.3in 2.5in 0.3in 3in]{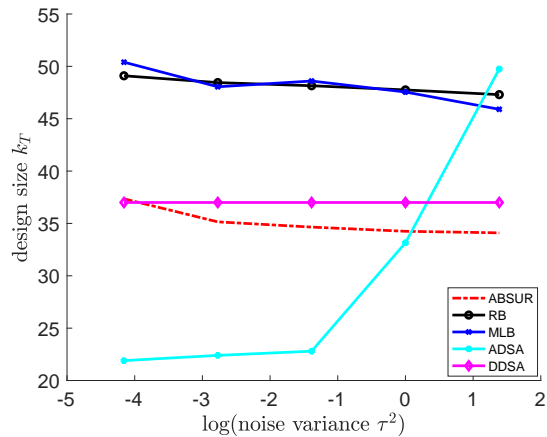} \\
		\end{center}
	\caption{Design size $k_T$ as a function of noise variance $\taun^2$ at $\taun^2 = \{4^{-3}, 4^{-2}, 4^{-1}, 1, 4\}$ in the 2-D experiment with $\epsilon \sim {\cal N}(0,\taun^2)$ and budget $N_T=2000$. Hyperparameters are set the same as in Table~\ref{tbl:experiments}.}
	\label{fig:kT_vs_taun2}
\end{figure*}

	\begin{figure*}
		\centering
		\begin{mdframed}
			\begin{multicols}{2}
				\textbf{Nomenclature}
				\begin{description}
                    \item[$n$] Sequential design step, indexes most quantities below
					\item[$\mathcal{A}$] Design set
					\item[$D$] Input space
                    \item[$d$] Dimension of input space
					\item[$Y(\cdot)$] Response
					\item[$X$] Design
					\item[$k$] Number of unique inputs
					\item[$N$] Total budget
					\item[$r$] Replicate count
					\item[$\bar{x}$] Design location
					\item[$\taun$] Noise variance
					\item[$f$] Latent function
					\item[$S$] Level set
					\item[$\epsilon$] Noise
					\item[$\bar{y}$] Average response
					\item[$\mathcal{ER}$] Error rate
					\item[$K(\cdot, \cdot)$] Covariance function
					\item[$\hat{f}(\cdot)$] Posterior mean
					\item[$v(\cdot)$] Posterior variance
					\item[$s(\cdot)$] Posterior standard deviation
					\item[$\mathcal{I}(\cdot)$] Acquisition function
					\item[$\rho$] cUCB weight
					\item[$\mu(\cdot)$] Lebesgue measure
					\item[$E$] Local empirical error
					\item[$\gamma$] Standard deviation threshold
					\item[$\eta$] Reduction factor
					\item[$L$] Number of fidelities
					\item[$c_{ovh}$] Optimization overhead
					\item[$T_{sim}$] Computation time
					\item[$\mathcal{L}$] Look-ahead integrated contour uncertainty
					\item[$\omega$] Level set contour weights
					\item[$c_{bt}$] Batch factor
					\item[$l$] Length-scale
					\item[$\sigma_{\text{se}}$] Function variance
					\item[$M$] Test set size
					%\item[$Z$] Geometric Brownian motion
					%\item[$\cal{K}$] Strike price
				\end{description}
			\end{multicols}
		\end{mdframed}
	\end{figure*}

	\bibliography{reference}
	\bibliographystyle{abbrvnat}
	
\end{document}